\documentclass{article}
\usepackage{iclr2027_conference,times}
\usepackage{amsmath,amssymb,amsthm,mathtools}
\usepackage{algorithm}
\usepackage{algpseudocode}
\usepackage{booktabs}
\usepackage{graphicx}
\usepackage{tikz}
\usetikzlibrary{arrows.meta,positioning}
\usepackage{microtype}
\usepackage{xcolor}
\usepackage[hyphens]{url}
\usepackage{hyperref}
\usepackage{tabularx}
\usepackage{ragged2e}
\definecolor{linkblue}{RGB}{20,70,130}
\hypersetup{
  colorlinks=true,
  linkcolor=linkblue,
  citecolor=linkblue,
  urlcolor=linkblue
}

\newtheorem{theorem}{Theorem}[section]
\newtheorem{lemma}[theorem]{Lemma}
\newtheorem{corollary}[theorem]{Corollary}
\newtheorem{proposition}[theorem]{Proposition}

\newtheorem{assumption}[theorem]{Assumption}

\newcommand{\indep}{\mathrel{\perp\!\!\!\perp}}
\newcommand{\eps}{\varepsilon}

\newcommand{\E}{\mathbb E}
\newcommand{\Pp}{\mathbb P}
\newcommand{\Var}{\operatorname{Var}}
\newcommand{\Cov}{\operatorname{Cov}}
\newcommand{\Proj}{\operatorname{Proj}}
\newcommand{\Pa}{\operatorname{Pa}}

\newcommand{\psiTwo}{\psi_2}

\title{The Statistical Cost of Causal Discovery with Feedback}

\author{
    Sunmin Oh\thanks{Department of Statistics; Institute for Data Innovation in Science, Seoul National University, Korea},
    Seungsu Han\thanks{Department of Operations Research and Financial Engineering, Princeton University, USA},
    Gunwoong Park\thanks{Department of Statistics; Interdisciplinary Program in Artificial Intelligence; Institute for Data Innovation in Science, Seoul National University, Korea. \texttt{gw.park23@gmail.com}}
}

\iclrfinalcopy

\renewcommand{\iclrruler}[1]{}

\begin{document}
\maketitle
\lhead{}


\begin{abstract}
What determines the unavoidable sample cost of learning cyclic causal structure? For cyclic linear non-Gaussian models, we study exact condensation recovery from observational data: identifying the strongly connected component (SCC) partition and all edges between components. We establish the first information-theoretic lower bounds on sample complexity for this target. For $p$ variables, maximum SCC size $s_{\max}$, and maximum external-parent count $d_B$, any estimator requires order $s_{\max}\log(ep/s_{\max})+d_B\log(ep/d_B)$ samples in the worst case over a regular model class. These bounds distinguish the costs of SCC membership and external-parent selection. Under principal invertibility and without correlation faithfulness, we establish a population block-exogeneity principle that identifies unknown root SCCs through residual independence and inclusion minimality. A sparse-adjustment characterization shows that small adjustment sets suffice to identify SCCs and their direct external parents, without regressing on all previously recovered variables. These characterizations yield BlockExo, which attains a structurally matching sample bound without knowing $s_{\max}$ or $d_B$ under suitable conditions. Simulations support the structural dependence of our sample bound and demonstrate BlockExo's sample-efficient recovery in comparisons with other methods for cyclic causal discovery.
\end{abstract}

\section{Introduction}
\label{sec:introduction}

Feedback is common in gene regulation, economic systems, and dynamical systems \citep{dai2024local,richardson2014ace,bongers2021foundations}.
For cyclic linear non-Gaussian (LiNG) models, a natural observational target is the condensation: the strongly connected component (SCC) partition and all edges between components.
This target is invariant across the relevant observational equivalence class \citep{sharifian2025near,madaleno2026coarsening}.
We ask which structural sample costs are unavoidable and which population principles enable recovery with a structurally matching sample bound.

We establish the first sample-complexity lower bounds for exact condensation recovery in cyclic LiNG models.
For $p$ variables, let $s_{\max}$ be the maximum SCC size and $d_B$ the maximum number of external parent variables of an SCC.
Over regular model classes, these bounds imply the necessary structural sample cost
$
    s_{\max}\log\left({ep}/{s_{\max}}\right)
    +d_B\log\left({ep}/{d_B}\right),
$
where $e=\exp(1)$ and $\log$ denotes the natural logarithm.
The two terms capture unavoidable costs of identifying SCC membership and selecting external parents, respectively.
The first term remains even when there are no edges between SCCs: identifying which variables belong to the same SCC alone incurs a sample cost.

We complement these lower bounds with population principles for identifying unknown SCCs and their external parents.
Under principal invertibility, our \textit{block-exogeneity} shows that, after adjusting for previously recovered variables, residual independence characterizes blocks with no incoming edges from the remaining variables, without correlation faithfulness.
The inclusion-minimal nonempty blocks satisfying this condition are exactly the remaining root SCCs, allowing their identification without a known SCC partition.
Small adjustment sets further suffice to identify both SCCs and their direct external parents.
This sparse-adjustment characterization avoids regressing on all previously recovered variables and controls regression dimensions through SCC size and external-parent count.
These principles yield BlockExo, an algorithm that attains a structurally matching sample bound under regularity and separation conditions, even when $s_{\max}$ and $d_B$ are unknown.

\paragraph{Related work.} GroupLiNGAM characterizes exogenous sets through residual independence under correlation faithfulness \citep{kawahara2010grouplingam}.
For exact condensation recovery, Coarsening provides an ICA-based method with a finite-sample recovery guarantee and polynomial-time computation \citep{madaleno2026coarsening}.
OptLiNGAM establishes necessary and achievable structural sample costs for LiNGAMs \citep{oh2025optimal}.
Distributional equivalence has also been characterized for LiNG models with latent variables and cycles \citep{dai2026glvling}.
Appendix~\ref{sec:related} gives detailed comparisons.

Our contributions are summarized as follows.
\begin{itemize}
    \item \textbf{Information-theoretic lower bounds.}
    We establish the first sample-complexity lower bounds for exact condensation recovery in cyclic LiNG models, separating the costs of SCC membership and external-parent selection over regular model classes (Theorem~\ref{thm:matching-lower}).
    \item \textbf{Block exogeneity for root SCC identification.}
    We characterize ancestral blocks by residual independence and identify root SCCs through inclusion minimality, under principal invertibility and without correlation faithfulness (Theorem~\ref{thm:block-exogeneity} and Corollary~\ref{cor:minimal-scc}).
    \item \textbf{Sample-efficient recovery through sparse adjustment.}
    We show that small adjustment sets suffice to identify ancestral blocks and their external parents (Proposition~\ref{prop:sparse-passing}).
    BlockExo attains a structurally matching sample bound without knowing $s_{\max}$ or $d_B$ under suitable conditions (Theorem~\ref{thm:finite-recovery}), and is minimax optimal in sample complexity on a separated model class for fixed structural bounds (Corollary~\ref{cor:uniform-recovery-optimality}).
\end{itemize}

Proofs of all theoretical results are provided in the appendix.

\section{Model and condensation recovery target}
\label{sec:setup}

\paragraph{Cyclic LiNG model.}
Let $G=(V,E)$ be a directed graph on $V=[p]:=\{1,\ldots,p\}$.
Its weighted adjacency matrix $B\in\mathbb R^{p\times p}$ satisfies $\operatorname{diag}(B)=0$, with $B_{ij}\ne0$ if and only if $j\to i\in E$.
We consider cyclic LiNG models \citep{lacerda2008discovering} of the form
\begin{equation}
    X=BX+\eps,
    \label{eq:sem}
\end{equation}
where $X,\eps\in\mathbb R^p$.
The coordinates of $\eps$ are mutually independent and centered, with positive finite variances, and at most one is Gaussian.
With $I$ denoting the identity matrix, we assume the following.

\begin{assumption}[Principal invertibility]
\label{ass:PI}
Every principal submatrix of $I-B$ is invertible.
\end{assumption}
Such well-posedness conditions are standard in the literature on linear cyclic causal models \citep{hyttinen2012learning,bongers2021foundations,dai2024local}.
This condition is automatic for DAGs and is implied by absolute stability, $\rho(|B|)<1$ \citep{misiakos2026stablespin}, where $\rho$ denotes the spectral radius and $|B|$ is the entrywise absolute value of $B$.
We observe $n$ i.i.d.\ copies of $X=(I-B)^{-1}\eps$ and write
$\Sigma=\operatorname{Cov}(X)=A\operatorname{Cov}(\eps)A^\top\succ0$,
where $A:=(I-B)^{-1}$.

For index sets $S,T\subseteq V$, write $X_S$ for the corresponding coordinate subvector, $G[S]$ for the induced subgraph, and $B_{S,T}$ for the submatrix with rows in $S$ and columns in $T$.
We use the same indexing convention for other matrices and omit braces for singleton indices.

\paragraph{Recovery target.}
An SCC is an inclusion-maximal set of mutually reachable nodes.
Write $\Pi=\{S_1,\ldots,S_m\}$ for the SCC partition of $G$.
The target is the condensation $G^{\rm sc}=(\Pi,E^{\rm sc})$, where
\[
    E^{\rm sc}
    = \{S\to T:S,T\in\Pi,\ S\ne T,\ B_{T,S}\ne0\}.
\]
Recovery requires the SCC partition and every edge of the condensation, including direct edges not determined by reachability alone.
This target is invariant across the observational equivalence class of a cyclic LiNG model \citep{sharifian2025near,madaleno2026coarsening}.
For a directed acyclic graph (DAG), every SCC is a singleton, so exact condensation recovery is equivalent to graph recovery.

\paragraph{Structural parameters.}
For a nonempty set $C\subseteq V$, define its external parent set by
\[
    P_{\rm ext}(C)
    = \{k\in V\setminus C:B_{C,k}\ne0\}.
\]
The maximum SCC size and maximum number of external parent variables per SCC are, respectively,
\[
    s_{\max}=\max_{S\in\Pi}|S|,
    \qquad
    d_B=\max_{S\in\Pi}|P_{\rm ext}(S)|.
\]
We count distinct external parent variables, regardless of how many edges they send into the SCC or whether they belong to the same parent SCC.
For DAGs, $s_{\max}=1$ and $d_B$ is the maximum indegree.

Our sample complexity results use the following regularity condition.
\begin{assumption}[Sub-Gaussianity and covariance eigenvalue bounds]
\label{ass:sg-ev}
There exist constants $K>0$ and $\lambda\ge1$, independent of
$p,s_{\max},d_B$, such that, for every $a\in\mathbb R^p$,
\[
    \|a^\top X\|_{\psi_2}
    \le K\sqrt{a^\top\Sigma a},
    \qquad
    \lambda^{-1}
    \le\lambda_{\min}(\Sigma)
    \le\lambda_{\max}(\Sigma)
    \le\lambda.
\]
Here $\|U\|_{\psi_2}:=\inf\{a>0:\E\exp(U^2/a^2)\le2\}$
denotes the sub-Gaussian norm of a random variable $U$.
$\lambda_{\min}$ and $\lambda_{\max}$ denote
the minimum and maximum eigenvalues, respectively.
\end{assumption}

\section{Necessary structural sample costs}
\label{sec:lower}

Exact condensation recovery requires recovering the SCC partition and identifying external parents.
We establish separate information-theoretic lower bounds for these two learning tasks.
Their combined structural dependence is attained by the recovery guarantee in Section~\ref{sec:finite} under suitable conditions.

For integers $s\ge1$ and $d\ge0$, let $\mathcal M_{p,s,d}(K,\lambda)$ be the class of cyclic LiNG models on $p$ variables satisfying Assumptions~\ref{ass:PI} and~\ref{ass:sg-ev} with constants $(K,\lambda)$, and $s_{\max}\le s$, $d_B\le d$.
For each model $M$, let $P_M$ denote the law of one observation, write $P_M^n:=P_M^{\otimes n}$ for the law of the $n$ i.i.d.\ observations, and let $G^{\rm sc}(M)$ be its true condensation.

\begin{theorem}[Sample-complexity lower bound]
\label{thm:matching-lower}
There exist universal constants $K,c_{\rm lb}>0$ and $\lambda\ge1$ such that the following holds.
For any integers $p\ge4$, $1\le s, d\le p/2$, and $n\ge1$, and any $0<\delta<1$, if
\begin{equation}
    n\le \frac{1-\delta}{c_{\rm lb}}
    \left[
        s\log\frac{ep}{s}+d\log\frac{ep}{d}
    \right],
    \label{eq:lower-sample-size}
\end{equation}
then every estimator $\widehat G^{\rm sc}$ satisfies
\begin{equation*}
    \sup_{M\in\mathcal M_{p,s,d}(K,\lambda)}
    P_M^n\!\left(
        \widehat G^{\rm sc}\ne G^{\rm sc}(M)
    \right)
    \ge\delta.
\end{equation*}
\end{theorem}

When the sample size satisfies \eqref{eq:lower-sample-size}, no estimator can achieve worst-case error probability below $\delta$.
For $s\ge2$, identifying SCC membership requires the first cost even when the condensation has no edges.
The second cost arises from external-parent selection even on DAGs, where all SCCs are singletons.
In this case, $s=1$ and $\log(ep)\le d\log(ep/d)$, so the bound has the same order as the parent-selection cost.

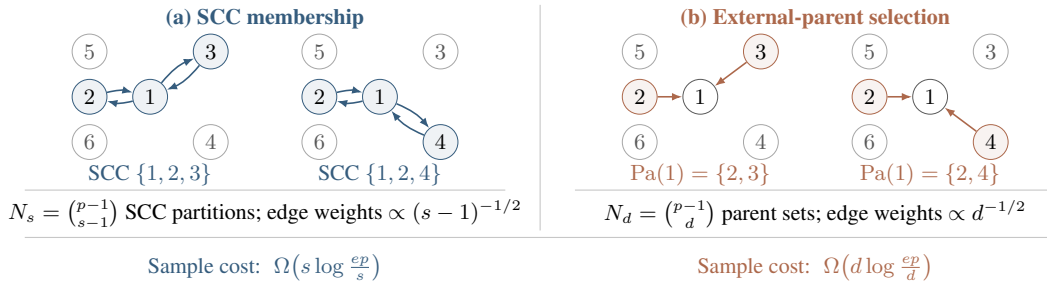
\begin{figure}[htbp]
\centering
\definecolor{lowerSCC}{HTML}{355C7D}
\definecolor{lowerParent}{HTML}{AC684B}
\tikzset{
 vertex/.style={circle,draw=black!35,fill=white,minimum size=5mm,inner sep=0pt,font=\footnotesize,text=black!65},
 selected/.style={vertex,draw=lowerActive,fill=lowerActive!8,text=black},
 edge/.style={-{Latex[length=1.5mm]},line width=.65pt,lowerActive},
 heading/.style={font=\small\bfseries,align=center},
 rule/.style={black!22,line width=.45pt},
 detail/.style={font=\footnotesize,align=center}
}
\newcommand{\lowerExample}[4]{%
 \begin{scope}[shift={(#3,#4)}]
  \ifnum#1=0\colorlet{lowerActive}{lowerSCC}\else\colorlet{lowerActive}{lowerParent}\fi
  \ifnum#1=0
   \node[selected] (v1) at (0,0) {$1$};
  \else
   \node[vertex,draw=black!65,text=black] (v1) at (0,0) {$1$};
  \fi
  \node[selected] (v2) at (-.87,0) {$2$};
  \ifnum#2=3
   \node[selected] (v3) at (.87,.65) {$3$};
   \node[vertex] (v4) at (.87,-.65) {$4$};
  \else
   \node[vertex] (v3) at (.87,.65) {$3$};
   \node[selected] (v4) at (.87,-.65) {$4$};
  \fi
  \node[vertex] (v5) at (-.87,.65) {$5$};
  \node[vertex] (v6) at (-.87,-.65) {$6$};
  \foreach \j in {2,#2}{
   \ifnum#1=0
    \draw[edge] (v1) to[bend left=14] (v\j);
    \draw[edge] (v\j) to[bend left=14] (v1);
   \else
    \draw[edge] (v\j) -- (v1);
   \fi
  }
  \ifnum#1=0
   \node[font=\footnotesize,text=lowerActive] at (0,-1.12) {SCC $\{1,2,#2\}$};
  \else
   \node[font=\footnotesize,text=lowerActive] at (0,-1.12) {$\Pa(1)=\{2,#2\}$};
  \fi
 \end{scope}
}
\resizebox{\linewidth}{!}{%
\begin{tikzpicture}
 \node[heading,text=lowerSCC] at (3.45,0) {(a) SCC membership};
 \node[heading,text=lowerParent] at (11.35,0) {(b) External-parent selection};
 \lowerExample{0}{3}{1.8}{-1.12}
 \lowerExample{0}{4}{5.1}{-1.12}
 \lowerExample{1}{3}{9.7}{-1.12}
 \lowerExample{1}{4}{13.0}{-1.12}
 \draw[rule] (7.4,.2) -- (7.4,-3.04);
 \draw[rule] (0,-2.46) -- (6.9,-2.46);
 \draw[rule] (7.9,-2.46) -- (14.8,-2.46);
 \node[detail] at (3.45,-2.80)
    {$N_s=\binom{p-1}{s-1}$ SCC partitions; edge weights $\propto(s-1)^{-1/2}$};
 \node[detail] at (11.35,-2.80)
    {$N_d=\binom{p-1}{d}$ parent sets; edge weights $\propto d^{-1/2}$};
 \draw[rule] (0,-3.14) -- (14.8,-3.14);
 \node[detail,text=lowerSCC] at (3.45,-3.57)
    {Sample cost:\; $\Omega\!\left(s\log\frac{ep}{s}\right)$};
 \node[detail,text=lowerParent] at (11.35,-3.57)
    {Sample cost:\; $\Omega\!\left(d\log\frac{ep}{d}\right)$};
\end{tikzpicture}%
}
\caption{Two lower-bound constructions at $p=6$: (a) SCC size $s=3$;
(b) parent count $d=2$.}
\label{fig:scc-lower-bound}
\end{figure}

Establishing the SCC term requires varying the partition itself, since changing internal edges can leave the condensation unchanged.
Figure~\ref{fig:scc-lower-bound}(a) varies SCC membership through reciprocal connections to a fixed center, with no edges between SCCs (Lemma~\ref{lem:star-packing}).
Figure~\ref{fig:scc-lower-bound}(b) varies the parents of one node in a DAG, keeping the SCC partition fixed (Lemma~\ref{lem:parent-packing}).
In both families, scaling the connections limits the information provided by each observation, yielding the costs through Fano's inequality.

\section{Population principles for exact condensation recovery}
\label{sec:population}

\subsection{Block exogeneity and unknown root SCCs}

Discovering a root SCC requires identifying its unknown boundary.
Its nodes need not be individually exogenous, and ancestry alone does not distinguish the SCC from larger ancestral sets.

A set is ancestral if it contains every parent, equivalently every ancestor, of each of its nodes.
Let $F\subseteq V$ be the already recovered ancestral set and $R=V\setminus F$.
An SCC is called a root SCC if its node in the condensation has no incoming edge.
For candidate blocks in $R$, ancestry and root status refer to $G[R]$; a current root SCC may therefore have external parents in $F$.

For centered random vectors $Y$ and $W$ with $\operatorname{Cov}(W)\succ0$, define the population least-squares projection by
\[
    \operatorname{Proj}(Y\mid W)
    :=
    \operatorname{Cov}(Y,W)
    \operatorname{Cov}(W)^{-1}W.
\]
For disjoint index sets $Q,C,D\subseteq V$, with $C\ne\varnothing$, define
\[
\begin{aligned}
    e_C(Q)
    :=X_C-\operatorname{Proj}(X_C\mid X_Q),\qquad
    r_{Q,C}(D)
    :=e_D(Q)-\operatorname{Proj}\bigl(e_D(Q)\mid e_C(Q)\bigr).
\end{aligned}
\]
We denote their coordinates by $e_c(Q)$ for $c\in C$ and $r_{Q,C}(\ell)$ for $\ell\in D$.
Projection onto an empty predictor set is zero.
The predictor covariance matrices are positive definite because $\Sigma\succ0$.

We first take $Q=F$.
Since $F$ is ancestral, $B_{F,R}=0$, and the structural equations give
\begin{equation}
    X_R=H_RX_F+(I-B_{R,R})^{-1}\eps_R,
    \qquad
    H_R=(I-B_{R,R})^{-1}B_{R,F}.
\label{eq:block-solution}
\end{equation}
The vector $X_F$ depends only on $\eps_F$, so $\eps_R\indep X_F$.
The population projection onto $X_F$ is therefore $H_RX_F$, so
\begin{equation}
    e_R(F)=(I-B_{R,R})^{-1}\eps_R.
\label{eq:induced-residual-model}
\end{equation}
Thus $e_R(F)$ follows the induced cyclic LiNG model on $R$, so it suffices to characterize its root SCCs.

\begin{theorem}[Block exogeneity]
\label{thm:block-exogeneity}
Suppose the cyclic LiNG model~\eqref{eq:sem} satisfies Assumption~\ref{ass:PI}. Let $F\subseteq V$ be an ancestral set and $R=V\setminus F$. 
For every nonempty $C\subseteq R$, the following are equivalent:
\begin{enumerate}
  \item[(i)] $C$ is ancestral in $G[R]$;
  \item[(ii)] $e_C(F)\indep r_{F,C}(R\setminus C)$;
  \item[(iii)] $e_c(F)\indep r_{F,C}(\ell)$ for all
  $c\in C$ and $\ell\in R\setminus C$.
\end{enumerate}
\end{theorem}

\begin{corollary}[Root SCC characterization]
\label{cor:minimal-scc}
Under the conditions of Theorem~\ref{thm:block-exogeneity}, with $R\ne\varnothing$, the inclusion-minimal nonempty sets $C\subseteq R$ satisfying condition~(ii) are exactly the root SCCs of $G[R]$.
\end{corollary}

After adjustment for $F$, the residual of an ancestral block $C$ involves only its own noises, and regressing the adjusted complement on $e_C(F)$ removes the effects transmitted from $C$.
The resulting residual vectors depend on disjoint sets of independent noises, giving $(i)\Rightarrow(ii)$.
Marginalization gives $(ii)\Rightarrow(iii)$.
Conversely, the Darmois--Skitovitch theorem \citep{darmois1953analyse,skitovitch1953property}, combined with a rank argument based on principal invertibility, shows that any incoming edge to $C$ in $G[R]$ forces $e_c(F)$ and $r_{F,C}(\ell)$ to be dependent for some $c\in C$ and $\ell\in R\setminus C$, giving $(iii)\Rightarrow(i)$.
The equivalence of $(ii)$ and $(iii)$ allows block independence to be assessed through scalar independence tests.

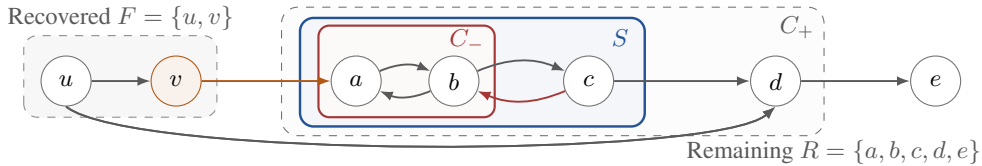
\begin{figure}[htbp]
\centering
\definecolor{sccblue}{RGB}{35,83,150}
\definecolor{parentorange}{RGB}{177,88,21}
\definecolor{rejectred}{RGB}{165,55,54}
\resizebox{.94\linewidth}{!}{%
\begin{tikzpicture}[
  vertex/.style={circle,draw=black!55,fill=white,minimum size=6.5mm,
    inner sep=0pt,font=\small},
  edge/.style={-{Latex[length=1.7mm]},line width=.7pt,black!65},
  every node/.style={inner sep=2pt}
]
\path[use as bounding box] (0,1.70) rectangle (12.90,3.68);
\draw[rounded corners=5pt,fill=black!3,draw=black!45,dashed]
  (.30,2.16) rectangle (2.81,3.22);
\draw[rounded corners=6pt,fill=black!1,draw=black!55,dashed]
  (3.65,1.92) rectangle (10.70,3.59);
\draw[rounded corners=5pt,fill=sccblue!5,draw=sccblue,line width=.9pt]
  (3.89,2.04) rectangle (8.37,3.45);
\draw[rounded corners=4pt,fill=rejectred!3,draw=rejectred,line width=.8pt]
  (4.13,2.16) rectangle (6.42,3.35);
\node[font=\small,text=black!65] at (1.56,3.45) {Recovered $F=\{u,v\}$};
\node[font=\small,text=rejectred] at (6.06,3.16) {$C_-$};
\node[font=\small,text=sccblue] at (8.05,3.20) {$S$};
\node[font=\small,text=black!65] at (10.35,3.33) {$C_+$};

\node[vertex] (u) at (.85,2.63) {$u$};
\node[vertex,draw=parentorange,fill=parentorange!8] (v) at (2.28,2.63) {$v$};
\node[vertex] (a) at (4.62,2.63) {$a$};
\node[vertex] (b) at (5.89,2.63) {$b$};
\node[vertex] (c) at (7.64,2.63) {$c$};
\node[vertex] (d) at (10.07,2.63) {$d$};
\node[vertex] (e) at (12.15,2.63) {$e$};
\draw[edge] (u) -- (v);
\draw[edge,parentorange] (v) -- (a);
\draw[edge] (a) to[bend left=23] (b);
\draw[edge] (b) to[bend left=23] (a);
\draw[edge] (b) to[bend left=23] (c);
\draw[edge,rejectred] (c) to[bend left=23] (b);
\draw[edge] (c) -- (d);
\draw[edge] (d) -- (e);
\draw[edge] (u.south) .. controls (1.70,1.62) and (9.30,1.62) .. (d.south);
\draw[edge] (u.south)
  .. controls (1.70,1.62) and (9.30,1.62) .. (d.south);

\node[
  anchor=north east,
  font=\footnotesize,
  text=black!65,
  inner sep=0pt
] at (12.75,1.88) {Remaining $R=\{a,b,c,d,e\}$};
\end{tikzpicture}%
}
\caption{Illustration of block exogeneity and sparse adjustment with overlapping feedback cycles.}
\label{fig:population-principles}
\end{figure}

In Figure~\ref{fig:population-principles}, adjustment for $F=\{u,v\}$ removes their upstream effects, so ancestry is evaluated in $G[R]$.
Although $C_-=\{a,b\}$ contains the cycle $a\leftrightarrow b$, the red edge $c\to b$ enters from outside $C_-$.
Because of this incoming edge, $e_b(F)$ and $r_{F,C_-}(c)$ share contributions from $\varepsilon_b$ and $\varepsilon_c$ and are dependent, so $C_-$ fails condition~(ii).
For $S=\{a,b,c\}$, $e_S(F)$ depends only on $\varepsilon_a,\varepsilon_b,\varepsilon_c$.
Regressing the adjusted variables $d,e$ on $e_S(F)$ removes the effects transmitted from $S$, leaving residuals driven only by $\varepsilon_d,\varepsilon_e$.
These residuals are independent of $e_S(F)$, so $S$ satisfies condition~(ii).

The larger candidate $C_+=S\cup\{d\}$ also has no incoming edge in $G[R]$ and satisfies condition~(ii), but is not inclusion-minimal because its proper subset $S$ satisfies the same condition.
In contrast, every nonempty proper subset of $S$ has an incoming edge from the rest of $S$ and fails the condition, so $S$ is inclusion-minimal.
Corollary~\ref{cor:minimal-scc} therefore identifies $S$ as a root SCC.

GroupLiNGAM characterizes exogenous sets through block residual independence under correlation faithfulness \citep{kawahara2010grouplingam}.
Our block-exogeneity principle shows that independence of every pair of coordinates, one from each of $e_C(F)$ and $r_{F,C}(R\setminus C)$, already suffices to establish ancestry without correlation faithfulness, under principal invertibility and independent structural errors.
This characterization identifies the root SCCs of $G[R]$ as exactly the inclusion-minimal nonempty blocks satisfying these scalar independence conditions.

\subsection{Sparse adjustment and external-parent identification}

Theorem~\ref{thm:block-exogeneity} uses the full recovered set $F$, whose size grows over successive recovery rounds.
To keep regression dimensions small, we seek a small subset $P\subseteq F$ that still permits root-SCC detection and external-parent identification.

For nonempty $C\subseteq R$ and $P\subseteq F$, we use two conditions.
$\mathsf{UP}$ checks whether adjustment for $P$ yields the same residual for $C$ as adjustment for all of $F$.
Once this holds, $\mathsf{OUT}$ checks whether $C$ is ancestral in $G[R]$.
\begin{align}
    e_c(P)&\indep X_k \qquad\forall c\in C,\ k\in F, \tag{\(\mathsf{UP}\)}\label{eq:up}\\
    e_c(P)&\indep r_{P,C}(\ell)
    \qquad\forall c\in C,\ \ell\in R\setminus C.\tag{\(\mathsf{OUT}\)}
\label{eq:out}
\end{align}
We call $(C,P)$ passing if it satisfies \eqref{eq:up} and \eqref{eq:out}, and write $\mathsf{PASS}(C,P)=1$ for passing pairs and zero otherwise.
We call $C$ passing when such a $P$ exists.

\begin{proposition}[Sparse adjustment and external-parent identification]
\label{prop:sparse-passing}
Under the conditions of Theorem~\ref{thm:block-exogeneity}, for every
$\varnothing\ne C\subseteq R$ and every $P\subseteq F$, the pair $(C,P)$ satisfies
\eqref{eq:up}--\eqref{eq:out} if and only if $C$ is ancestral in $G[R]$ and
$P_{\rm ext}(C)\subseteq P$.  Consequently, for an ancestral block $C$,
$P_{\rm ext}(C)$ is the unique minimum-cardinality passing adjustment.
\end{proposition}

Covariance alone suffices to verify $\mathsf{UP}$, even without Gaussianity.
Linearity and independent components of $\varepsilon$ ensure that any remaining upstream effect is detectable through covariance with $X_F$:
\[
    \mathsf{UP}
    \quad\Longleftrightarrow\quad
    \operatorname{Cov}\bigl(e_C(P),X_F\bigr)=0
    \quad\Longleftrightarrow\quad
    e_C(P)=e_C(F).
\]
Once this holds, $\mathsf{OUT}$ is equivalent to ancestry in $G[R]$, even if upstream effects persist in $R\setminus C$.
Both conditions use regressions only on $P$ or $P\cup C$.
Variables in $F\setminus P$ enter only the $\mathsf{UP}$ covariance checks and are not included as regression predictors.

To see why the minimum adjustment identifies external parents, suppose $C$ is ancestral in $G[R]$.
Then
\[
    X_C
    =
    (I-B_{C,C})^{-1}B_{C,F}X_F+e_C(F).
\]
The invertible factor $(I-B_{C,C})^{-1}$ preserves which columns of $B_{C,F}$ are nonzero.
Feedback may change individual coefficients, but each external parent still appears in at least one of the block's regression equations.
Since $\operatorname{Cov}(X_F)\succ0$, these regression coefficients
are unique, so an omitted external parent cannot be replaced
by other variables in $F$.
Including all external parents in $P$, on the other hand, reproduces the regression of $X_C$ on $X_F$, giving $e_C(P)=e_C(F)$.
Therefore the unique minimum-cardinality passing adjustment
is exactly $P_{\rm ext}(C)$.

In Figure~\ref{fig:population-principles}, adjusting for $u$ leaves the noise $\varepsilon_v$ in $S$, whereas $e_S(\{v\})=e_S(F)$.
The complement residuals retain contributions from $u$ but remain independent of $e_S(\{v\})$:
\[
    r_{\{v\},S}(d)=B_{d,u}e_u(\{v\})+\varepsilon_d
    \quad\indep\quad e_S(\{v\}).
\]
The residual $r_{\{v\},S}(e)$ likewise contains a contribution from $u$ transmitted through $d$ and is independent of $e_S(\{v\})$.
Thus $\{v\}$ satisfies $\mathsf{UP}$ and $\mathsf{OUT}$, while $\varnothing$ and $\{u\}$ fail $\mathsf{UP}$.
Proposition~\ref{prop:sparse-passing} identifies the minimum adjustment $\{v\}$ as $P_{\rm ext}(S)$, recovering the incoming condensation edge $\{v\}\to S$ with adjustment for $v$ alone rather than all of $F$.

\subsection{Exact condensation recovery}

Algorithm~\ref{alg:population} applies these principles sequentially: adding a root SCC of $G[R]$ to $F$ yields another ancestral set, so the characterizations continue to apply.
A passing pair of minimum total size consists of a root SCC and its exact external parent set.
Algorithm~\ref{alg:population} recovers the condensation by repeatedly selecting the first passer in increasing order of $|C|+|P|$.

\begin{algorithm}[H]
\caption{BlockExo: exact condensation recovery}
\label{alg:population}
\begin{algorithmic}[1]
\Require Search bounds $s\ge s_{\max}$ and $d\ge d_B$
\State $F\gets\varnothing$;
       $\widehat\Pi\gets\varnothing$;
       $\widehat E^{\rm sc}\gets\varnothing$
\While{$F\ne V$}
    \State $R\gets V\setminus F$; $S\gets\varnothing$
    \For{$t=1,\ldots,\min\{s,|R|\}+\min\{d,|F|\}$}
        \ForAll{$C\subseteq R$, $P\subseteq F$ with $1\le|C|\le s$, $|P|\le d$, $|C|+|P|=t$, in increasing $|C|$}
            \If{$\mathsf{PASS}(C,P)=1$}
                \State $(S,\widehat P(S))\gets(C,P)$; stop all search loops
            \EndIf
        \EndFor
    \EndFor
    \State \textbf{if} $S=\varnothing$ \textbf{then return} failure

    \State $\displaystyle
    \widehat E^{\rm sc}\gets\widehat E^{\rm sc}
    \cup
    \left\{
    T\to S:
    T\in\widehat\Pi,\
    T\cap\widehat P(S)\ne\varnothing
    \right\}$
    \State $\widehat\Pi\gets\widehat\Pi\cup\{S\}$;
           $F\gets F\cup S$
\EndWhile
\State \Return $\widehat G ^{\rm sc}=(\widehat\Pi,\widehat E^{\rm sc})$
\end{algorithmic}
\end{algorithm}

\begin{theorem}[Population recovery]
\label{thm:population-recovery}
Suppose the cyclic LiNG model~\eqref{eq:sem} satisfies Assumption~\ref{ass:PI}, and $s\ge s_{\max}$ and $d\ge d_B$.
Algorithm~\ref{alg:population} exactly recovers the condensation $G^{\rm sc}$.
\end{theorem}


\paragraph{Search bounds.}
The procedure does not require the exact values of $s_{\max}$ and $d_B$.
At the population level, the procedure completes if and only if $s\ge s_{\max}$ and $d\ge d_B$ (Proposition~\ref{prop:population-insufficient-bounds}).
With valid search bounds and correct decisions, every regression uses at most $s_{\max}+d_B$ predictors, and the procedure examines $p^{O(s_{\max}+d_B)}$ candidate pairs in total; neither bound depends on $|F|$, $s$, or $d$.
In finite samples, estimation error can cause failure even with valid search bounds, but these bounds still limit the total number of candidate pairs examined to $p^{O(s+d)}$.
Appendix~\ref{app:search-bound-sensitivity} compares sample costs under exact and larger valid search bounds, confirming exact recovery under both.

\section{Finite-sample recovery and structural optimality}
\label{sec:finite}

We establish a sample bound for exact condensation recovery by Algorithm~\ref{alg:population} with any valid search bounds.
The sample bound depends on the actual SCC size and external-parent count and matches the structural dependence of the lower bound in Theorem~\ref{thm:matching-lower}.

\subsection{Scores and separation}
\label{sec:scores}

The population characterization allows us to assess $\mathsf{UP}$ through covariance and $\mathsf{OUT}$ through residual dependence.
For random variables $Y,Z$ with positive finite variances, let $\mathcal D(Y,Z)$ be a nonnegative dependence functional evaluated after centering and scaling to unit variance, satisfying $\mathcal D(Y,Z)=0$ if and only if $Y\indep Z$.
Taking empty maxima as zero, define for an ancestral set $F$, with $R=V\setminus F$, and a candidate pair $(C,P)$,
\[
    D_F^{\rm up}(C,P)
    :=\max_{c\in C,\ k\in F}
    \left|\operatorname{Cov}\bigl(e_c(P),X_k\bigr)\right|,\quad
    D_F^{\rm out}(C,P)
    :=\max_{c\in C,\ \ell\in R\setminus C}
    \mathcal D\bigl(e_c(P),r_{P,C}(\ell)\bigr).
\]

Then
\[
    D_F^{\rm up}(C,P)=D_F^{\rm out}(C,P)=0
    \quad\Longleftrightarrow\quad
    \mathsf{PASS}(C,P)=1.
\]
By Proposition~\ref{prop:sparse-passing}, this is equivalent to $C$ being ancestral in $G[R]$ and $P_{\rm ext}(C)\subseteq P$.

We define the candidate class
over all ancestral recovery states and the population separation margins:
\[
    \mathcal A:=
    \left\{
    (F,C,P):
        F\subseteq V\text{ is ancestral in }G,\;
        \varnothing\ne C\subseteq V\setminus F,\;
        P\subseteq F,\; |C|+|P|\le s_{\max}+d_B
    \right\}.
\]
\begin{align}
 &\Delta_{\rm par}=\min_{(F,C,P)\in\mathcal A}
 \{D_F^{\rm up}(C,P):e_C(P)\ne e_C(F)\},\nonumber\\
 &\Delta_{\mathcal D}=\min_{(F,C,P)\in\mathcal A}
 \{D_F^{\rm out}(C,P):e_C(P)=e_C(F),\;
 C\text{ is non-ancestral in }G[V\setminus F]\}.
 \label{eq:dep-margin}
\end{align}
Here, empty minima are $+\infty$, with reciprocal zero.
The margin $\Delta_{\rm par}$ quantifies the upstream effect remaining after an insufficient adjustment, and $\Delta_{\mathcal D}$ quantifies residual dependence when adjustment is sufficient but the block is non-ancestral.

Each insufficient adjustment has a positive $\mathsf{UP}$ score.
For a sufficiently adjusted non-ancestral block, Proposition~\ref{prop:sparse-passing} guarantees at least one dependent scalar pair, so the maximum defining its $\mathsf{OUT}$ score is positive.
Since $\mathcal A$ is finite for a fixed model, both margins are positive.
This pointwise positivity does not imply uniform separation over a model class.
Proposition~\ref{prop:weak-cycle} makes this distinction concrete:
as the coefficients of a two-variable cycle vanish, its distribution approaches that of the edgeless model while its SCC partition remains different.

\subsection{Main recovery guarantee}
\label{sec:main-recovery}

We use the tail and covariance conditions in Assumption~\ref{ass:sg-ev}, as is common in high-dimensional linear SEM analysis \citep{gao2022optimal,ghoshal2018learning}.
Sub-Gaussianity provides the concentration rate used below; consistency can also hold under weaker tail conditions.
The covariance eigenvalue bounds control regression estimation and residual standardization.
Every regression needed for candidates in $\mathcal A$ involves at most $s_{\max}+d_B$ predictors and one response, so the analysis only requires uniform eigenvalue bounds for principal submatrices of size at most $s_{\max}+d_B+1$.
The global condition is a concise sufficient assumption.

We fit the regressions and compute the empirical statistics using all $n$ observations, following the construction in Appendix~\ref{app:finite}.
For each scalar pair in $\mathsf{OUT}$, $\mathcal D$ is evaluated on the population residuals, whereas $\widehat{\mathcal D}$ is computed from the corresponding fitted residuals.
The uniform estimation bound below uses $\ell(t)$ defined in Appendix~\ref{app:finite} for confidence level $1-\delta$, where $t=|C|+|P|$.

\begin{assumption}[Uniform estimation error for dependence scores]
\label{ass:uniform-approximation}
There exists $c_{\mathcal D}>0$ such that, uniformly over models satisfying
Assumption~\ref{ass:sg-ev} with fixed $(K,\lambda)$, for every
$1\le t\le p$, $0<\delta<1$, and $n\ge c_{\mathcal D}\ell(t)$,
with probability at least $1-\delta/3$,
\[
 \sup_{\substack{(F,C,P):\,|C|+|P|=t\\
                  c\in C,\ \ell\in V\setminus(F\cup C)}}
 \left|\widehat{\mathcal D}\bigl(e_c(P),r_{P,C}(\ell)\bigr)
       -\mathcal D\bigl(e_c(P),r_{P,C}(\ell)\bigr)\right|
 \le c_{\mathcal D}\sqrt{\frac{\ell(t)}n},
\]
where the supremum ranges over ancestral $F\subseteq V$,
$\varnothing\ne C\subseteq V\setminus F$, and $P\subseteq F$.
\end{assumption}

This bound accounts jointly for errors in estimating the residuals and their dependence.
The constant $c_{\mathcal D}$ may depend on $(K,\lambda)$ and the chosen estimator, including its fixed tuning parameters, but not on $p,s_{\max},d_B,n,\delta$.
Lemmas~\ref{lem:k4-normalized} and~\ref{lem:hsic-admissible} show that the distance-covariance and Gaussian-HSIC estimators, respectively, satisfy this assumption.

Using this estimation rate, we set $\tau_{\rm up}(t)=\kappa_{\rm up}\sqrt{\ell(t)/n}$ and $\tau_{\rm out}(t)=\kappa_{\rm out}\sqrt{\ell(t)/n}$.
Let $\widehat D_F^{\rm up}$ and $\widehat D_F^{\rm out}$ denote the empirical counterparts of the population scores.
Define $\widehat{\mathsf{PASS}}(C,P)=1$ if
\[
\widehat D_F^{\rm up}(C,P)\le\tau_{\rm up}(t)
\quad\text{and}\quad
\widehat D_F^{\rm out}(C,P)\le\tau_{\rm out}(t),
\]
and zero otherwise.

For $0<\delta<1$, define the complexity term used in the sample bound below:
\[
 L:=s_{\max}\log\frac{ep}{s_{\max}}
   +d_B\log\frac{ep}{d_B}+\log p+\log s_{\max}
   +\log\frac{64}{\delta},
\]
with $d_B\log(ep/d_B)=0$ when $d_B=0$.
The first two terms describe the block and adjustment searches.
Although $\mathcal A$ ranges over ancestral recovery states, each scalar statistic depends only on $C$, $P$, and the tested variables, without further dependence on $F$.
Thus the number of ancestral states contributes no additional counting term.
Lemma~\ref{lem:count} gives $\ell(\min\{p,s_{\max}+d_B\})\asymp L$,
connecting the candidate-dependent thresholds to the structural sample bound below.

\begin{theorem}[Achievability of structural sample complexity]
\label{thm:finite-recovery}
Consider the cyclic LiNG model~\eqref{eq:sem}.
Suppose Assumptions~\ref{ass:PI}, \ref{ass:sg-ev}, and
\ref{ass:uniform-approximation} hold, with $\mathcal D$ as in
Section~\ref{sec:scores}.
Let $\widehat G^{\rm sc}$ be the output of
Algorithm~\ref{alg:population} with $s\ge s_{\max}$,
$d\ge d_B$, and $\widehat{\mathsf{PASS}}$ in place of
$\mathsf{PASS}$.
There exist fixed threshold multipliers
$\kappa_{\rm up},\kappa_{\rm out}>0$ and a constant $c_0>0$,
such that, for every $0<\delta<1$,
\begin{equation}
 n\ge c_0L
 \max\{1,\Delta_{\rm par}^{-2},\Delta_{\mathcal D}^{-2}\}
 \quad\Longrightarrow\quad
 \Pr(\widehat G^{\rm sc}=G^{\rm sc})\ge1-\delta.
 \label{eq:recovery-probability}
\end{equation}
The constants $\kappa_{\rm up},\kappa_{\rm out},c_0$ depend only on $K,\lambda,c_{\mathcal D}$.
\end{theorem}

The choice of $\mathcal D$ affects $\Delta_{\mathcal D}$ and $c_{\mathcal D}$, but not the structural dependence of the sample bound.

\subsection{Structural optimality}
\label{sec:structural-optimality}

We establish a uniform recovery guarantee over model classes whose separation margins are bounded away from zero.
Fix $\mathcal D$ as in Section~\ref{sec:scores}. For
$\gamma_{\rm par},\gamma_{\mathcal D}>0$, define
\[
 \begin{aligned}
 \mathcal M^{\rm sep}_{p,s,d}(K,\lambda;\gamma_{\rm par},\gamma_{\mathcal D})
 :=\bigl\{M\in\mathcal M_{p,s,d}(K,\lambda):
  \Delta_{\rm par}(M)\ge\gamma_{\rm par},\;
 \Delta_{\mathcal D}(M)\ge\gamma_{\mathcal D}\bigr\}.
 \end{aligned}
\]
Here $(s,d)$ denotes the structural bounds defining the model class.
The margins are evaluated over $\mathcal A$ with $s_{\max}+d_B$ replaced by $s+d$.
In the corollary below, $L$ is taken at $(s,d)$ instead of $(s_{\max},d_B)$.

\begin{corollary}[Uniform recovery and optimality]
\label{cor:uniform-recovery-optimality}
Fix $1\le s\le p/2$ and $0\le d\le p/2$.
Suppose Assumption~\ref{ass:uniform-approximation} holds over
$\mathcal M^{\rm sep}_{p,s,d} (K,\lambda;\gamma_{\rm par},\gamma_{\mathcal D})$.
Let $\widehat G^{\rm sc}$ be the output of Algorithm~\ref{alg:population} with search bounds at least $s$ and $d$, respectively, using $\widehat{\mathsf{PASS}}$ and the threshold multipliers of Theorem~\ref{thm:finite-recovery}.
With $c_0$ as in that theorem, for every $0<\delta<1$,
\begin{equation}
 n\ge c_0L
 \max\{1,\gamma_{\rm par}^{-2},\gamma_{\mathcal D}^{-2}\}
 \quad\Longrightarrow\quad
 \sup_{M\in\mathcal M^{\rm sep}_{p,s,d}
 (K,\lambda;\gamma_{\rm par},\gamma_{\mathcal D})}
 P_M^n\!\left(\widehat G^{\rm sc}\ne G^{\rm sc}(M)\right)
 \le\delta.
 \label{eq:main-rate}
\end{equation}
Under Assumption~\ref{ass:uniform-approximation}, the separated class of Proposition~\ref{prop:fixed-separation-lower} has minimax sample complexity $\Theta(\log p)$ for fixed structural bounds and fixed target error probability $0<\delta<1/2$.
\end{corollary}

For fixed $\delta,K,\lambda,\gamma_{\rm par},\gamma_{\mathcal D}$, the sample bound in~\eqref{eq:main-rate} matches the structural dependence $s\log(ep/s)+d\log(ep/d)$ of the lower bound in Theorem~\ref{thm:matching-lower}.
For fixed structural bounds, we further establish minimax sample complexity
$\Theta(\log p)$ on the separated model class of Proposition~\ref{prop:fixed-separation-lower}.
Larger valid search bounds also preserve this optimal rate: uniform accuracy over the candidate class ensures that each round selects a passing pair with $|C|+|P|\le s_{\max}+d_B$.

For DAGs, $s=1$ and $d$ bounds the maximum indegree, so the guarantee applies to exact DAG recovery with structural dependence $\log(ep)+d\log(ep/d)\asymp d\log(ep/d)$ for $d\ge1$.
This agrees with the optimal dimension and sparsity dependence for LiNGAMs in \citet{oh2025optimal}.

\section{Simulations}
\label{sec:experiments}

We empirically examine sample cost and compare sample efficiency across methods using sub-Gaussian linear SEMs.
We consider disjoint 2-cycles for $s_{\max}=2$ and overlapping 2-cycles sharing a node for $s_{\max}=3,4$.
Each point reports the exact-recovery rate across 20 replicates.
Implementation details and additional experiments are provided in Appendix~\ref{app:empirical}.

\subsection{Structural sample costs}

Figure~\ref{fig:main-evidence} provides empirical support for the structural and separation dependence of our sample bound.
Recovery curves that differ on the raw sample-size axis align more closely after normalization by the corresponding structural and separation factors.
The alignment is stronger in panel~(b), which varies dimension
at fixed SCC size, than in panel~(a), which varies SCC size
at fixed dimension.
Appendix~\ref{app:margin-audit} details this normalization.

\begin{figure}[t]
\centering
\includegraphics[width=\linewidth]{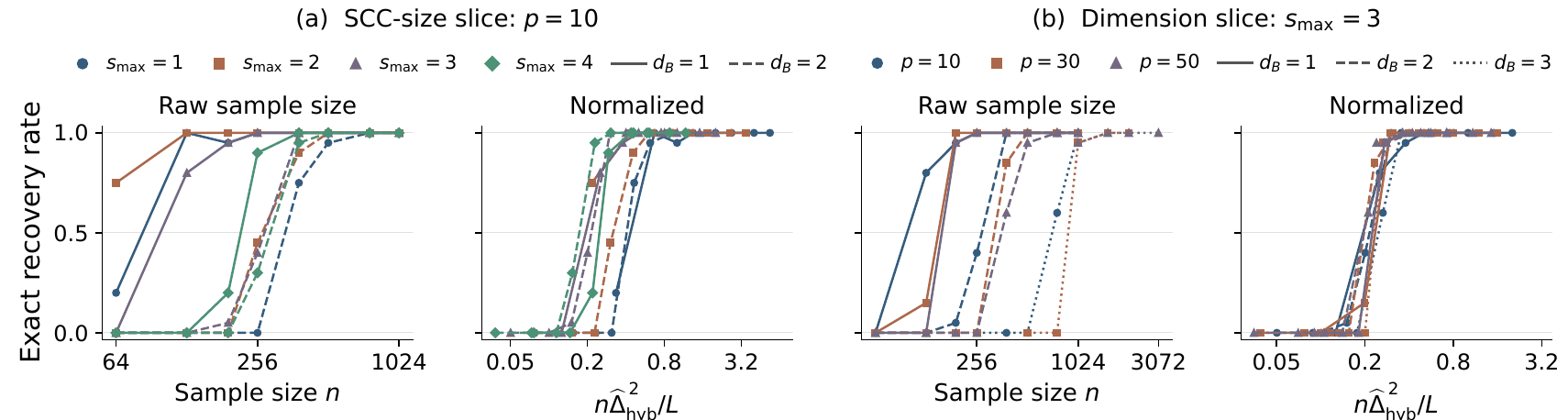}
\caption{Structural sample costs for exact condensation recovery with
exact search bounds.}
\label{fig:main-evidence}
\end{figure}

\subsection{Sample efficiency comparison across methods}

BlockExo achieves sample-efficient exact condensation recovery across disjoint and overlapping cycle structures (Figure~\ref{fig:fullsample-baseline-m4-preview}).
At $p=50$, we compare it with Coarsening \citep{madaleno2026coarsening} at fixed thresholds $0.1$ and $0.2$, DisjointCycles \citep{drton2025disjoint}, and StableSpIn \citep{misiakos2026stablespin}.
BlockExo uses common search bounds $(s,d)=(3,2)$ across all four settings.
DisjointCycles and StableSpIn achieve no exact recoveries over the evaluated grid.
GroupLiNGAM \citep{kawahara2010grouplingam} is infeasible under our
memory budget (Appendix~\ref{app:fullsample-protocol}).

BlockExo achieves sample-efficient exact recovery even without the exact values of $s_{\max}$ and $d_B$.
With $d_B=1$, it achieves exact recovery in all 20 replicates at $n=384$ for disjoint cycles and $n=512$ for overlapping cycles, compared with 12 and 13 of 20 for Coarsening $0.2$, respectively.
With $d_B=2$, BlockExo reaches 20 of 20 at $n=512$ for disjoint cycles, compared with 14 of 20 for Coarsening $0.2$, and at $n=768$ for overlapping cycles, matching Coarsening $0.2$.
BlockExo's median runtimes range from 8 to 34 seconds across the four settings: longer than Coarsening, shorter than or comparable to StableSpIn, and comparable to DisjointCycles on disjoint cycles (Appendix~\ref{app:runtime}).


\begin{figure}[t]
\centering
\includegraphics[width=\linewidth]{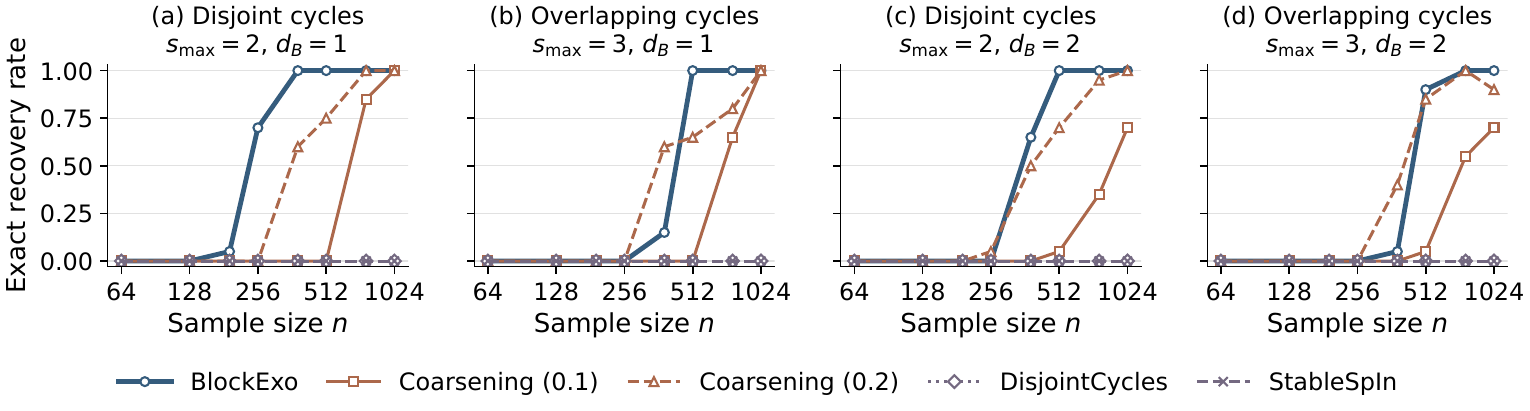}
\caption{Exact condensation recovery comparison at $p=50$.}
\label{fig:fullsample-baseline-m4-preview}
\end{figure}

\section{Discussion}
\label{sec:discussion}

We establish information-theoretic lower bounds for exact condensation recovery and a procedure whose sample bound matches their structural dependence.
The lower bounds distinguish the unavoidable sample costs of SCC membership and external-parent selection.
BlockExo achieves this matching under suitable conditions, even when $s_{\max}$ and $d_B$ are unknown.
These necessary and achievable structural sample costs provide a concrete target for the design and analysis of methods for exact condensation recovery.

Block exogeneity provides a foundation for developing new algorithms for cyclic causal discovery.
Scalar residual independence and inclusion minimality identify root SCCs without a known partition or correlation faithfulness.
Sparse adjustment further permits identification of SCCs and their direct external parents without regressing on all previously recovered variables.
These characterizations open a path to new algorithms that are both sample-efficient and computationally efficient.

Characterizing the separation margins $\Delta_{\rm par}$ and
$\Delta_{\mathcal D}$ in terms of graph structure, coefficients, and noise distributions is an important next step in explaining sample cost.
Our sample bound distinguishes the structural costs from the additional difficulty captured by the separation margins.
Relating these margins to model properties would help explain differences in condensation recovery difficulty and identify conditions that guarantee uniform separation.

\subsection*{AI use statement}
We used generative AI tools to assist with drafting proofs, designing
simulation experiments, developing code for synthetic-data generation and method implementation, interpreting empirical results, preparing figures, and revising the manuscript.
The authors supplied key prior work, including references needed for the proofs, and guidance on the intended mathematical results and assumptions. The authors reviewed and revised the AI-assisted proofs, refining their mathematical arguments, exposition, and logical organization. The authors take responsibility for the final content of this work.

\subsection*{Ethics statement}
This work develops theoretical results for causal discovery and evaluates them using synthetically generated data.
It does not involve human participants or personal data.
Applications of the proposed recovery principles should assess the stated modeling assumptions, including causal sufficiency, before using the inferred structure to inform substantive decisions.

\subsection*{Reproducibility statement}
Sections~\ref{sec:setup}--\ref{sec:finite} specify the model assumptions and recovery procedures. Appendices~\ref{app:lower}--\ref{app:separation} provide proofs of the theoretical results, including the separation and optimality analysis. Appendix~\ref{app:empirical} describes the synthetic graph constructions, noise distributions, estimator and baseline settings, evaluation protocol, and reference calculations used for normalization. We will publicly release the implementation, experiment configurations, and scripts and data for reproducing the figures.

\bibliography{references}
\bibliographystyle{iclr2027_conference}

\appendix
\section{Related work: recovery targets and statistical guarantees}
\label{sec:related}

Table~\ref{tab:positioning} summarizes related LiNG methods.

\begin{table}[H]
\centering
\normalsize
\caption{Recovery targets, population principles, and statistical guarantees of related LiNG methods.}
\label{tab:positioning}
\vskip\baselineskip
\begin{tabularx}{\linewidth}{@{}l>{\hsize=.79\hsize\linewidth=\hsize\raggedright\arraybackslash}X>{\hsize=.96\hsize\linewidth=\hsize\raggedright\arraybackslash}X>{\hsize=1.25\hsize\linewidth=\hsize\raggedright\arraybackslash}X@{}}
\toprule
Method & Recovery target & Population basis & Statistical guarantees \\
\midrule
OptLiNGAM & sparse DAG & residual independence; parent selection & Structurally matching bounds: $d_{\rm in}\log(ep/d_{\rm in})$ \\
\addlinespace
GroupLiNGAM & block partition and order & exogenous-set residual independence & Population characterization \\
\addlinespace
DisjointCycles & cycle-disjoint structure & moment relations for root cycles & Structure recovery consistency \\
\addlinespace
StableSpIn & stable graph representative & stability; sparse-input conditions & MLE consistency \\
\addlinespace
Coarsening & condensation & ICA; SCC invariance & Sample bound: $p/(\beta_{\min}^{2}\sqrt{\delta})$ \\
\midrule
\textbf{BlockExo} & condensation & block exogeneity; sparse adjustment & Structurally matching bounds: $s\log(ep/s)+d\log(ep/d)$ \\
\bottomrule\end{tabularx}

\par\smallskip
{\justifying
The displayed rates omit confidence and distributional factors
for OptLiNGAM and BlockExo, and other model-dependent constants
for Coarsening.
The Coarsening bound uses threshold $\beta_{\min}/2$, where
$\beta_{\min}:=\min_{B_{ij}\ne0}|B_{ij}|$
\citep{madaleno2026coarsening}.\par}
\end{table}

\paragraph{Acyclic sample complexity.}
Our structural sample bound matches OptLiNGAM's dimension and sparsity
dependence for DAGs and adds the cost of identifying SCC membership
in cyclic graphs.
OptLiNGAM establishes optimal sample complexity for sparse DAG recovery
under sub-Gaussianity, without faithfulness or a known indegree
\citep{oh2025optimal}.
For DAGs, $s_{\max}=1$ and $d_B=d_{\rm in}$, the maximum indegree,
so our structural factor reduces to order
$d_{\rm in}\log(ep/d_{\rm in})$ for $d_{\rm in}\ge1$.

\paragraph{Population characterizations of exogenous blocks.}
GroupLiNGAM characterizes exogenous sets through block residual
independence, whereas our block-exogeneity principle establishes
ancestry through scalar residual independence.
GroupLiNGAM tests independence between a candidate variable subset
and the residual obtained by regressing the remaining variables on it;
recursive splitting identifies an ordered partition
\citep{kawahara2010grouplingam}.
Our characterization requires only independence of every scalar pair
between $e_C(F)$ and $r_{F,C}(R\setminus C)$ after ancestral adjustment.
Inclusion minimality then identifies root SCCs without a known partition.
Sparse adjustment further permits ancestry testing and direct-parent
identification using small adjustment sets
(Proposition~\ref{prop:sparse-passing}).

GroupLiNGAM assumes correlation faithfulness and permits dependent
external influences within a group, allowing within-group latent
confounding. Our characterization uses principal invertibility and
independent structural errors, without correlation faithfulness.
Related work estimates a causal order through residual independence
when the multivariate groups are specified in advance
\citep{entner2012groups}.

\paragraph{Recovery with disjoint cycles.}
DisjointCycles recovers structure when directed cycles share no nodes,
whereas our block-exogeneity characterization also applies to SCCs
with overlapping cycles.
For the disjoint-cycle setting, \citet{drton2025disjoint} characterize
distributional equivalence and use quadratic and cubic moment relations
to locate root cycles. Decorrelation and multivariate regression then
yield a block-topological order and a consistent structure-learning
procedure. Our approach identifies root SCCs through residual
independence and inclusion minimality.

\paragraph{Exact condensation recovery.}
Coarsening provides ICA-based condensation recovery, while our results
establish information-theoretic lower bounds for this target and attain
a structurally matching sample bound.
Early ICA-based work enumerates cyclic LiNG models compatible with an
observed distribution \citep{lacerda2008discovering}.
\citet[Theorem~4.4 and Corollary~4.5]{sharifian2025near} establish that
the SCC partition and condensation graph are invariant across
distribution-equivalent graphs, establishing identifiability of the
condensation from observational data.
Coarsening recovers the condensation without enumerating the
variable-level equivalence class; its finite-sample guarantee follows
from control of unmixing-matrix and edge-support estimation
\citep{madaleno2026coarsening}.
Our lower bounds distinguish the unavoidable costs of SCC membership
and external-parent selection (Theorem~\ref{thm:matching-lower}), and
our recovery guarantee matches their structural dependence under
regularity and separation conditions
(Corollary~\ref{cor:uniform-recovery-optimality}).

StableSpIn targets a stable variable-level graph representative under
sparse-input conditions and establishes MLE consistency
\citep{misiakos2026stablespin}; our target is the SCC partition and
all edges between components.
Distributional equivalence has also been characterized for LiNG models
with latent variables and cycles \citep{dai2026glvling}, whereas we
study condensation sample complexity under causal sufficiency.

\paragraph{Computational complexity.}
The exact GroupLiNGAM procedure examines exponentially many subsets in $p$
\citep{kawahara2010grouplingam}.
For OptLiNGAM, the distance-covariance implementation has cost
$O(n\log n\,p^{d_{\rm in}+1}d_{\rm in}^{4}(p-d_{\rm in}))$
when all independence decisions are correct \citep[Lemma~18]{oh2025optimal}.
With valid search bounds and correct decisions, BlockExo examines
$p^{O(s_{\max}+d_B)}$ candidate pairs and uses at most
$s_{\max}+d_B$ predictors in each regression, independently of the
input search bounds. Its empirical worst-case candidate count is
$p^{O(s+d)}$, with regressions and scalar tests for each pair.
Coarsening has cost $O(knp^2+p^3)$, where $k$ is the number of FastICA
iterations \citep[Theorem~2]{madaleno2026coarsening}.
StableSpIn's optimization costs $O(np^2+p^3)$ per iteration for a fixed
power-iteration count \citep{misiakos2026stablespin}.

\section{Structural lower-bound proofs}
\label{app:lower}

For a minimax lower bound, it suffices to establish the claim on a suitable
subclass of the model class. Throughout this section, we therefore restrict
attention to models whose structural noises are i.i.d.\ from the following
fixed distribution. Let $\phi_\sigma$ denote the $N(0,\sigma^2)$ density and
set
\[
 f(x)=\frac12\phi_{1/\sqrt2}\left(x-\frac1{\sqrt2}\right)
     +\frac12\phi_{1/\sqrt2}\left(x+\frac1{\sqrt2}\right).
\]
If $\eta\sim f$, then $\eta=Z+R/\sqrt2$ in distribution, where
$Z\sim N(0,1/2)$, $R$ is independent of $Z$, and
$\Pp(R=1)=\Pp(R=-1)=1/2$. Hence
$\E\eta=0$, $\Var(\eta)=1$,
$\E e^{t\eta}=e^{t^2/4}\cosh(t/\sqrt2)\le e^{t^2/2}$, and
$\E\eta^4=5/2\ne3$. Thus $f$ is centered, variance one, non-Gaussian,
and sub-Gaussian.
Writing $g=\log f$, we have
$g(x)=\mathrm{const}-x^2+\log\cosh(\sqrt2x)$ and
$g''(x)=-2\tanh^2(\sqrt2x)$, so $\|g''\|_\infty=2$.

For each coefficient matrix $B$, let $P_B$ denote the law of
\[
 X=(I-B)^{-1}\eps,
\]
and write $G^{\rm sc}(B)$ for the condensation of the graph determined by $B$.
We first establish common KL and regularity bounds, then construct separate packings for SCC membership and external-parent selection.
The generalized-Fano argument follows the sparse-DAG lower-bound framework of \citet{gao2022optimal} and \citet{oh2025optimal}.

\subsection{Common KL and regularity bounds}

\begin{lemma}[Local KL control]
\label{lem:kl-control}
For every constant $\rho_0\in(0,1)$ independent of dimension $p$ and every
zero-diagonal matrix $B$ satisfying $\|B\|_{\rm op}\le\rho_0$,
\[
 \operatorname{KL}(P_B\|P_0)
 \le
 \left\{
   \frac{1}{1-\rho_0}
   +\frac{1}{(1-\rho_0)^2}
 \right\}
 \|B\|_F^2.
\]
\end{lemma}

\begin{proof}
The change of variables $\eps=(I-B)x$ gives
$p_B(x)=|\det(I-B)|\prod_{i=1}^p f(((I-B)x)_i)$ and
$p_0(x)=\prod_{i=1}^p f(x_i)$.
For every $t\in[0,1]$, the bound $\|tB\|_{\rm op}<1$ implies
$\det(I-tB)\ne0$. Continuity of $t\mapsto\det(I-tB)$ and
$\det(I)=1$ therefore give $\det(I-B)>0$.

Let $\Gamma=(I-B)^{-1}-I$. Under $P_B$,
$X=\eps+\Gamma\eps$, and hence
\[
 \log\frac{p_B(X)}{p_0(X)}
 =\log\det(I-B)
   +\sum_{i=1}^p
   \{g(\eps_i)-g(\eps_i+(\Gamma\eps)_i)\}.
\]
Taylor's theorem and $-g''(x)=2\tanh^2(\sqrt2x)\le2$ give
\[
 g(\eps_i)-g(\eps_i+(\Gamma\eps)_i)
 \le -g'(\eps_i)(\Gamma\eps)_i+(\Gamma\eps)_i^2.
\]
Since $f(x),xf(x)\to0$ as $|x|\to\infty$, integration by parts gives
$\E[g'(\eps_i)]=0$ and $\E[g'(\eps_i)\eps_i]=-1$.
Together with independence, centering, and $\Cov(\eps)=I$, this yields
\begin{align*}
 \E[g'(\eps_i)\eps_j]
 =-\mathbf 1\{i=j\}, \quad
 -\sum_{i=1}^p\E[g'(\eps_i)(\Gamma\eps)_i]
 =\operatorname{tr}(\Gamma), \quad
 \E\|\Gamma\eps\|_2^2
 =\operatorname{tr}(\Gamma\Gamma^\top)
 =\|\Gamma\|_F^2.
\end{align*}
Taking expectations in the Taylor bound and applying these identities gives
\begin{equation}
 \operatorname{KL}(P_B\|P_0)
 \le \log\det(I-B)+\operatorname{tr}(\Gamma)
 +\|\Gamma\|_F^2.
 \label{eq:local-kl-intermediate}
\end{equation}

Since $\|B\|_{\rm op}<1$, the log-determinant and Neumann series give
\[
\log\det(I-B)
=-\sum_{k\ge1}\frac{\operatorname{tr}(B^k)}{k},
\qquad
\Gamma=\sum_{k\ge1}B^k.
\]
For $k\ge 1$, we have 
\begin{align*}
&\|B^k\|_F
\le \|B\|_F\|B\|_{\rm op}^{k-1}
\le \|B\|_F\rho_0^{k-1},\\
&|\operatorname{tr}(B^k)|
=|\langle B^\top,B^{k-1}\rangle_F|
\le\|B\|_F\|B^{k-1}\|_F.
\end{align*}
These bounds give
\begin{align*}
&\left|\log\det(I-B)+\operatorname{tr}(\Gamma)\right|
=\left|\sum_{k\ge2}\left(1-\frac1k\right)
  \operatorname{tr}(B^k)\right|
\le \|B\|_F^2\sum_{k\ge2}\rho_0^{k-2}
=\frac{\|B\|_F^2}{1-\rho_0},\\
&\|\Gamma\|_F
\le\sum_{k\ge1}\|B^k\|_F
\le\|B\|_F\sum_{k\ge1}\rho_0^{k-1}
=\frac{\|B\|_F}{1-\rho_0}.
\end{align*}
Substituting these bounds into \eqref{eq:local-kl-intermediate} yields
\begin{align*}
 \operatorname{KL}(P_B\|P_0)
 \le \frac{\|B\|_F^2}{1-\rho_0}
 +\frac{\|B\|_F^2}{(1-\rho_0)^2}
 =\left\{
   \frac1{1-\rho_0}
   +\frac1{(1-\rho_0)^2}
   \right\}\|B\|_F^2.
\end{align*}
\end{proof}

\begin{lemma}[Uniform packing regularity]
\label{lem:packing-regularity}
Under the fixed noise distribution above, suppose
$\|B\|_{\rm op}\le\rho_0<1$. Then
\[
 (1+\rho_0)^{-2} \le \lambda_{\min}(\Sigma)\le
 \lambda_{\max}(\Sigma)\le(1-\rho_0)^{-2},
\]
where
\[
 \Sigma=(I-B)^{-1}(I-B)^{-\top}.
\]
Moreover, there exists a universal constant $K>0$ such that, for every
deterministic $a\in\mathbb R^p$,
\[
 \|a^\top X\|_{\psiTwo}
 \le K\sqrt{a^\top\Sigma a}.
\]

\end{lemma}

\begin{proof}
For every $x\in\mathbb R^p$,
\[
 (1-\rho_0)\|x\|_2
 \le\|(I-B)x\|_2
 \le(1+\rho_0)\|x\|_2.
\]
Thus the singular values of $(I-B)^{-1}$ lie in
$[(1+\rho_0)^{-1},(1-\rho_0)^{-1}]$. Hence the eigenvalues of $\Sigma$
lie in the squared interval
$[(1+\rho_0)^{-2},(1-\rho_0)^{-2}]$.

For any deterministic $a\in\mathbb R^p$, let $c=(I-B)^{-\top}a$.
The moment-generating-function bound above implies
$\|\eps_i\|_{\psiTwo}\le C_0$ for a universal constant $C_0$.
The standard inequality for sums of independent centered sub-Gaussian
variables therefore gives
\[
 \|a^\top X\|_{\psiTwo}
 =\|c^\top\eps\|_{\psiTwo}
 \le C\Bigl(\sum_{i=1}^p c_i^2\|\eps_i\|_{\psiTwo}^2\Bigr)^{1/2}
 \le CC_0\|c\|_2
 =CC_0\sqrt{a^\top\Sigma a}.
\]
Taking $K:=CC_0$ proves the claim.
\end{proof}

For the remainder of this section, fix universal constants
$c_{\rm KL},K>0$ and $\lambda\ge1$ such that every $B$ with
$\|B\|_{\rm op}\le1/2$ satisfies
\[
 \operatorname{KL}(P_B\|P_0)\le c_{\rm KL}\|B\|_F^2,
 \qquad
 \lambda^{-1}\le\lambda_{\min}(\Sigma)
 \le\lambda_{\max}(\Sigma)\le\lambda,
\]
as well as the sub-Gaussian bound in
Lemma~\ref{lem:packing-regularity}. Such constants exist by the two lemmas
above.

\subsection{SCC and external-parent packings}

\begin{lemma}[Reciprocal-star SCC packing]
\label{lem:star-packing}
For every $2\le s\le p/2$ and every
$S\subseteq\{2,\ldots,p\}$ with $|S|=s-1$, define $B_S$ by
\[
 (B_S)_{1j}=(B_S)_{j1}
 =\frac{1}{2\sqrt{s-1}},
 \qquad j\in S,
\]
with all other entries equal to zero. Then:

\begin{enumerate}
 \item[(i)] The model defined by $B_S$ and $f$ belongs to
 $\mathcal M_{p,s,d}(K,\lambda)$ for every $d\ge0$. Its only nontrivial
 SCC is $\{1\}\cup S$, and $d_B=0$.

 \item[(ii)] Writing $N_s:=\binom{p-1}{s-1}$,
 the packing contains $N_s\ge3$ distinct condensations and
 \[
  \log N_s
  \ge \frac{\log2}{2(1+\log2)}s\log\frac{ep}{s}.
 \]

 \item[(iii)] The observation laws satisfy
 \[
  \max_S\operatorname{KL}(P_{B_S}^n\|P_0^n)
  \le\frac{c_{\rm KL}n}{2}.
 \]
\end{enumerate}
\end{lemma}

\begin{proof}
We first verify statement~(i).
The matrix $B_S$ represents a weighted reciprocal star with center $1$ and leaves $S$, with all other nodes isolated.
Let $e_j$ denote the standard coordinate vectors and set $u_S=(s-1)^{-1/2}\sum_{j\in S}e_j$.
Then
\[
B_Se_1=\tfrac12u_S,\qquad
B_Su_S=\tfrac12e_1,\qquad
B_Sv=0\quad\text{for }\;v\perp\operatorname{span}\{e_1,u_S\}.
\]
Thus the symmetric matrix $B_S$ has nonzero eigenvalues
$1/2$ and $-1/2$, giving $\|B_S\|_{\rm op}=1/2$.
Every principal submatrix of $B_S$ also has operator norm at most
$1/2$, so every principal submatrix of $I-B_S$ is invertible.
Thus Assumption~\ref{ass:PI} holds.

The nodes in $\{1\}\cup S$ are mutually reachable through node~$1$,
whereas all other nodes are isolated. Hence $\{1\}\cup S$ is the only
nontrivial SCC and no SCC has an external parent. Since
$\|B_S\|_{\rm op}=1/2$, Lemma~\ref{lem:packing-regularity} with
$\rho_0=1/2$ and the constants fixed above verifies
Assumption~\ref{ass:sg-ev}.
Therefore the model belongs to $\mathcal M_{p,s,d}(K,\lambda)$ for every $d\ge0$.

For statement~(ii), distinct sets $S$ produce distinct nontrivial SCCs and
therefore distinct condensations. For $1\le k\le N$,
\begin{equation}
 \binom Nk
 =\prod_{j=0}^{k-1}\frac{N-j}{k-j}
 \ge\left(\frac Nk\right)^k,
 \label{eq:binomial-lower-bound}
\end{equation}
because $(N-j)/(k-j)\ge N/k$ is equivalent to $j(N-k)\ge0$.
Applying \eqref{eq:binomial-lower-bound} with $(N,k)=(p-1,s-1)$ gives
\[
 \log N_s
 \ge(s-1)\log\frac{p-1}{s-1}
 \ge\frac{s}{2}\log\frac ps
 \ge\frac{\log2}{2(1+\log2)}s\log\frac{ep}{s}.
\]
The second inequality uses $s-1\ge s/2$ and
$(p-1)/(s-1)\ge p/s$. The third follows from $p/s\ge2$ and
$\log(p/s)/\log(ep/s)\ge\log2/(1+\log2)$, proving statement~(ii).

For statement~(iii), $\|B_S\|_F^2=1/2$, so
Lemma~\ref{lem:kl-control} and tensorization give
\[
 \operatorname{KL}(P_{B_S}^n\|P_0^n)
 \le\frac{c_{\rm KL}n}{2}.
\]
This bound holds for every $S$, completing the proof.
\end{proof}

\begin{lemma}[Parent-star external-parent packing]
\label{lem:parent-packing}
For every $p\ge4$, $1\le d\le p/2$, and
$Q\subseteq\{2,\ldots,p\}$ with $|Q|=d$, define $B_Q$ by
\[
 (B_Q)_{1q}=\frac{1}{2\sqrt d},
 \qquad q\in Q,
\]
with all other entries equal to zero. Then:

\begin{enumerate}
 \item[(i)] The model defined by $B_Q$ and $f$ belongs to
 $\mathcal M_{p,s,d}(K,\lambda)$ for every $s\ge1$. It is acyclic,
 satisfies $s_{\max}=1$, and has $d_B=d$.

 \item[(ii)] Writing $N_d:=\binom{p-1}{d}$,
 the packing contains $N_d\ge3$ distinct condensations and
 \[
  \log N_d
  \ge \frac{\log2}{2(1+\log2)}d\log\frac{ep}{d}.
 \]

 \item[(iii)] The observation laws satisfy
 \[
  \max_Q\operatorname{KL}(P_{B_Q}^n\|P_0^n)
  \le\frac{c_{\rm KL}n}{4}.
 \]
\end{enumerate}
\end{lemma}

\begin{proof}
We first verify statement~(i). The only directed edges are $q\to1$ for
$q\in Q$, so the graph is acyclic
with $s_{\max}=1$ and $d_B=d$.
Acyclicity also implies Assumption~\ref{ass:PI}: after a topological
ordering, every principal submatrix of $I-B_Q$ is triangular with unit
diagonal.
Since $B_Q$ has only one nonzero row,
\[
 \|B_Q\|_{\rm op}^2=\|B_Q\|_F^2
 =d\left(\frac1{2\sqrt d}\right)^2=\frac14.
\]
Lemma~\ref{lem:packing-regularity} at $\rho_0=1/2$ and the constants fixed
above therefore imply
that the model defined by $B_Q$ and $f$ belongs to
$\mathcal M_{p,s,d}(K,\lambda)$ for every $s\ge1$.

For statement~(ii), distinct sets $Q$ give different parent sets of
node~$1$, so the
packing contains exactly $N_d=\binom{p-1}{d}$ distinct condensations.
Since $1\le d\le p-2$, $N_d\ge p-1\ge3$.
Applying \eqref{eq:binomial-lower-bound} with $(N,k)=(p-1,d)$ yields
\begin{align*}
 \log N_d
 \ge d\log\frac{p-1}{d}
 &\ge d\left\{\log(3/2)+\log\frac{p}{2d}\right\}\\
 &\ge \frac{\log2}{2(1+\log2)}d
       \left\{1+\log2+\log\frac{p}{2d}\right\}
 =\frac{\log2}{2(1+\log2)}d\log\frac{ep}{d}.
\end{align*}
The second inequality uses $p-1\ge3p/4$ for $p\ge4$.
The third uses $\frac12\log2<\log(3/2)$,
$\log2/[2(1+\log2)]<1$, and
$\log(p/(2d))\ge0$.

For statement~(iii), Lemma~\ref{lem:kl-control} at $\rho_0=1/2$ and
tensorization give
\[
 \operatorname{KL}(P_{B_Q}^n\|P_0^n)
 =n\operatorname{KL}(P_{B_Q}\|P_0)
 \le c_{\rm KL}n\|B_Q\|_F^2
 =\frac{c_{\rm KL}n}{4}.
\]
This bound holds for every $Q$, completing the proof.
\end{proof}

\subsection{Proof of Theorem~\ref{thm:matching-lower}}

\begin{proof}[Proof of Theorem~\ref{thm:matching-lower}]
The parent family lies in $\mathcal M_{p,s,d}(K,\lambda)$ for every $s\ge1$,
and the SCC family does so for $s\ge2$, with the universal constants fixed
above. Write
$h_s=s\log(ep/s)$ and $h_d=d\log(ep/d)$.
Lemmas~\ref{lem:star-packing} and~\ref{lem:parent-packing} give the following
bounds, with the SCC bound applying only for $s\ge2$:
\begin{align}
 \log N_s
 &\ge \frac{\log2}{2(1+\log2)}h_s,
 &
 \max_S\operatorname{KL}(P_{B_S}^n\|P_0^n)
 &\le\frac{c_{\rm KL}n}{2},
 \label{eq:scc-packing-summary}\\
 \log N_d
 &\ge \frac{\log2}{2(1+\log2)}h_d,
 &
 \max_Q\operatorname{KL}(P_{B_Q}^n\|P_0^n)
 &\le\frac{c_{\rm KL}n}{4}.
 \label{eq:parent-packing-summary}
\end{align}

Set
\[
 c_{\rm lb}
 :=
 \frac{2(1+\log2)c_{\rm KL}}{\log2}
 +4(1+\log2).
\]
Since the models within each packing have distinct condensations, generalized
Fano's inequality \citep{yu1997assouad,tsybakov2009introduction}, applied
separately to \eqref{eq:scc-packing-summary} and
\eqref{eq:parent-packing-summary}, gives
\begin{align*}
 \max_S P_{B_S}^n\left(\widehat G^{\rm sc}\ne G^{\rm sc}(B_S)\right)
 &\ge
 1-
 \frac{(1+\log2)c_{\rm KL}n}{\log2\,h_s}
 -
 \frac{2(1+\log2)}{h_s}
 \ge 1-\frac{c_{\rm lb}n}{2h_s},
 \qquad s\ge2,
 \\
 \max_Q P_{B_Q}^n\left(\widehat G^{\rm sc}\ne G^{\rm sc}(B_Q)\right)
 &\ge
 1-
 \frac{(1+\log2)c_{\rm KL}n}{2\log2\,h_d}
 -
 \frac{2(1+\log2)}{h_d}
 \ge 1-\frac{c_{\rm lb}n}{2h_d}.
\end{align*}
The last inequalities use $n\ge1$. The first bound is obtained on models
with $d_B=0$, and the second is obtained on acyclic models; hence they
also establish the two separate necessity claims in Section~\ref{sec:lower}.

It remains to combine these bounds over the full model class. If $s\ge2$,
use the packing corresponding to $\max\{h_s,h_d\}$. Since both packings are
subsets of $\mathcal M_{p,s,d}(K,\lambda)$,
\[
 \sup_{M\in\mathcal M_{p,s,d}(K,\lambda)}
 P_M^n\left(\widehat G^{\rm sc}\ne G^{\rm sc}(M)\right)
 \ge 1-\frac{c_{\rm lb}n}{2\max\{h_s,h_d\}}
 \ge 1-\frac{c_{\rm lb}n}{h_s+h_d}.
\]
If $s=1$, the SCC packing is unavailable, but the function
$x\mapsto x\log(ep/x)$ is increasing on $[1,p/2]$, and therefore
$h_s=\log(ep)\le h_d$. The parent-packing bound then gives
\[
 \sup_{M\in\mathcal M_{p,1,d}(K,\lambda)}
 P_M^n\left(\widehat G^{\rm sc}\ne G^{\rm sc}(M)\right)
 \ge 1-\frac{c_{\rm lb}n}{2h_d}
 \ge 1-\frac{c_{\rm lb}n}{h_s+h_d}.
\]
Thus the same bound holds for every $1\le s\le p/2$. Under the sample-size
condition in \eqref{eq:lower-sample-size}, its right-hand side is at least
$\delta$.
\end{proof}

\section{Population proof details}
\label{app:population-proofs}

Throughout this section, we consider the cyclic LiNG model
\eqref{eq:sem} on $V=[p]$ under Assumption~\ref{ass:PI}, using the notation
of Section~\ref{sec:setup}. Thus $B$ is the weighted adjacency matrix,
$A=(I-B)^{-1}$, and the coordinates of $\eps$ are mutually independent,
with at most one Gaussian coordinate.

\subsection{Block exogeneity and root SCCs}

The following lemma isolates the argument used to rule out incoming edges.

\begin{lemma}[Residual independence forces ancestry]
\label{lem:residual-ancestry}
Let $R\subseteq V$ be partitioned as $R=C\mathbin{\cup}D$, with
$C\ne\varnothing$, and write $A^{(R)}=(I-B_{R,R})^{-1}$. For matrices $Q$ and $L$ of compatible dimensions, let
\[
 Y=A^{(R)}_{C,R}\eps_R,
 \qquad
 Z=(A^{(R)}_{D,R}-QA^{(R)}_{C,R})\eps_R+L\eps_{V\setminus R}.
\]
If $\Cov(Y,Z)=0$ and $Y_c\indep Z_\ell$ for every $c\in C$ and
$\ell\in D$, then $C$ is ancestral in $G[R]$.
\end{lemma}

\begin{proof}
The claim is immediate when $D=\varnothing$, so assume $D\ne\varnothing$.
Assumption~\ref{ass:PI} and Jacobi's complementary-minor identity
\citep[Section~0.8.5]{horn2013matrix} imply
that the principal submatrices $A^{(R)}_{C,C}$ and $A^{(R)}_{D,D}$ are
nonsingular: indeed,
$\det(A^{(R)}_{C,C})=\det(I-B_{D,D})/\det(I-B_{R,R})$ and
$\det(A^{(R)}_{D,D})=\det(I-B_{C,C})/\det(I-B_{R,R})$ are nonzero.

Since $Y$ depends only on $\eps_R$, independence of the sources and
$\Cov(Y,Z)=0$ give
\[
 \sum_{k\in R}\Var(\eps_k)A^{(R)}_{C,k}
 \bigl(A^{(R)}_{D,k}-QA^{(R)}_{C,k}\bigr)^\top=0.
\]
The Darmois--Skitovitch theorem
\citep{darmois1953analyse,skitovitch1953property,kagan1973characterization},
applied to each pair $(Y_c,Z_\ell)$, gives
\[
 A^{(R)}_{C,k}
 \bigl(A^{(R)}_{D,k}-QA^{(R)}_{C,k}\bigr)^\top=0
\]
for every non-Gaussian source $\eps_k$, $k\in R$.
There is at most one Gaussian source. Hence either every summand already
vanishes, or the covariance identity forces the sole remaining summand to
vanish.
Therefore, for every $k\in R$,
\begin{equation}
 A^{(R)}_{C,k}\ne0
 \quad\Longrightarrow\quad
 A^{(R)}_{D,k}-QA^{(R)}_{C,k}=0.
 \label{eq:no-shared-source}
\end{equation}

Nonsingularity of $A^{(R)}_{C,C}$ and~\eqref{eq:no-shared-source} imply
\[
 A^{(R)}_{D,C}-QA^{(R)}_{C,C}=0.
\]
If $C$ were non-ancestral, then $B_{C,D}\ne0$. The $(C,D)$ block of
$I=(I-B_{R,R})A^{(R)}$ is zero, which gives
\[
 0=(I-B_{C,C})A^{(R)}_{C,D}-B_{C,D}A^{(R)}_{D,D}.
\]
Since $A^{(R)}_{D,D}$ is nonsingular, $A^{(R)}_{C,D}\ne0$. Choose
$k\in D$ with $A^{(R)}_{C,k}\ne0$. Equation~\eqref{eq:no-shared-source}
then yields
\[
 A^{(R)}_{R,C\cup\{k\}}
 =\begin{bmatrix}I_{|C|}\\Q\end{bmatrix}
 A^{(R)}_{C,C\cup\{k\}}.
\]
The matrix $A^{(R)}_{C,C\cup\{k\}}$ has rank $|C|$ because it contains
the nonsingular submatrix $A^{(R)}_{C,C}$, and
$\left[\begin{smallmatrix}I_{|C|}\\Q\end{smallmatrix}\right]$ has full
column rank. Hence the right-hand side has rank $|C|$. In contrast, the
left-hand side consists of $|C|+1$ columns of the invertible matrix
$A^{(R)}$ and therefore has rank $|C|+1$, a contradiction. Thus
$B_{C,D}=0$, so $C$ is ancestral in $G[R]$.
\end{proof}

\begin{proof}[Proof of Theorem~\ref{thm:block-exogeneity}]
Statement $(ii)$ implies statement $(iii)$ by marginalization. We prove the
other two implications. Write $D=R\setminus C$. Since $F$ is ancestral,
\eqref{eq:induced-residual-model} gives
\begin{equation}
 e_C(F)=A_{C,R}\eps_R,
 \qquad
 e_D(F)=A_{D,R}\eps_R,
 \qquad
 A_{R,R}=(I-B_{R,R})^{-1}.
 \label{eq:common-block-residuals}
\end{equation}
Let $\Omega_R=\Cov(\eps_R)$, which is diagonal and positive definite, and
define
\begin{equation}
 Q_C
 :=\Cov(e_D(F),e_C(F))\Cov(e_C(F))^{-1}
 =A_{D,R}\Omega_RA_{C,R}^{\top}
   \bigl(A_{C,R}\Omega_RA_{C,R}^{\top}\bigr)^{-1}.
 \label{eq:general-QC}
\end{equation}
The inverse exists because $A_{C,C}$ is nonsingular and
$\Omega_R\succ0$. By definition,
\begin{equation}
 r_{F,C}(D)=(A_{D,R}-Q_CA_{C,R})\eps_R.
 \label{eq:common-complement-residual}
\end{equation}

For $(i)\Rightarrow(ii)$, suppose that $C$ is ancestral in $G[R]$. Then
$B_{C,D}=0$ and
\[
 A_{R,R}=
 \begin{bmatrix}A_{C,C}&0\\A_{D,C}&A_{D,D}\end{bmatrix}.
\]
Consequently,
\[
 e_C(F)=A_{C,C}\eps_C,
 \qquad
 Q_C=A_{D,C}A_{C,C}^{-1},
 \qquad
 r_{F,C}(D)=A_{D,D}\eps_D.
\]
$e_C(F)$ and $r_{F,C}(D)$ depend on disjoint independent source vectors, proving
statement $(ii)$.

For $(iii)\Rightarrow(i)$, suppose that statement $(iii)$ holds. Since
$r_{F,C}(D)$ is the projection residual,
$\Cov(e_C(F),r_{F,C}(D))=0$. The representations
in~\eqref{eq:common-block-residuals} and
\eqref{eq:common-complement-residual} satisfy
Lemma~\ref{lem:residual-ancestry} with $Q=Q_C$ in \eqref{eq:general-QC} and no additional source
term. The lemma therefore shows that $C$ is ancestral in $G[R]$, proving
statement $(i)$.
\end{proof}

The graph-theoretic part of the characterization will also be used below.

\begin{lemma}[Ancestral sets and root SCCs]
\label{lem:ancestral-root-scc}
In any finite directed graph, the root SCCs are exactly the
inclusion-minimal nonempty ancestral sets. Moreover, every nonempty
ancestral set contains a root SCC.
\end{lemma}

\begin{proof}
A root SCC is ancestral because it has no incoming edge from another SCC.
For every nonempty proper subset of a root SCC, strong connectivity provides a path from outside the subset to a node inside it, so the subset is not ancestral.

Now let $U$ be a nonempty ancestral set.
Then $U$ is a union of SCCs, and the condensation subgraph induced by these SCCs has a source $S$.
No SCC inside $U$ enters $S$ by the source property, and no SCC outside $U$ enters $S$ by ancestry.
Thus $S$ is a root SCC of the full graph.
Every nonempty ancestral set therefore contains a root SCC, and inclusion minimality forces equality.
\end{proof}

\begin{proof}[Proof of Corollary~\ref{cor:minimal-scc}]
By Theorem~\ref{thm:block-exogeneity}, the nonempty sets satisfying
condition~$(ii)$ are exactly the nonempty ancestral sets of $G[R]$.
Lemma~\ref{lem:ancestral-root-scc} identifies their inclusion-minimal
elements as the root SCCs of $G[R]$.
\end{proof}

\subsection{Sparse adjustment and external parents}

\begin{proof}[Proof of Proposition~\ref{prop:sparse-passing}]
If $F=\varnothing$, then $P=\varnothing$ and $\mathsf{UP}$ holds automatically.
Theorem~\ref{thm:block-exogeneity} gives the equivalence between
$\mathsf{OUT}$ and ancestry, and an ancestral $C$ has
$P_{\rm ext}(C)=\varnothing$.
Thus all claims hold in this case.
Assume henceforth that $F\ne\varnothing$.

Write $D=R\setminus C$. For $T=C,D$, let $H_T$ denote the rows
of $H_R$ indexed by $T$ in~\eqref{eq:block-solution}. Then
\begin{equation}
 X_T=H_TX_F+e_T(F),
 \qquad T=C,D.
 \label{eq:sparse-block-decomposition}
\end{equation}
Here $e_T(F)\indep X_F$ by the ancestry of $F$.
Define $G_T=G_T(P)$ by
\[
 G_TX_F:=H_TX_F-\Proj(H_TX_F\mid X_P).
\]
Since $X_P$ is a subvector of $X_F$, independence gives
$\Proj(e_T(F)\mid X_P)=0$. Therefore
\[
\begin{aligned}
 e_T(P)
 &=X_T-\Proj(X_T\mid X_P)\\
 &=e_T(F)+H_TX_F-\Proj(H_TX_F\mid X_P)\\
 &=e_T(F)+G_TX_F.
\end{aligned}
\]

If $\mathsf{UP}$ holds, the independence in that condition gives
\[
 0=\Cov(e_C(P),X_F)=G_C\Cov(X_F).
\]
Positive definiteness of $\Cov(X_F)$ gives $G_C=0$. Conversely,
$e_C(P)=e_C(F)$ implies $e_C(P)\indep X_F$. Hence
\begin{equation}
 \mathsf{UP}\quad\Longleftrightarrow\quad e_C(P)=e_C(F).
 \label{eq:up-exact-adjustment}
\end{equation}

Under $\mathsf{UP}$, we have $e_C(P)=e_C(F)$.
Since $X_F\indep e_C(F)$, the projection coefficients of $e_D(P)$
and $e_D(F)$ onto the common candidate residual are both $Q_C$
from~\eqref{eq:general-QC}.
Consequently,
\begin{equation}
 r_{P,C}(D)
 =e_D(F)+G_DX_F-Q_Ce_C(F)
 =r_{F,C}(D)+G_DX_F.
 \label{eq:sparse-complement-residual}
\end{equation}

For the forward implication, suppose that $\mathsf{UP}$ and $\mathsf{OUT}$
hold. Since $F$ is ancestral, $X_F=A_{F,F}\eps_F$. The representations
\eqref{eq:common-block-residuals},
\eqref{eq:common-complement-residual}, and
\eqref{eq:sparse-complement-residual} have the form in
Lemma~\ref{lem:residual-ancestry}. Condition $\mathsf{OUT}$ supplies the
required scalar independences, while the projection normal equation gives
$\Cov(e_C(P),r_{P,C}(D))=0$. Thus
Lemma~\ref{lem:residual-ancestry} applies with $Q=Q_C$ and
$L=G_DA_{F,F}$. Hence $C$ is ancestral in $G[R]$.

For this ancestral $C$, block triangularity gives
$H_C=(I-B_{C,C})^{-1}B_{C,F}$. Because $(I-B_{C,C})^{-1}$ is invertible
and ancestry rules out parents in $D$,
\[
 \{j\in F:H_{C,j}\ne0\}
 =\{j\in F:B_{C,j}\ne0\}
 =P_{\rm ext}(C).
\]
Moreover, $\mathsf{UP}$ and~\eqref{eq:up-exact-adjustment} give
\[
 \Proj(X_C\mid X_P)=H_CX_F.
\]
Positive definiteness of $\Cov(X_F)$ makes the
linear coefficients in $X_F$ unique. Since the coefficient matrix of the
left-hand side has zero columns outside $P$, every nonzero column of $H_C$
must be indexed by $P$.
Using the support identity above, we obtain
$P_{\rm ext}(C)\subseteq P$.

For the reverse implication, suppose that $C$ is ancestral in $G[R]$ and
$P_{\rm ext}(C)\subseteq P$. The support identity above uses only ancestry,
so $H_CX_F$ is a linear function of $X_P$. By
\eqref{eq:sparse-block-decomposition} and $e_C(F)\indep X_F$, $\Proj(X_C\mid X_P)=H_CX_F$, and hence
$e_C(P)=e_C(F)$.
Equation~\eqref{eq:up-exact-adjustment} gives $\mathsf{UP}$.

The proof of Theorem~\ref{thm:block-exogeneity} gives
\[
 e_C(F)=A_{C,C}\eps_C,
 \qquad
 r_{F,C}(D)=A_{D,D}\eps_D.
\]
Since $X_F=A_{F,F}\eps_F$, equation~\eqref{eq:sparse-complement-residual}
gives
\[
 r_{P,C}(D)=A_{D,D}\eps_D+G_DA_{F,F}\eps_F,
 \qquad e_C(P)=A_{C,C}\eps_C.
\]
Independence of the noise coordinates therefore gives
$e_C(P)\indep r_{P,C}(D)$, proving $\mathsf{OUT}$.

The minimum-cardinality claim follows because the passing adjustments for
an ancestral $C$ are exactly the supersets of $P_{\rm ext}(C)$.
\end{proof}

\subsection{Minimum passers and population recovery}

\begin{lemma}[Minimum passer]
\label{lem:minimum-passer}
Fix an ancestral $F$ with $R=V\setminus F\ne\varnothing$.
Let $(C,P)$ minimize $|C|+|P|$ among passing pairs.
Then $C$ is a root SCC of $G[R]$ and $P=P_{\rm ext}(C)$.
\end{lemma}
\begin{proof}
Let $(C,P)$ be any passing pair. By Proposition~\ref{prop:sparse-passing}, $C$ is ancestral in $G[R]$ and $P_{\rm ext}(C)\subseteq P$.
By Lemma~\ref{lem:ancestral-root-scc}, $C$ contains a root SCC $S$ of $G[R]$.
Every external parent of $S$ lies in $F$ and is also an external parent of $C$, so
\[
 P_{\rm ext}(S)\subseteq P_{\rm ext}(C)\subseteq P,
 \qquad
 |S|+|P_{\rm ext}(S)|\le |C|+|P|.
\]
The pair $(S,P_{\rm ext}(S))$ passes by Proposition~\ref{prop:sparse-passing}.
If $(C,P)$ is a minimizer, equality must hold, forcing $C=S$ and $P=P_{\rm ext}(S)$.
This proves the claim.

Since every root SCC $S$ yields the passing pair $(S,P_{\rm ext}(S))$, the same argument also gives
\[
 \min_{(C,P):\,\mathsf{PASS}(C,P)=1}(|C|+|P|)
 =\min_{S:\,\text{root SCC of }G[R]}
   (|S|+|P_{\rm ext}(S)|).
\]
The equality case above identifies all minimizers.
\end{proof}


\begin{proof}[Proof of Theorem~\ref{thm:population-recovery}]
We prove by induction that the following invariants hold before every round:
\begin{align}
 &\widehat\Pi\subseteq\Pi,\qquad
 F=\bigcup_{T\in\widehat\Pi}T,
 \qquad F\text{ is ancestral in }G,
 \tag{I1}\label{eq:population-partition-invariant}\\
 &\widehat E^{\rm sc}
 =\{T\to U\in E^{\rm sc}:T,U\in\widehat\Pi\}.
 \tag{I2}\label{eq:population-edge-invariant}
\end{align}
They hold initially because $F$, $\widehat\Pi$, and
$\widehat E^{\rm sc}$ are empty.

Suppose that~\eqref{eq:population-partition-invariant} and~\eqref{eq:population-edge-invariant} hold before a round and that $F\ne V$.
Write $R=V\setminus F$.
Since $F$ is an ancestral union of true SCCs, the SCCs of $G[R]$ are exactly
the unrecovered SCCs
\begin{equation}
 \Pi_R:=\{T\in\Pi:T\subseteq R\}
 =\Pi\setminus\widehat\Pi.
 \label{eq:remainder-true-sccs}
\end{equation}
Indeed, removing $F$ preserves each remaining SCC and cannot create new
directed paths between distinct SCCs.

Every root SCC $S$ of $G[R]$ satisfies
\[
 |S|\le s_{\max}\le s,\qquad
 P_{\rm ext}(S)\subseteq F,\qquad
 |P_{\rm ext}(S)|\le d_B\le d.
\]
Thus the minimizing pair in Lemma~\ref{lem:minimum-passer} is eligible.
Write the first passing pair selected by the algorithm as
$(S,\widehat P(S))$. By the search order and
Lemma~\ref{lem:minimum-passer}, $S$ is a root SCC of $G[R]$,
$\widehat P(S)=P_{\rm ext}(S)$, and
\begin{equation}
 |S|+|\widehat P(S)|
 =\min_{U:\,\text{root SCC of }G[R]}(|U|+|P_{\rm ext}(U)|)
 \le s_{\max}+d_B.
 \label{eq:population-first-passing-size}
\end{equation}
Thus the failure branch is not reached. For every $T\in\widehat\Pi$,
\[
 T\to S\in E^{\rm sc}
 \quad\Longleftrightarrow\quad
 T\cap P_{\rm ext}(S)\ne\varnothing
 \quad\Longleftrightarrow\quad
 T\cap\widehat P(S)\ne\varnothing.
\]
The update therefore adds exactly the condensation edges involving $S$ and previously recovered SCCs: the displayed equivalence identifies all incoming edges, while ancestry of $F$ rules out edges from $S$ into $F$.
Thus \eqref{eq:population-edge-invariant} is preserved.
Since $S$ is a root SCC of $G[R]$, $F\cup S$ remains ancestral.
Together with adding the unrecovered true SCC $S$ to $\widehat\Pi$, this preserves \eqref{eq:population-partition-invariant}.

Each round adds one unrecovered SCC, so the algorithm reaches $F=V$ after
at most $|\Pi|$ rounds. The two invariants then give
$\widehat\Pi=\Pi$ and $\widehat E^{\rm sc}=E^{\rm sc}$.
\end{proof}

\subsection{Population recovery with arbitrary search bounds}
\label{app:population-bounds}

\begin{proposition}[Insufficient search bounds]
\label{prop:population-insufficient-bounds}
Suppose the cyclic LiNG model~\eqref{eq:sem} and Assumption~\ref{ass:PI} hold.
Run Algorithm~\ref{alg:population} with arbitrary integers $s\ge1$, $d\ge0$ and population $\mathsf{PASS}$ decisions.
The algorithm returns failure if and only if $s<s_{\max}$ or $d<d_B$.
\end{proposition}

\begin{proof}
If $s\ge s_{\max}$ and $d\ge d_B$,
Theorem~\ref{thm:population-recovery} rules out failure.
For the converse, we show that termination without failure
requires $s\ge s_{\max}$ and $d\ge d_B$.

We use an argument similar to the proof of Theorem~\ref{thm:population-recovery}.
The invariants hold initially; suppose \eqref{eq:population-partition-invariant} and~\eqref{eq:population-edge-invariant} hold before a round with $F\ne V$, and write $R=V\setminus F$.
If the round completes, let $(C,P)$ be its first passing pair.
As in the proof of Lemma~\ref{lem:minimum-passer}, $C$ contains a root SCC $S$ of $G[R]$ such that
\[
 S\subseteq C,\qquad P_{\rm ext}(S)\subseteq P,
 \qquad \mathsf{PASS}(S,P_{\rm ext}(S))=1.
\]
Since
\[
 |S|\le |C|\le s,\qquad
 |P_{\rm ext}(S)|\le |P|\le d,
\]
this pair is also eligible, and the search order gives
\[
 |C|+|P|
 \le |S|+|P_{\rm ext}(S)|
 \le |C|+|P|.
\]
Equality and the inclusions above force
$C=S$ and $P=P_{\rm ext}(S)$.
Under the invariants, $S$ is an unrecovered true SCC.
The update therefore preserves both invariants by the same
argument as in Theorem~\ref{thm:population-recovery}.

If the algorithm does not return failure, it reaches $F=V$ because each completed round adds an unrecovered true SCC.
Every true SCC and its exact external parent set have therefore been selected within the search bounds.
Hence
\[
 s_{\max}=\max_{S\in\Pi}|S|\le s,
 \qquad
 d_B=\max_{S\in\Pi}|P_{\rm ext}(S)|\le d.
\]
\end{proof}

\paragraph{Computational count.}
With valid search bounds and correct decisions, every examined pair has $|C|+|P|\le s_{\max}+d_B$ by \eqref{eq:population-first-passing-size}.
At a round with sets $F,R$, the number of eligible pairs of total size $t$ is at most
\[
 \sum_{\substack{1\le q\le |R|\\0\le t-q\le |F|}}
 \binom{|R|}{q}\binom{|F|}{t-q}
 \le \binom pt.
\]
Over at most $p$ rounds, the total number of examined pairs is at most
\[
 p\sum_{t=1}^{\min\{p,s_{\max}+d_B\}}\binom pt
 \le p^{O(s_{\max}+d_B)}.
\]
Each candidate evaluation requires regressions and scalar statistics, with at most $s_{\max}+d_B$ predictors in any regression.
In general, the same counting argument gives a worst-case bound of $p^{O(s+d)}$.


\section{Finite-sample proof details}
\label{app:finite}

This section first counts the scalar tests at each candidate size and establishes uniform bounds for residuals fitted on the same observations.
We then prove Theorem~\ref{thm:finite-recovery} and Corollary~\ref{cor:uniform-recovery-optimality}.
Finally, we verify Assumption~\ref{ass:uniform-approximation} for normalized distance covariance and Gaussian HSIC.
The upper-bound analysis requires only the natural bounds
$s_{\max}\le p$ and $d_B\le p-1$.

We establish uniform bounds at each total candidate size and then combine them over the sizes needed for recovery.
For $1\le h\le p$, let
\[
 \mathcal A(h):=\{(F,C,P):F\text{ ancestral},\
 \varnothing\ne C\subseteq V\setminus F,\ P\subseteq F,\ |C|+|P|=h\},
\]
and define
\[
\begin{aligned}
 \mathcal I_{\rm up}(h)
 &:=\{(F,C,P,c,k):(F,C,P)\in\mathcal A(h),\ c\in C,\ k\in F\},\\
 \mathcal I_{\rm out}(h)
 &:=\{(F,C,P,c,\ell):(F,C,P)\in\mathcal A(h),\ c\in C,\
       \ell\in V\setminus(F\cup C)\}.
\end{aligned}
\]
Thus $\mathcal A=\bigcup_{h=1}^{\min\{p,s_{\max}+d_B\}}\mathcal A(h)$.
For $0<\alpha<1$, define
\[
 \ell(h;\alpha)
 :=h\log\frac{ep}{h}+h\log2+\log p+\log h
   +\log\frac{64}{\alpha}+\log[h(h+1)].
\]
For the target error probability $\delta$, write $\ell(h)=\ell(h;\delta)$.

All sample regressions use empirically centered responses and predictors, and empirical covariances are computed from centered arrays.
Fitted residual arrays are therefore already centered.
Constants depend only on the stated tail, covariance, and fixed estimator parameters.

The recovery procedure returns failure if a required sample regression is singular or a residual scale is zero.
For definiteness, the affected statistics are assigned zero;
these cases are excluded on the regularity events used below.


\subsection{Candidate count}

\begin{lemma}[Candidate count and complexity]
\label{lem:count}
The following statements hold.
\begin{enumerate}
 \item[(i)] For every $1\le h\le p$, the number of distinct scalar
 statistics indexed by
 $\mathcal I_{\rm up}(h)\cup\mathcal I_{\rm out}(h)$ satisfies
 \[
  N_{\rm test}\le2ph\left(\frac{2ep}{h}\right)^h.
 \]

 \item[(ii)] For every $1\le h\le p$ and $0<\alpha<1$,
 \[
  \log\frac{32(N_{\rm test}\vee1)}{\alpha}\le\ell(h;\alpha),
  \quad \text{ and} \quad h+1\le\ell(h;\alpha).
 \]

 \item[(iii)] For every fixed $0<\alpha<1$,
 $h\mapsto\ell(h;\alpha)$ is increasing on $[1,p]$.

 \item[(iv)] For $1\le s\le p$ and $0\le d\le p-1$, define
 \[
  L_{s,d}:=s\log\frac{ep}{s}+d\log\frac{ep}{d}
  +\log p+\log s+\log\frac{64}{\delta},
 \]
 with $d\log(ep/d)=0$ when $d=0$. Then
 \[
  \frac12L_{s,d}
  \le\ell(\min\{p,s+d\};\delta)
  \le3L_{s,d}.
 \]
\end{enumerate}
\end{lemma}
\begin{proof}
For statement~(i), choose the union $C\cup P$ and then its nonempty subset $C$.
Since $C$ and $P$ are disjoint, the number of candidate pairs is
\[
 \sum_{q=1}^h\binom pq\binom{p-q}{h-q}
 =\binom ph\sum_{q=1}^h\binom hq
 =(2^h-1)\binom ph.
\]
For each pair, there are at most $h$ choices of $c$ and $p$ choices of the
tested variable for each of the two scores. Since $F$ only restricts
eligibility,
\[
 N_{\rm test}
 \le2ph(2^h-1)\binom ph
 \le2ph\left(\frac{2ep}{h}\right)^h,
\]
where the last inequality uses $2^h-1\le2^h$ and
$\binom ph\le(ep/h)^h$. This proves statement~(i).

For statement~(ii), taking logarithms in statement~(i) gives the first
inequality; the case $N_{\rm test}=0$ is immediate. The second follows from
$h\log(ep/h)\ge h$ and $\log(64/\alpha)>1$.

For statement~(iii), the derivative of $h\log(ep/h)$ is $\log(p/h)$.
All remaining $h$-dependent terms in $\ell(h;\alpha)$ are strictly
increasing, so the claim follows.

For statement~(iv), put $u=\min\{p,s+d\}$ and
$J=s\log(ep/s)+d\log(ep/d)$. Then
$L_{s,d}=J+\log p+\log s+\log(64/\delta)$.
Since $s,d\le p$, we have $J\ge s+d\ge u$.

If $s+d\le p$, then $u=s+d$ and
\[
J-u\log(ep/u)
=u\left(-\frac{s}{u}\log\frac{s}{u}
        -\frac{d}{u}\log\frac{d}{u}\right)
\in[0,u\log2].
\]
Then the binary-entropy bound gives $u\log(ep/u)\le J\le u\log(ep/u)+u\log2$.

If $s+d>p$, then $u=p$ and
$u\log(ep/u)=p\le J$.
Thus, in both cases, $u\log(ep/u)\le J$.

Also, $u\ge1$ and $2\log u\le u$ imply
\[
 \log[u(u+1)]\le2\log u+\log2\le u+\log2.
\]
Using $u\le p$ and $u\le J$, we obtain
\begin{align*}
 \ell(u;\delta)
 &\le J+(1+\log2)u
       +2\log p+\log(64/\delta)+\log2\\
 &\le(2+\log2)J
       +2\log p+\log(64/\delta)+\log2 \le3L_{s,d},
\end{align*}
where the last inequality uses $\log s\ge0$,
$\log2<1$, and $\log(64/\delta)\ge\log2$.

For the reverse inequality, if $s+d\le p$, the entropy bound
and $\log s\le\log u$ give 
\[
L_{s,d}\le u\log(ep/u)+u\log2+\log p+\log u+\log(64/\delta)
\le\ell(u;\delta).
\]
If $s+d>p$, the function $x\log(ep/x)$ is increasing on $(0,p]$,
so $J\le2p$. Then
\[
 L_{s,d}\le2p+2\log p+\log(64/\delta)
 \le2\ell(p;\delta).
\]
\end{proof}

\subsection{Residual estimation on the same observations}

The next lemma establishes uniform residual and scale bounds over our candidate collection.

The required response--predictor pairs are
\[
\begin{aligned}
 \mathcal R(h)
 :={}&\bigl\{(c,P):(F,C,P,c,k)\in\mathcal I_{\rm up}(h)
                    \text{ for some }(F,C,k)\bigr\}\\
 &\mathrel{\phantom{:={}}}\cup
 \bigl\{(c,P):(F,C,P,c,\ell)\in\mathcal I_{\rm out}(h)
                    \text{ for some }(F,C,\ell)\bigr\}\\
 &\mathrel{\phantom{:={}}}\cup
 \bigl\{(\ell,P\cup C):(F,C,P,c,\ell)
                    \in\mathcal I_{\rm out}(h)\bigr\}.
\end{aligned}
\]
For $u\in\mathbb R^n$, write
\[
 \|u\|_n:=\left(\frac1n\sum_{i=1}^n u_i^2\right)^{1/2}.
\]
For any $(v,Q)\in\mathcal R(h)$, write
\[
 Y=e_v(Q),
 \qquad
 y=(Y^{(i)}-\overline Y)_{i=1}^n,
\]
let $\widehat y$ be the residual array from least-squares regression of
$X_v$ on $X_Q$ with an intercept, and set
\[
 \sigma_Y^2=\Var(Y),
 \qquad
 \widehat\sigma_y=\|y\|_n,
 \qquad
 \widehat\sigma_{\widehat y}=\|\widehat y\|_n.
\]

\begin{lemma}[Uniform residual prediction]
\label{lem:residual-prediction}
Under Assumption~\ref{ass:sg-ev}, for $0<\alpha<1$, there is a constant
$c_{\rm res}>0$, depending only on
$(K,\lambda)$, such that, if
$n\ge c_{\rm res}\ell(h;\alpha)$, there is an event
$\mathcal E_{\rm res}(\alpha)$ with
$\Pr(\mathcal E_{\rm res}(\alpha))\ge1-\alpha$ on which the following hold
simultaneously:
\begin{enumerate}
 \item[(i)] Every required sample regression is nonsingular, and for every
 $(v,Q)\in\mathcal R(h)$,
 \[
  \|\widehat y-y\|_n
  \le c_{\rm res}\sqrt{\ell(h;\alpha)/n}.
 \]

 \item[(ii)] For every $(v,Q)\in\mathcal R(h)$,
 \[
  \widehat\sigma_y^2,\widehat\sigma_{\widehat y}^2
  \in[\lambda^{-1}/2,2\lambda],
  \qquad
  |\widehat\sigma_y/\sigma_Y-1|
  \le c_{\rm res}\sqrt{\ell(h;\alpha)/n}.
 \]

 \item[(iii)] For every observed variable $X_k$ entering a required
 $\mathsf{UP}$ statistic, writing $\sigma_{X_k}$ and
 $\widehat\sigma_{X_k}$ for its population and empirical standard
 deviations,
 \[
  \widehat\sigma_{X_k}^2\in[\lambda^{-1}/2,2\lambda],
  \qquad
  |\widehat\sigma_{X_k}/\sigma_{X_k}-1|
  \le c_{\rm res}\sqrt{\ell(h;\alpha)/n}.
 \]
\end{enumerate}
\end{lemma}

\begin{proof}
Assume $\mathcal R(h)\ne\varnothing$ so that $N_{\rm test}\ge1$, since otherwise
all three statements hold automatically.
We first construct a covariance event common to the three statements.

	Write $\widehat\Sigma$ for covariance centered by the empirical mean, with
	divisor $n$. Let $c_{\rm cov}>0$ denote the universal numerical constant in
	the mean-centered sub-Gaussian covariance concentration inequality
	\citep{vershynin2026high,chakrabortty2021inference}. Thus, for every index set $H$ with $m_H:=|H|\le h+1$ and every $t>0$,
	\begin{equation}
	\Pr\left(
	\|\widehat\Sigma_{H,H}-\Sigma_{H,H}\|_{\rm op}
	\le
	c_{\rm cov}K^2\lambda
	\left(
	\sqrt{\frac{m_H+t}{n}}
	+\frac{m_H+t}{n}
	\right)
	\right)
	\ge 1-2e^{-t}.
	\label{eq:covariance-concentration}
	\end{equation}

	Associate to each scalar $\mathsf{UP}$ statistic the covariance matrix on $P\cup\{c,k\}$
    and to each scalar $\mathsf{OUT}$ statistic the covariance matrix on $P\cup C\cup\{\ell\}$.
    These matrices contain all submatrices needed for
	their regressions.
    After identifying duplicates, there are at most $N_{\rm test}$ required index sets; denote their collection by $\mathcal H$.

	Set $t=\log(2N_{\rm test}/\alpha)$. By
	Lemma~\ref{lem:count}(ii), $t\le\ell(h;\alpha)$ and
	$m_H\le h+1\le\ell(h;\alpha)$, so
	$m_H+t\le2\ell(h;\alpha)$.
	Hence, by a union bound,
	\[
	\Pr\left(
	\max_{H\in\mathcal H}
	\|\widehat\Sigma_{H,H}-\Sigma_{H,H}\|_{\rm op}
	\le
	c_{\rm cov}K^2\lambda
	\left(
	\sqrt{\frac{2\ell(h;\alpha)}{n}}
	+\frac{2\ell(h;\alpha)}{n}
	\right)
	\right)
	\ge
	1-2N_{\rm test}e^{-t}
	\ge
	1-\alpha.
	\]
	Denote this event by $\mathcal E_{\rm res}(\alpha)$.

	Define the covariance error bound
	\[
	\eta_\Sigma
	:=
	c_{\rm cov}K^2\lambda
	\left(
	\sqrt{\frac{2\ell(h;\alpha)}{n}}
	+\frac{2\ell(h;\alpha)}{n}
	\right).
	\]
	Assume $n\ge C_{\rm samp}\ell(h;\alpha)$,
	where $C_{\rm samp}:=\max\{2,32c_{\rm cov}^2K^4\lambda^4\}$. Then
	\begin{equation}
		\eta_\Sigma
		\le
		2\sqrt{2}\,c_{\rm cov}K^2\lambda
		\sqrt{\frac{\ell(h;\alpha)}{n}}
		\le \frac{\lambda^{-1}}{2}.
		\label{eq:covariance-error-rate}
	\end{equation}
	Since $\lambda_{\min}(\Sigma_{H,H})\ge\lambda^{-1}$, this operator-norm
	bound gives
	\[
	\frac12\Sigma_{H,H}
	\preceq
	\widehat\Sigma_{H,H}
	\preceq
	\frac32\Sigma_{H,H}
	\]
	on every required submatrix.
	In particular, every required sample regression is nonsingular.

	For statement~(i), fix $(v,Q)\in\mathcal R(h)$.
    Here, assume $Q\ne\varnothing$, since $\widehat y=y$ when $Q=\varnothing$.
	For response $X_v$ and predictors $X_Q$, write
	$\beta=\Sigma_{Q,Q}^{-1}\Sigma_{Q,v}$ and
	$\widehat\beta=\widehat\Sigma_{Q,Q}^{-1}\widehat\Sigma_{Q,v}$.
	On the event $\mathcal E_{\rm res}(\alpha)$, for $H=Q\cup\{v\}$,
    $
    \max\left\{
     \|\widehat\Sigma_{Q,Q}-\Sigma_{Q,Q}\|_{\rm op},
     \|\widehat\Sigma_{Q,v}-\Sigma_{Q,v}\|_2
    \right\}
    \le\eta_\Sigma\le\lambda^{-1}/2$.
	Using
	$\|\beta\|_2
	\le
	\|\Sigma_{Q,Q}^{-1}\|_{\rm op}\|\Sigma_{Q,v}\|_2
	\le
	\lambda^2$,
	Theorem~3.1 of \citet{kuchibhotla2023uniform}
    gives
	\[
	\|\widehat\beta-\beta\|_2
	\le
	\frac{\eta_\Sigma(1+\|\beta\|_2)}
	{\lambda^{-1} - \eta_\Sigma}
	\le
	2\lambda\eta_\Sigma
	\left(1+\lambda^2\right).
	\]
	By empirical centering,
	$\widehat y-y=-X_Q(\widehat\beta-\beta)$. Hence
	\begin{align*}
		\|\widehat y-y\|_n
		=
		\left\{
		(\widehat\beta-\beta)^\top
		\widehat\Sigma_{Q,Q}
		(\widehat\beta-\beta)
		\right\}^{1/2}
		\le
		\sqrt{\frac{3\lambda}{2}}\,
		\|\widehat\beta-\beta\|_2
		\le
		\sqrt{\frac{3\lambda}{2}}\,
		2\lambda\eta_\Sigma
		\left(1+\lambda^2\right).
	\end{align*}
    Then for $
    C_{\rm pred} :=\sqrt{\frac{3\lambda}{2}}\,
	4\sqrt2\,c_{\rm cov}K^2\lambda^2(1+\lambda^2)
    $, the inequality \eqref{eq:covariance-error-rate} gives
	\[
    \|\widehat y-y\|_n \le C_{\rm pred}\sqrt{\ell(h;\alpha)/n}.
	\]

	Since the covariance event holds simultaneously for every required
	submatrix, this prediction bound holds uniformly over $\mathcal R(h)$.

	For statement~(ii), if $Q=\varnothing$, then
$\widehat y=y$, and the variance and scale bounds follow
from the covariance event on $\{v\}$.
Assume $Q\ne\varnothing$ and let $H=Q\cup\{v\}$.
    Let $a,\widehat a\in\mathbb R^H$ be the oracle and fitted residual coefficient vectors, respectively, with $a_v=\widehat a_v=1$, $a_Q=-\beta$, and $\widehat a_Q=-\widehat\beta$.
    Consequently,
	\begin{align*}
		\widehat\sigma_y^2
		&=
		a^\top\widehat\Sigma_{H,H}a
		\ge
		(\lambda^{-1}/2)\|a\|_2^2
		\ge
		\lambda^{-1}/2,\\
		\widehat\sigma_{\widehat y}^2
		&=
		\widehat a^\top\widehat\Sigma_{H,H}\widehat a
		\ge
		(\lambda^{-1}/2)\|\widehat a\|_2^2
		\ge
		\lambda^{-1}/2.
	\end{align*}

	The population residual variance is the Schur complement
	$\sigma_Y^2
	=
	\Sigma_{vv}
	-\Sigma_{v,Q}\Sigma_{Q,Q}^{-1}\Sigma_{Q,v}
	\le
	\Sigma_{vv}
	\le
	\lambda$.
	The Loewner upper bound gives
	$\widehat\sigma_y^2\le(3/2)\sigma_Y^2\le3\lambda/2$, and empirical least
	squares minimizes squared error, so
	$\widehat\sigma_{\widehat y}^2\le\widehat\sigma_y^2$.
	This proves the stated two-sided empirical variance bounds.

	Since
	$\sigma_Y^2=a^\top\Sigma_{H,H}a\ge\lambda^{-1}\|a\|_2^2$,
	\[
	\left|
	\frac{\widehat\sigma_y^2}{\sigma_Y^2}-1
	\right|
	=
	\frac{
		|a^\top(\widehat\Sigma_{H,H}-\Sigma_{H,H})a|
	}{\sigma_Y^2}
	\le
	\frac{\eta_\Sigma\|a\|_2^2}{\sigma_Y^2}
	\le
	\lambda\eta_\Sigma
	\le
	C_{\rm scale}\sqrt{\frac{\ell(h;\alpha)}n},
	\]
	where $C_{\rm scale}:=2\sqrt2\,c_{\rm cov}K^2\lambda^2$.
	Since $|\sqrt{x}-1|\le|x-1|$ for $x\ge0$, the same bound holds for $|\widehat\sigma_y/\sigma_Y-1|$.

	For statement~(iii), on $\mathcal E_{\rm res}(\alpha)$, the bounds for an observed variable
	$X_k$ give
	\[
		|\widehat\Sigma_{kk}-\Sigma_{kk}|
		\le
		\eta_\Sigma,
		\qquad
		\widehat\Sigma_{kk}
		\in
		[\lambda^{-1}/2,3\lambda/2],
		\qquad
		\left|
		\frac{\widehat\sigma_{X_k}}{\sigma_{X_k}}-1
		\right|
		\le
		C_{\rm scale}\sqrt{\frac{\ell(h;\alpha)}n}.
	\]
    
	Taking
	\[
	 c_{\rm res}:=\max\{C_{\rm samp},C_{\rm pred},C_{\rm scale}\}
	\]
ensures the required sample-size condition and all three bounds.
\end{proof}

\subsection{Uniform $\mathsf{UP}$ approximation}

\begin{lemma}[$\mathsf{UP}$ covariance approximation]
\label{lem:up-cov}
There exists a constant $c_{\rm up}>0$, depending only on
$(K,\lambda)$, such that for every $0<\delta<1$, if
$n\ge c_{\rm up}\ell(h;\delta),$
then
\[
 \Pr\!\left(
 \max_{(F,C,P,c,k)\in\mathcal I_{\rm up}(h)}
 \left|
 \widehat{\Cov}(\widehat e_c(P),X_k)
 -\Cov(e_c(P),X_k)
 \right|
 \le
 c_{\rm up}\sqrt{\frac{\ell(h;\delta)}{n}}
 \right)
 \ge 1-\frac{\delta}{3}.
\]
\end{lemma}
\begin{proof}
Assume $\mathcal I_{\rm up}(h)\ne\varnothing$, since otherwise the assertion holds automatically.
Then $N_{\rm test}\ge1$. 
For every indexed test, the triangle inequality gives
\begin{align*}
 &|\widehat\Cov(\widehat e_c(P),X_k)-\Cov(e_c(P),X_k)|\\
 &\quad\le
 |\widehat\Cov(\widehat e_c(P),X_k)
       -\widehat\Cov(e_c(P),X_k)|
 +|\widehat\Cov(e_c(P),X_k)-\Cov(e_c(P),X_k)|.
\end{align*}
Since $\ell(h;\delta)\ge\log(32/\delta)\ge\log32$,
\[
 \ell(h;\delta/6)=\ell(h;\delta)+\log6
 \le \left(1+\frac{\log6}{\log32}\right)\ell(h;\delta).
\]
Apply Lemma~\ref{lem:residual-prediction}(i) and (iii) with
$\alpha=\delta/6$.
On that event, Cauchy--Schwarz gives
\[
 |\widehat\Cov(\widehat e_c(P),X_k)
       -\widehat\Cov(e_c(P),X_k)|
 \le \|\widehat e_c(P)-e_c(P)\|_n
       \widehat\sigma_{X_k}
 \le \sqrt{2\lambda}\,c_{\rm res}
 \sqrt{1+\frac{\log6}{\log32}}
 \sqrt{\frac{\ell(h;\delta)}n}.
\]

It remains to control the oracle covariance estimation error.
For each indexed test, set $H=P\cup\{c,k\}$.
If $P=\varnothing$, then $e_c(P)=X_c$, so take $a=e_c$, with $\|a\|_2=1\le\sqrt{1+\lambda^2}$.
If $P\ne\varnothing$, then
\[
 \beta=\Sigma_{P,P}^{-1}\Sigma_{P,c},
 \qquad
 e_c(P)=X_c-\beta^\top X_P=a^\top X_H,
\]
where $a_c=1,\; a_P=-\beta,\; a_{H\setminus(P\cup\{c\})}=0$.
Moreover,
\[
 \lambda^{-1}\|\beta\|_2^2
 \le \beta^\top\Sigma_{P,P}\beta
 \le \Var(X_c)\le\lambda.
\]
Here the second inequality follows from the orthogonal variance decomposition
$\Var(X_c)=\Var(\beta^\top X_P)+\Var(e_c(P))$.
Thus $\|a\|_2\le\sqrt{1+\lambda^2}$.

Applying~\eqref{eq:covariance-concentration} with
$t=\ell(h;\delta)$ and taking a union bound over at most
$N_{\rm test}$ sets $H$, each of size at most $h+1$,
Lemma~\ref{lem:count}(ii) gives, for $n\ge2\ell(h;\delta)$,
\[
 \Pr\!\left(
 \max_H\|\widehat\Sigma_{H,H}-\Sigma_{H,H}\|_{\rm op}
 \le 2\sqrt2\,c_{\rm cov}K^2\lambda
 \sqrt{\frac{\ell(h;\delta)}n}
 \right)
 \ge1-2N_{\rm test}e^{-\ell(h;\delta)}
 \ge1-\frac{\delta}{16}.
\]
Consequently, on this event,
\[
 |\widehat\Cov(e_c(P),X_k)-\Cov(e_c(P),X_k)|
 =|a^\top(\widehat\Sigma_{H,H}-\Sigma_{H,H})e_k|
 \le2\sqrt2\,c_{\rm cov}K^2\lambda\sqrt{1+\lambda^2}
 \sqrt{\frac{\ell(h;\delta)}n}
\]
uniformly over $\mathcal I_{\rm up}(h)$.

Set
\begin{align*}
 c_{\rm up}:=\max\Bigg\{
 \left(1+\frac{\log6}{\log32}\right)c_{\rm res},\ 2,
 \sqrt{2\lambda}\,c_{\rm res}
   \sqrt{1+\frac{\log6}{\log32}}
 +2\sqrt2\,c_{\rm cov}K^2\lambda\sqrt{1+\lambda^2}
 \Bigg\}.
\end{align*}
Combining the two bounds and applying a union bound gives
\begin{align*}
&\Pr\!\left(
 \max_{(F,C,P,c,k)\in\mathcal I_{\rm up}(h)}
 \left|
 \widehat{\Cov}(\widehat e_c(P),X_k)
 -\Cov(e_c(P),X_k)
 \right|
 \le c_{\rm up}\sqrt{\frac{\ell(h;\delta)}n}
 \right)\\
&\qquad\ge 1-\frac{\delta}{6}-\frac{\delta}{16}
>1-\frac{\delta}{3},
\end{align*}
which proves the claim.
\end{proof}

\subsection{Proof of Theorem~\ref{thm:finite-recovery}}

\begin{proof}[Proof of Theorem~\ref{thm:finite-recovery}]
Let $c_{\rm res}$ and $c_{\rm up}$ be the constants in
Lemmas~\ref{lem:residual-prediction} and~\ref{lem:up-cov}.
Choose fixed multipliers satisfying
\[
 \kappa_{\rm up}\ge2\sqrt2c_{\rm up},\qquad
 \kappa_{\rm out}\ge2\sqrt2c_{\mathcal D}.
\]
For each $1\le h\le \min\{p,s_{\max}+d_B\}$, put $\delta_h=\delta/[h(h+1)]$.
By the definition of $\ell$,
\begin{align*}
 &\ell(h;\delta_h)
 =\ell(h;\delta)+\log[h(h+1)]\le2\ell(h;\delta),\\
 &\ell(h;\delta_h/3)
 \le2\ell(h;\delta)+\log3\le3\ell(h;\delta).
\end{align*}
Lemma~\ref{lem:count}(iii)--(iv) therefore gives
\[
 \ell(h;\delta_h)\le2\ell(\min\{p,s_{\max}+d_B\};\delta)\le6L,
 \qquad
 \ell(h;\delta_h/3)\le3\ell(\min\{p,s_{\max}+d_B\};\delta)\le9L.
\]

Write $\widehat D_F^{\rm up}(C,P)$ and $\widehat D_F^{\rm out}(C,P)$ for the empirical counterparts of the two population scores
\[
 \widehat D_F^{\rm up}(C,P)
 :=\max_{c\in C,\,k\in F}
   |\widehat\Cov(\widehat e_c(P),X_k)|,
 \qquad
 \widehat D_F^{\rm out}(C,P)
 :=\max_{c\in C,\,\ell\in R\setminus C}
   \widehat{\mathcal D}\bigl(e_c(P),r_{P,C}(\ell)\bigr),
\]
with empty maxima equal to zero.

Set
\[
c_0:=\max\{9c_{\rm res},6c_{\rm up},6c_{\mathcal D},
          27\kappa_{\rm up}^2,27\kappa_{\rm out}^2\}.
\]
Since the sample-size condition gives $n\ge c_0L$, for every $h$ under consideration,
\begin{align*}
 n\ge9c_{\rm res}L
   \ge c_{\rm res}\ell(h;\delta_h/3),\qquad
 n\ge6c_{\rm up}L
   \ge c_{\rm up}\ell(h;\delta_h),\qquad
 n\ge6c_{\mathcal D}L
   \ge c_{\mathcal D}\ell(h;\delta_h).
\end{align*}
Hence, for each such $h$, we may apply
Lemma~\ref{lem:residual-prediction}(i) and (ii) with
$\alpha=\delta_h/3$, and apply
Lemma~\ref{lem:up-cov} and
Assumption~\ref{ass:uniform-approximation} with
$\delta$ replaced by $\delta_h$.
Their failure probabilities sum to at most $\delta_h$.
Since
\[
 \sum_{h=1}^{\min\{p,s_{\max}+d_B\}}\delta_h
 =\delta\left(1-\frac{1}{\min\{p,s_{\max}+d_B\}+1}
 \right)
 \le\delta,
\]
there is an event $\mathcal E$ of probability at least $1-\delta$ on
which all required regressions for $\mathcal A$ are nonsingular, all
required residual scales are positive, and both scalar score approximation
bounds hold. The inequalities
$|\max_i a_i-\max_i b_i|\le\max_i|a_i-b_i|$ and
$||a|-|b||\le|a-b|$ therefore give, for every candidate of size $h$,
\begin{align}
 |\widehat D_F^{\rm up}(C,P)-D_F^{\rm up}(C,P)|
 &\le\sqrt2c_{\rm up}\sqrt{\ell(h)/n}\le\tau_{\rm up}(h)/2,\nonumber\\
 |\widehat D_F^{\rm out}(C,P)-D_F^{\rm out}(C,P)|
 &\le\sqrt2c_{\mathcal D}\sqrt{\ell(h)/n}\le\tau_{\rm out}(h)/2.
 \label{eq:recovery-score-approximation}
\end{align}
If a margin is infinite, the corresponding violating family is empty.
For each finite margin, the last two terms in the definition of $c_0$ give
\[
 \tau_{\rm up}(h)
 \le\kappa_{\rm up}\sqrt{3L/n}
 \le\Delta_{\rm par}/3,
 \qquad
 \tau_{\rm out}(h)
 \le\kappa_{\rm out}\sqrt{3L/n}
 \le\Delta_{\mathcal D}/3.
\]
A population passer has both scores zero and therefore passes empirically.
For a nonpasser, Proposition~\ref{prop:sparse-passing} gives either
insufficient adjustment or non-ancestry after exact adjustment. Respectively,
\begin{align*}
 \widehat D_F^{\rm up}(C,P)
 &\ge\Delta_{\rm par}-\tau_{\rm up}(h)/2>\tau_{\rm up}(h),\\
 \widehat D_F^{\rm out}(C,P)
 &\ge\Delta_{\mathcal D}-\tau_{\rm out}(h)/2>\tau_{\rm out}(h).
\end{align*}
Thus on $\mathcal E$,
\begin{equation}
 \widehat{\mathsf{PASS}}(C,P)=\mathsf{PASS}(C,P)
 \qquad\text{for every }(F,C,P)\in\mathcal A.
 \label{eq:empirical-population-pass-agreement}
\end{equation}

On $\mathcal E$, we follow the induction in the proof of Theorem~\ref{thm:population-recovery}.
Suppose the invariants \eqref{eq:population-partition-invariant} and \eqref{eq:population-edge-invariant} hold before a round with $F\ne V$; in particular, $F$ is ancestral.
By~\eqref{eq:population-first-passing-size} and the search order, every candidate examined up to and including the first population passer belongs to $\mathcal A$.
Equation~\eqref{eq:empirical-population-pass-agreement} ensures that the empirical search rejects all preceding candidates and accepts this passer.
The invariants hold initially, so the same induction and termination argument gives exact recovery on $\mathcal E$.
\end{proof}

\begin{proof}[Proof of Corollary~\ref{cor:uniform-recovery-optimality}]
In this proof, write $L_{s,d}$ for $L$ evaluated at the class bounds
$(s,d)$, and reserve $L$ for its value at the actual sizes
$(s_{\max},d_B)$ of a model.
Fix a model in the separated class. Its actual sizes satisfy
$s_{\max}\le s$ and $d_B\le d$. Hence the model-specific candidate class $\mathcal A$
is contained in the candidate class defined by $|C|+|P|\le s+d$.
Taking minima over the larger set gives
\[
 \Delta_{\rm par}(M)\ge\gamma_{\rm par},\qquad
 \Delta_{\mathcal D}(M)\ge\gamma_{\mathcal D}
\]
for the model-specific margins used in Theorem~\ref{thm:finite-recovery}.
Since $x\log(ep/x)$ is increasing on $0<x\le p$, the model-specific
complexity satisfies $L\le L_{s,d}$. The search bounds are valid for every model in the class,
so Theorem~\ref{thm:finite-recovery} gives
\[
 n\ge c_0L_{s,d}
 \max\{1,\gamma_{\rm par}^{-2},\gamma_{\mathcal D}^{-2}\}
 \quad\Longrightarrow\quad
 P_M^n(\widehat G^{\rm sc}\ne G^{\rm sc}(M))\le\delta.
\]
The multipliers and constants are common throughout the class. Taking the
supremum proves the uniform recovery bound with the same constant $c_0$.

For optimality, fix $s_0\ge2$, $d_0\ge1$, and $0<\delta<1/2$.
Take $K_0,\lambda_0,\tau_0,c_{\rm lb,0}$ from
Proposition~\ref{prop:fixed-separation-lower}.
The uniform recovery bound with
$(s,d,K,\lambda)=(s_0,d_0,K_0,\lambda_0)$ and
$\gamma_{\rm par}=\gamma_{\mathcal D}=\tau_0$
gives sample complexity $O(\log p)$.

Applying that proposition with $\delta$ replaced by $(1+\delta)/2$
shows that every estimator has worst-case error probability
at least $(1+\delta)/2>\delta$ on the same class whenever
\[
 n\le\frac{1-\delta}{2c_{\rm lb,0}}
 \left\{
 s_0\log\frac{ep}{s_0}+d_0\log\frac{ep}{d_0}
 \right\}.
\]
Since the expression in braces is $\asymp\log p$ for fixed
$s_0,d_0$, the minimax sample complexity on this class is
$\Theta(\log p)$.
\end{proof}

\subsection{Concrete dependence-score estimators}
\label{app:dependence-estimators}

We verify Assumption~\ref{ass:uniform-approximation} for normalized distance
covariance and Gaussian HSIC using complete order-four U-statistic formulas
evaluated on fitted residuals.

\paragraph{Shared residuals.}
For each $(F,C,P,c,\ell)\in\mathcal I_{\rm out}(h)$, use the local
notation
\[
 Y:=e_c(P),
 \qquad
 Z:=r_{P,C}(\ell)=e_\ell(P\cup C),
\]
where the equality follows from the Frisch--Waugh--Lovell identity
\citep{frisch1933partial,lovell1963seasonal}. Let $y,z$ be the empirically
centered arrays obtained from the population residuals $Y,Z$, and let
$\widehat y,\widehat z$ be the fitted residual arrays from the regressions
indexed by $(c,P)$ and $(\ell,P\cup C)$, respectively.

For each estimator, we first bound the estimation error of the complete U-statistic on i.i.d. population residuals, then control the effects of residual estimation and empirical standardization.
This decomposition follows the proof strategy of \citet[Lemmas~4--6]{zhao2022learning} and \citet[Lemmas~9--14]{oh2025optimal}.
Both kernels depend only on pairwise differences, so empirical centering leaves their values unchanged.

\subsubsection{Normalized distance covariance}

For scalar random variables $(Y,Z)$ with finite, positive variances, let
\[
 (Y_i,Z_i),\qquad i=1,2,\ldots,
\]
be i.i.d. copies of $(Y,Z)$. The squared distance covariance admits the
representation
\begin{align}
 \mathcal V^2(Y,Z)
 :=\E\!\left[|Y_1-Y_2||Z_1-Z_2|\right]
 +\E\!\left[|Y_1-Y_2||Z_3-Z_4|\right]
 -2\E\!\left[|Y_1-Y_2||Z_1-Z_3|\right].
\label{eq:dcov}
\end{align}
The standardized squared distance covariance is
\[
 \mathcal D_{\rm dCov}(Y,Z)
 :=\frac{\mathcal V^2(Y,Z)}{\sqrt{\Var(Y)\Var(Z)}}.
\]
It vanishes exactly under independence \citep{szekely2007measuring}.

For arrays $(u,v)$ and a quartet $q=(q_1,q_2,q_3,q_4)$, define
\[
 a_{ij}^q(u):=|u_{q_i}-u_{q_j}|,
 \qquad
 b_{ij}^q(v):=|v_{q_i}-v_{q_j}|,
 \qquad i,j\in[4].
\]
The symmetrized quartet kernel is
\begin{equation}
 \bar h_4(u,v;q)
 =\frac1{4!}\sum_{\pi\in S_4}
 \left\{
 a_{\pi_1\pi_2}^q(u)b_{\pi_1\pi_2}^q(v)
 +a_{\pi_1\pi_2}^q(u)b_{\pi_3\pi_4}^q(v)
 -2a_{\pi_1\pi_2}^q(u)b_{\pi_1\pi_3}^q(v)
 \right\},
 \label{eq:h4}
\end{equation}
where $S_4$ is the set of permutations of $[4]$.
For $n\ge4$, define
\[
 T_{4,n}(u,v)
 :=\binom n4^{-1}
 \sum_{1\le q_1<q_2<q_3<q_4\le n}\bar h_4(u,v;q).
\]

\begin{lemma}[Empirical distance-covariance perturbation]
\label{lem:dcov-lipschitz}
For $n\ge8$, let $u,v,u',v'$ be centered
arrays with empirical norms at most $\overline\sigma$. Then
\[
 |T_{4,n}(u,v)-T_{4,n}(u',v')|
 \le10\overline\sigma(\|u-u'\|_n+\|v-v'\|_n).
\]
If their norms are also at least $\underline\sigma>0$, then
\begin{equation*}
 \left|\frac{T_{4,n}(u,v)}{\|u\|_n\|v\|_n}
       -\frac{T_{4,n}(u',v')}{\|u'\|_n\|v'\|_n}\right|
 \le20\frac{\overline\sigma}{\underline\sigma^2}
 (\|u-u'\|_n+\|v-v'\|_n).
\end{equation*}
\end{lemma}
\begin{proof}
Let $\langle f(q)\rangle$ denote the uniform average over all ordered
four-tuples $q$ of distinct indices in $[n]$. For distinct positions
$r,s\in[4]$, each ordered pair $(i,j)$ occurs $(n-2)(n-3)$ times.
Thus, for any centered array $w$,
\begin{align*}
 \left\langle |w_{q_r}-w_{q_s}|^2\right\rangle
 &=\frac1{n(n-1)}\sum_{i\ne j}(w_i-w_j)^2
 =\frac{2n}{n-1}\|w\|_n^2.
\end{align*}
The reverse triangle inequality, with $w=u-u'$, gives
\[
 \left\langle|a_{rs}^q(u)-a_{rs}^q(u')|^2\right\rangle
 \le\left\langle
 \left|(u_{q_r}-u_{q_s})-(u'_{q_r}-u'_{q_s})\right|^2
 \right\rangle
 =\frac{2n}{n-1}\|u-u'\|_n^2.
\]
For any distinct $t,j\in[4]$, Cauchy--Schwarz and
$ab-a'b'=(a-a')b+a'(b-b')$ therefore yield
\begin{align*}
 \left\langle
 |a_{rs}^q(u)b_{tj}^q(v)-a_{rs}^q(u')b_{tj}^q(v')|
 \right\rangle\le\frac{2n}{n-1}
 \{\|u-u'\|_n\|v\|_n+\|u'\|_n\|v-v'\|_n\}.
\end{align*}
Averaging over ordered four-tuples equals symmetrizing the kernel and
averaging over increasing four-tuples. The absolute coefficient sum in~\eqref{eq:h4} is 4, and $n\ge8$ gives $8n/(n-1)\le10$. Hence,
\begin{align*}
 |T_{4,n}(u,v)-T_{4,n}(u',v')|
 &\le\frac{8n}{n-1}
 \{\|u-u'\|_n\|v\|_n+\|u'\|_n\|v-v'\|_n\}\\
 &\le10\overline\sigma(\|u-u'\|_n+\|v-v'\|_n).
\end{align*}
The same calculation without perturbation gives
\begin{equation*}
 |T_{4,n}(u,v)|\le\frac{8n}{n-1}\|u\|_n\|v\|_n
 \le10\|u\|_n\|v\|_n.
\end{equation*}

For the normalized statistics, the preceding inequalities and the lower
bound $\underline\sigma$ on the empirical norms give
\begin{align*}
 \left|\frac{T_{4,n}(u,v)}{\|u\|_n\|v\|_n}
       -\frac{T_{4,n}(u',v')}{\|u'\|_n\|v'\|_n}\right|
 &\le\frac{|T_{4,n}(u,v)-T_{4,n}(u',v')|}{\|u\|_n\|v\|_n}+\frac{|T_{4,n}(u',v')|
 |\|u'\|_n\|v'\|_n-\|u\|_n\|v\|_n|}
 {\|u\|_n\|v\|_n\|u'\|_n\|v'\|_n},\\
 &\le10\frac{\overline{\sigma}}{\underline{\sigma}^2}(\|u-u'\|_n+\|v-v'\|_n)+\frac{10}{\underline{\sigma}^2}|\|u'\|_n\|v'\|_n-\|u\|_n\|v\|_n|
\end{align*}
Then $|\|u'\|_n\|v'\|_n-\|u\|_n\|v\|_n|\le\overline\sigma(\|u-u'\|_n+\|v-v'\|_n)$ proves the second assertion.
\end{proof}

Define the empirical dependence score by
\begin{equation*}
 \widehat{\mathcal D}_4\bigl(e_c(P),r_{P,C}(\ell)\bigr)
 :=\frac{T_{4,n}(\widehat y,\widehat z)}
 {\widehat\sigma_{\widehat y}\widehat\sigma_{\widehat z}}.
\end{equation*}

\begin{lemma}[Normalized distance-covariance approximation]
\label{lem:k4-normalized}
There exists a constant $c_{\rm dCov}>0$, depending only on $(K,\lambda)$, such that $\widehat{\mathcal D}_4$ satisfies Assumption~\ref{ass:uniform-approximation} for $\mathcal D=\mathcal D_{\rm dCov}$ with $c_{\mathcal D}=c_{\rm dCov}$.
\end{lemma}
\begin{proof}
Assume $\mathcal I_{\rm out}(h)\ne\varnothing$, since otherwise
the assertion holds automatically. Then $N_{\rm test}\ge1$.
Since $\ell(h;\delta)\ge\log32$,
\[
 \ell(h;\delta/6)
 \le\left(1+\frac{\log6}{\log32}\right)\ell(h;\delta).
\]
For each index, write $\widehat{\mathcal D}_4$ and $\mathcal D_{\rm dCov}$ for the corresponding empirical and population scores, respectively.
Let $\widetilde{\mathcal D}_{4,n}(Y,Z) :=T_{4,n}(y,z)/(\widehat\sigma_y\widehat\sigma_z)$ be the statistic computed from oracle residuals on the same observations.
Then
\begin{align*}
 |\widehat{\mathcal D}_4-\mathcal D_{\rm dCov}|
 &\le
 |\widehat{\mathcal D}_4-\widetilde{\mathcal D}_{4,n}|
 +|\widetilde{\mathcal D}_{4,n}-\mathcal D_{\rm dCov}|\\
 &=
 \left|
 \frac{T_{4,n}(\widehat y,\widehat z)}
 {\widehat\sigma_{\widehat y}\widehat\sigma_{\widehat z}}
 -
 \frac{T_{4,n}(y,z)}
 {\widehat\sigma_y\widehat\sigma_z}
 \right|
 +
 \left|
 \frac{T_{4,n}(y,z)}
 {\widehat\sigma_y\widehat\sigma_z}
 -
 \frac{\mathcal V^2(Y,Z)}{\sigma_Y\sigma_Z}
 \right|.
\end{align*}

We first bound the estimation error of the oracle numerator.
Empirical centering leaves pairwise differences unchanged, so
$T_{4,n}(y,z)$ equals the complete U-statistic on the i.i.d.
population residual pairs.
The order-four representation of \citet{huo2016fast} therefore gives
$\E T_{4,n}(y,z)=\mathcal V^2(Y,Z)$.

Write $\|U\|_{\psi_1}:=\inf\{a>0:\E\exp(|U|/a)\le2\}$ for the sub-exponential norm. Let $c_\times>0$ be a universal constant such that
\[
 \|UV\|_{\psi_1}
 \le c_\times\|U\|_{\psi_2}\|V\|_{\psi_2}
\]
for all sub-Gaussian random variables $U,V$. Choose a universal
$c_{\rm Ber}\ge1$ such that every centered $\xi$ with
$\|\xi\|_{\psi_1}\le\nu$ satisfies
\[
 \E e^{t\xi}\le\exp(c_{\rm Ber}^2\nu^2t^2/2),
 \qquad |t|\le(c_{\rm Ber}\nu)^{-1}.
\]
This is the sub-exponential moment-generating-function bound
\citep{vershynin2026high}.

Population residuals are fixed sub-Gaussian linear forms with $\psi_2$
norm at most $K\sqrt\lambda$. Each product of two distances in
\eqref{eq:h4} consequently has $\psi_1$ norm at most
$4c_\times K^2\lambda$. Symmetrization, the total absolute coefficient sum of four, and centering give, for four i.i.d. population residual pairs,
\begin{align*}
 &\xi:=\bar h_4\bigl((Y_i)_{i=1}^4,(Z_i)_{i=1}^4;(1,2,3,4)\bigr)
       -\mathcal V^2(Y,Z),\\
 &\|\xi\|_{\psi_1} \le 2\cdot4\cdot(4c_\times K^2\lambda) =32c_\times K^2\lambda=:\nu.
\end{align*}
By \citet[Theorem~4.4]{pitcan2026concentration}, applied with
sub-exponential parameters $(c_{\rm Ber}^2\nu^2,c_{\rm Ber}\nu)$,
for every $x>\log2$,
\[
 \Pr\!\left(|T_{4,n}(y,z)-\mathcal V^2(Y,Z)|
 >c_{\rm Ber}\nu\left\{
 \sqrt{\frac{2x}{\lfloor n/4\rfloor}}
 +\frac{2x}{\lfloor n/4\rfloor}\right\}\right)
 \le2e^{-x}.
\]
With $x=2\ell(h;\delta)$, $\lfloor n/4\rfloor\ge n/8$ and
$n\ge16\ell(h;\delta)$ give
\begin{align*}
 \sqrt{\frac{4\ell(h;\delta)}{\lfloor n/4\rfloor}}
 +\frac{4\ell(h;\delta)}{\lfloor n/4\rfloor}
 &\le4\sqrt2\sqrt{\frac{\ell(h;\delta)}n}
 +\frac{32\ell(h;\delta)}n\le(4\sqrt2+8)\sqrt{\frac{\ell(h;\delta)}n}
 \le14\sqrt{\frac{\ell(h;\delta)}n}.
\end{align*}
A union bound over at most $N_{\rm test}$ $\mathsf{OUT}$ statistics gives
\begin{equation}
 \begin{aligned}
 &\Pr\!\left(
  \max_{\mathcal I_{\rm out}(h)}
  |T_{4,n}(y,z)-\mathcal V^2(Y,Z)|
  \le448c_{\rm Ber}c_\times K^2\lambda
  \sqrt{\frac{\ell(h;\delta)}n}
  \right)\\
 &\qquad\ge1-2N_{\rm test}e^{-2\ell(h;\delta)}
  \ge1-\frac{\delta}{128}.
 \end{aligned}
 \label{eq:dcov-oracle-numerator}
\end{equation}
The last inequality uses Lemma~\ref{lem:count}(ii) and
$\ell(h;\delta)\ge\log32$.

Now apply Lemma~\ref{lem:residual-prediction}(i) and (ii) with
$\alpha=\delta/6$. On this event, 
\begin{equation}
 |\widehat\sigma_y/\sigma_Y-1|
 +|\widehat\sigma_z/\sigma_Z-1|
 \le2c_{\rm res}\sqrt{1+\frac{\log6}{\log32}}
 \sqrt{\frac{\ell(h;\delta)}n}.
 \label{eq:two-relative-scale-errors}
\end{equation}
Moreover, $\E|Y_1-Y_2|^2=2\sigma_Y^2$ and  $\E|Z_1-Z_2|^2=2\sigma_Z^2$,
so Cauchy--Schwarz applied to \eqref{eq:dcov} yields
\begin{equation}
 |\mathcal V^2(Y,Z)|\le8\sigma_Y\sigma_Z.
 \label{eq:dcov-population-bound}
\end{equation}
If $n\ge16(1+\log6/\log32)c_{\rm res}^2\ell(h;\delta)$, the left-hand side of
\eqref{eq:two-relative-scale-errors} is at most $1/2$.

Thus
$
 {\widehat\sigma_y}/{\sigma_Y},
 {\widehat\sigma_z}/{\sigma_Z}\in[1/2,3/2],
$
and
\begin{align*}
 \left|\frac{\sigma_Y\sigma_Z}
 {\widehat\sigma_y\widehat\sigma_z}-1\right|
 &=\frac{\left|1-
 (\widehat\sigma_y/\sigma_Y)(\widehat\sigma_z/\sigma_Z)\right|}
 {(\widehat\sigma_y/\sigma_Y)(\widehat\sigma_z/\sigma_Z)}
 \le4\left|1-
 \frac{\widehat\sigma_y}{\sigma_Y}\frac{\widehat\sigma_z}{\sigma_Z}\right|
 \le6\left\{
 \left|\frac{\widehat\sigma_y}{\sigma_Y}-1\right|
 +\left|\frac{\widehat\sigma_z}{\sigma_Z}-1\right|
 \right\}.
\end{align*}
On the intersection of the residual-prediction event and the event in~\eqref{eq:dcov-oracle-numerator},
the preceding inequality, \eqref{eq:two-relative-scale-errors}, \eqref{eq:dcov-population-bound}, and
$\widehat\sigma_y\widehat\sigma_z\ge\lambda^{-1}/2$ give
\begin{align}
 |\widetilde{\mathcal D}_{4,n}-\mathcal D_{\rm dCov}|
 &\le
 \frac{|T_{4,n}(y,z)-\mathcal V^2(Y,Z)|}
 {\widehat\sigma_y\widehat\sigma_z}
 +\frac{|\mathcal V^2(Y,Z)|}{\sigma_Y\sigma_Z}
 \left|\frac{\sigma_Y\sigma_Z}
 {\widehat\sigma_y\widehat\sigma_z}-1\right|\notag\\
 &\le\left\{
 \frac{448c_{\rm Ber}c_\times K^2\lambda}{\lambda^{-1}/2}
 +96c_{\rm res}\sqrt{1+\frac{\log6}{\log32}}
 \right\}\sqrt{\frac{\ell(h;\delta)}n}.
 \label{eq:dcov-oracle-normalized}
\end{align}

It remains to bound the difference between the empirical and oracle
scores. On the same event, Lemma~\ref{lem:residual-prediction}(ii)
bounds the empirical norms of all oracle and fitted residual arrays
between $\sqrt{\lambda^{-1}/2}$ and $\sqrt{2\lambda}$.
Lemma~\ref{lem:dcov-lipschitz} and
Lemma~\ref{lem:residual-prediction}(i) therefore give,
uniformly over $\mathcal I_{\rm out}(h)$,
\begin{equation}
 |\widehat{\mathcal D}_4-\widetilde{\mathcal D}_{4,n}|
 \le
 \frac{20\sqrt{2\lambda}}{\lambda^{-1}/2}
 \bigl(\|\widehat y-y\|_n+\|\widehat z-z\|_n\bigr)
 \le
 \frac{40\sqrt{2\lambda}}{\lambda^{-1}/2}
 c_{\rm res}\sqrt{1+\frac{\log6}{\log32}}
 \sqrt{\frac{\ell(h;\delta)}n}.
 \label{eq:dcov-fitted-oracle}
\end{equation}

Combining~\eqref{eq:dcov-fitted-oracle} and
\eqref{eq:dcov-oracle-normalized}, choose
\begin{align*}
 c_{\rm dCov}:=\max\Bigg\{&
 \left(1+\frac{\log6}{\log32}\right)c_{\rm res},
 16,
 16\left(1+\frac{\log6}{\log32}\right)c_{\rm res}^2,\\
 &\frac{40\sqrt{2\lambda}}{\lambda^{-1}/2}
 c_{\rm res}\sqrt{1+\frac{\log6}{\log32}}
 +\frac{448c_{\rm Ber}c_\times K^2\lambda}{\lambda^{-1}/2}
 +96c_{\rm res}\sqrt{1+\frac{\log6}{\log32}}
 \Bigg\}.
\end{align*}
This choice ensures all required sample-size conditions and yields
the uniform error bound in Assumption~\ref{ass:uniform-approximation}
with failure probability at most $\delta/6+\delta/128<\delta/3$,
proving the lemma.
\end{proof}

\subsubsection{Gaussian HSIC}

Fix, independently of the data, a Gaussian-kernel bandwidth
$b\ge b_->0$, and write
\[
 k_b(x,x'):=\exp\left\{-\frac{(x-x')^2}{2b^2}\right\}.
\]
For scalar random variables $(U,V)$, let
$(U_i,V_i)$, $i=1,2,\ldots$, be i.i.d. copies of $(U,V)$ and define
\begin{align*}
 \operatorname{HSIC}_b(U,V)
 :=\E\!\left[k_b(U_1,U_2)k_b(V_1,V_2)\right]
 +\E\!\left[k_b(U_1,U_2)k_b(V_3,V_4)\right]
 -2\E\!\left[k_b(U_1,U_2)k_b(V_1,V_3)\right].
\end{align*}
The population score is
\[
 \mathcal D_{\rm HSIC}(Y,Z)
 :=\operatorname{HSIC}_b(Y/\sigma_Y,Z/\sigma_Z).
\]
It vanishes exactly under independence \citep{gretton2005hsic,sriperumbudur2018tensor}.

For arrays $(u,v)$ and a quartet $q=(q_1,q_2,q_3,q_4)$, define the
symmetrized HSIC kernel
\begin{align}
 \bar h_{4,b}^{\rm HSIC}(u,v;q)
 :=\frac1{4!}
 \sum_{\pi\in S_4}
 \Bigl\{&k_b(u_{q_{\pi_1}},u_{q_{\pi_2}})
 k_b(v_{q_{\pi_1}},v_{q_{\pi_2}}) +k_b(u_{q_{\pi_1}},u_{q_{\pi_2}})
 k_b(v_{q_{\pi_3}},v_{q_{\pi_4}})\notag\\
 &-2k_b(u_{q_{\pi_1}},u_{q_{\pi_2}})
 k_b(v_{q_{\pi_1}},v_{q_{\pi_3}})
 \Bigr\}.
 \label{eq:hsic-h4}
\end{align}
This is the symmetrized order-four kernel for the unbiased HSIC
estimator of \citet[Theorem~3]{song2012feature}.
For $n\ge4$, define the complete statistic
\[
 T_{4,n}^{\rm HSIC}(u,v)
 :=\binom n4^{-1}
 \sum_{1\le q_1<q_2<q_3<q_4\le n}
 \bar h_{4,b}^{\rm HSIC}(u,v;q).
\]
Define the empirical dependence score by
\[
 \widehat{\mathcal D}_{{\rm HSIC},4}
 \bigl(e_c(P),r_{P,C}(\ell)\bigr)
 :=T_{4,n}^{\rm HSIC}\left(
 \frac{\widehat y}{\widehat\sigma_{\widehat y}},
 \frac{\widehat z}{\widehat\sigma_{\widehat z}}
 \right).
\]

\begin{lemma}[Gaussian HSIC approximation]
\label{lem:hsic-admissible}
There exists a constant $c_{\rm HSIC}>0$, depending only on
$(K,\lambda,b_-)$, such that $\widehat{\mathcal D}_{{\rm HSIC},4}$ satisfies
Assumption~\ref{ass:uniform-approximation} for
$\mathcal D=\mathcal D_{\rm HSIC}$ with
$c_{\mathcal D}=c_{\rm HSIC}$.
\end{lemma}
\begin{proof}
Assume $\mathcal I_{\rm out}(h)\ne\varnothing$, since otherwise
the assertion holds automatically. Then $N_{\rm test}\ge1$.
Since $\ell(h;\delta)\ge\log32$,
\[
 \ell(h;\delta/6)
 \le\left(1+\frac{\log6}{\log32}\right)\ell(h;\delta).
\]
For each index, write $\widehat{\mathcal D}_{{\rm HSIC},4}$ and
$\mathcal D_{\rm HSIC}$ for the corresponding empirical and
population scores, respectively.
Define the oracle statistic on the same observations by
\[
 \widetilde{\mathcal D}_{{\rm HSIC},4}
 :=T_{4,n}^{\rm HSIC}\left(\frac y{\sigma_Y},\frac z{\sigma_Z}\right).
\]
Then
$
 |\widehat{\mathcal D}_{{\rm HSIC},4}-\mathcal D_{\rm HSIC}|
 \le
 |\widehat{\mathcal D}_{{\rm HSIC},4}
       -\widetilde{\mathcal D}_{{\rm HSIC},4}|
 +|\widetilde{\mathcal D}_{{\rm HSIC},4}-\mathcal D_{\rm HSIC}|.
$

We first bound the estimation error of the oracle score.
Empirical centering leaves pairwise differences unchanged, so the
oracle statistic equals the complete unbiased HSIC statistic on the
i.i.d. pairs $(Y_i/\sigma_Y,Z_i/\sigma_Z)$. Since $0\le k_b\le1$,
\citet[Theorem~4]{song2012feature} gives, for every $0<\eta<1$,
\[
 \Pr\!\left(
 |\widetilde{\mathcal D}_{{\rm HSIC},4}-\mathcal D_{\rm HSIC}|
 >8\sqrt{\frac{\log(2/\eta)}n}\right)\le\eta.
\]
Take $\eta=2e^{-2\ell(h;\delta)}$.
By Lemma~\ref{lem:count}(ii) and $\ell(h;\delta)\ge\log32$, we have
\[
N_{\rm test}\eta
=2N_{\rm test}e^{-2\ell(h;\delta)}
\le\delta/128.
\]
A union bound over at most $N_{\rm test}$ $\mathsf{OUT}$
statistics therefore gives, with probability at least $1-\delta/128$,
\begin{equation}
 \max_{\mathcal I_{\rm out}(h)}
 |\widetilde{\mathcal D}_{{\rm HSIC},4}-\mathcal D_{\rm HSIC}|
 \le8\sqrt2\sqrt{\frac{\ell(h;\delta)}n}.
 \label{eq:hsic-oracle-bound}
\end{equation}

It remains to bound the difference between the empirical and oracle
scores. Apply Lemma~\ref{lem:residual-prediction}(i) and (ii) with
$\alpha=\delta/6$. On the residual-prediction event,
\[
 \left\|\frac{\widehat y}{\widehat\sigma_{\widehat y}}
                  -\frac{y}{\sigma_Y}\right\|_n
 \le
 \frac{\|\widehat y-y\|_n}{\widehat\sigma_{\widehat y}}
 +\frac{\|y\|_n
            |\widehat\sigma_{\widehat y}-\sigma_Y|}
           {\widehat\sigma_{\widehat y}\sigma_Y}.
\]
The reverse triangle inequality gives $|\widehat\sigma_{\widehat y}-\sigma_Y| \le\|\widehat y-y\|_n+|\widehat\sigma_y-\sigma_Y|$.
Substituting the bounds from Lemma~\ref{lem:residual-prediction}(i) and (ii) into the preceding inequalities gives, uniformly over $\mathcal I_{\rm out}(h)$,
\begin{equation}
 \left\|\frac{\widehat y}{\widehat\sigma_{\widehat y}}
                  -\frac{y}{\sigma_Y}\right\|_n
 \le c_{\rm std}\sqrt{\frac{\ell(h;\delta)}n}.
 \label{eq:standardized-residual-bound}
\end{equation}
Here
\[
 c_{\rm std}
 :=\sqrt{1+\frac{\log6}{\log32}}\left[
 \frac{c_{\rm res}}{\sqrt{\lambda^{-1}/2}}
 +\frac{\sqrt{2\lambda}}
 {\sqrt{\lambda^{-1}/2}\sqrt{\lambda^{-1}}}
   \{c_{\rm res}+c_{\rm res}\sqrt{\lambda}\}\right].
\]
The same bound holds for $z$.

For each observation, define
\[
 \delta_i^Y:=\left|
 \frac{\widehat y_i}{\widehat\sigma_{\widehat y}}-
 \frac{y_i}{\sigma_Y}\right|,
 \qquad
 \delta_i^Z:=\left|
 \frac{\widehat z_i}{\widehat\sigma_{\widehat z}}-
 \frac{z_i}{\sigma_Z}\right|.
\]
The mean-value theorem, together with
$\sup_{t\in\mathbb R}\left|\frac{d}{dt}e^{-t^2/(2b^2)}\right|=e^{-1/2}/b$, gives
\[
 |k_b(u,v)-k_b(u',v')|
 \le \frac{e^{-1/2}}{b_-}(|u-u'|+|v-v'|).
\]
Since Gaussian kernel values lie in $[0,1]$,
$|ab-a'b'|\le|a-a'|+|b-b'|$ applies to each kernel product.
For a quartet $q=(q_1,q_2,q_3,q_4)$, these inequalities bound
the absolute difference between the fitted and oracle evaluations
of the order-four kernel by
\[
 \frac{4e^{-1/2}}{b_-}
 \sum_{a=1}^4(\delta_{q_a}^Y+\delta_{q_a}^Z).
\]
In the complete average, each index belongs to exactly $\binom{n-1}{3}$
quartets, so Cauchy--Schwarz gives
\begin{align*}
 &\binom n4^{-1}
 \sum_{1\le q_1<q_2<q_3<q_4\le n}
 \sum_{a=1}^4(\delta_{q_a}^Y+\delta_{q_a}^Z)\\
 &\quad=\frac{\binom{n-1}{3}}{\binom n4}
 \sum_{i=1}^n(\delta_i^Y+\delta_i^Z)
 =\frac4n\sum_{i=1}^n(\delta_i^Y+\delta_i^Z)
 \le4(\|\delta^Y\|_n+\|\delta^Z\|_n).
\end{align*}
Applying \eqref{eq:standardized-residual-bound} to both arrays then gives
\begin{equation}
 |\widehat{\mathcal D}_{{\rm HSIC},4}
      -\widetilde{\mathcal D}_{{\rm HSIC},4}|
 \le\frac{16e^{-1/2}}{b_-}
 (\|\delta^Y\|_n+\|\delta^Z\|_n)
 \le\frac{32e^{-1/2}}{b_-}c_{\rm std}
 \sqrt{\frac{\ell(h;\delta)}n}.
 \label{eq:hsic-fitted-oracle}
\end{equation}

On the intersection of the event from Lemma~\ref{lem:residual-prediction}(i) and (ii) with $\alpha=\delta/6$
and the event in~\eqref{eq:hsic-oracle-bound}, both error bounds hold.
Combining~\eqref{eq:hsic-oracle-bound} and \eqref{eq:hsic-fitted-oracle}, choose
\begin{align*}
 c_{\rm HSIC}:=\max\Bigg\{&
 \left(1+\frac{\log6}{\log32}\right)c_{\rm res},
 \frac4{\log32},
 8\sqrt2+
 \frac{32e^{-1/2}}{b_-}c_{\rm std}
 \Bigg\}.
\end{align*}
This choice ensures all required sample-size conditions and yields
the uniform error bound in Assumption~\ref{ass:uniform-approximation}
with failure probability at most $\delta/6+\delta/128<\delta/3$,
proving the lemma.
\end{proof}

\section{Separation and optimality}
\label{app:separation}

Vanishing cycles explain the need for quantitative separation.
The lower bound with fixed separation supplies the necessity argument for the optimality conclusion in Corollary~\ref{cor:uniform-recovery-optimality}.
We retain the noise density $f$, $g=\log f$, the laws $P_B$, and the constants $c_{\rm KL},K,\lambda$ fixed in Appendix~\ref{app:lower}.
Class margins use candidates with $|C|+|P|\le s+d$, as in Section~\ref{sec:structural-optimality}.

\subsection{Vanishing cyclic separation}

\begin{proposition}[No uniform recovery without cyclic separation]
\label{prop:weak-cycle}
Let
\[
 B_0=0_{2\times2},
 \qquad
 B_\gamma=
 \begin{bmatrix}
 0&\gamma\\
 \gamma&0
 \end{bmatrix},
\]
and write $\Pi(B)$ for the SCC partition generated by $B$.
Then, for every $0<\gamma\le1/2$ and every estimator $\widehat\Pi$ based
on $n$ observations,
\[
 n\le\frac{1}{16c_{\rm KL}\gamma^2}
 \quad\Longrightarrow\quad
 \max\left\{
 P_{B_0}^n\bigl(\widehat\Pi\ne\Pi(B_0)\bigr),
 P_{B_\gamma}^n\bigl(\widehat\Pi\ne\Pi(B_\gamma)\bigr)
 \right\}
 \ge\frac38.
\]
\end{proposition}

\begin{proof}
The two models have different SCC partitions: $B_0$ has two singleton
SCCs, whereas $B_\gamma$ has one two-node SCC. Moreover,
$\|B_\gamma\|_{\rm op}=\gamma\le1/2$ and
$\det(I-B_\gamma)=1-\gamma^2>0$.
Thus both models satisfy Assumption~\ref{ass:PI}, and
Lemma~\ref{lem:packing-regularity} gives common constants for
Assumption~\ref{ass:sg-ev}.

Since $\|B_\gamma\|_F^2=2\gamma^2$, Lemma~\ref{lem:kl-control} and tensorization give
\[
 \operatorname{KL}(P_{B_\gamma}^n\|P_0^n)
 \le2c_{\rm KL}n\gamma^2.
\]
Any estimator of the SCC partition induces a test between the two models.
Pinsker's inequality and Le Cam's two-point bound
\citep[Ch.~2]{tsybakov2009introduction} therefore imply
\begin{align*}
 \max\left\{
 P_{B_0}^n\bigl(\widehat\Pi\ne\Pi(B_0)\bigr),
 P_{B_\gamma}^n\bigl(\widehat\Pi\ne\Pi(B_\gamma)\bigr)
 \right\}
 &\ge
 \frac{1-\operatorname{TV}(P_{B_\gamma}^n,P_0^n)}{2}
 \\
 &\ge
 \frac{1-\sqrt{\operatorname{KL}(P_{B_\gamma}^n\|P_0^n)/2}}{2}
 \\
 &\ge
 \frac{1-\sqrt{c_{\rm KL}n\gamma^2}}{2}.
\end{align*}
If $n\le(16c_{\rm KL}\gamma^2)^{-1}$, the last expression is at least $3/8$,
which proves the proposition.
\end{proof}

\subsection{Lower bound with fixed separation}

We construct uniformly separated model families by adding isolated variables to fixed finite models.

\begin{lemma}[Isolated padding and fixed-motif separation]
\label{lem:padding}
Let $B^\circ$ define a fixed LiNG model on a node set $W$ of size $m_0$
satisfying Assumption~\ref{ass:PI}. For $p\ge m_0$, embed this model in $[p]$ and let
every coordinate in $W^c$ be an isolated noise variable with density $f$.
Fix integers $s\ge1$ and $d\ge0$, and a dependence functional
$\mathcal D$ satisfying the independence zero-set condition in
Section~\ref{sec:scores}. Define $\Delta_{\rm par}$ and
$\Delta_{\mathcal D}$ as in \eqref{eq:dep-margin}, with the candidate-size
bound $|C|+|P|\le s+d$. Then both margins belong to $(0,\infty]$ and depend
only on $(B^\circ,s,d,\mathcal D)$, not on $p$ or the embedding of $W$.
\end{lemma}

\begin{proof}
The coordinates in $W^c$ are mutually independent and independent of
$X_W$. For a candidate $(F,C,P)$, set
$F_W=F\cap W$, $C_W=C\cap W$, and $P_W=P\cap W$, and use a superscript
$\circ$ for quantities in the motif model.
If $C_W=\varnothing$, then $e_C(P)=e_C(F)$ and $C$ is ancestral,
so neither margin includes this candidate.
Hence assume $C_W\ne\varnothing$.

Since
\[
 \Cov(X_P)
 =
 \begin{bmatrix}
  \Cov(X_{P_W})&0\\
  0&\Cov(X_{P\setminus W})
 \end{bmatrix},
\]
the normal equations give, for $v\notin P$,
\[
 e_v(P)=
 \begin{cases}
  e_v^\circ(P_W),&v\in W,\\
  X_v,&v\notin W.
 \end{cases}
\]
Consequently,
\[
 e_C(P)
 =
 \bigl(e_{C_W}^\circ(P_W),X_{C\setminus W}\bigr),
\]
and the two blocks are independent. The normal equations for the projection
of $e_\ell(P)$ onto $e_C(P)$ therefore give
\[
 r_{P,C}(\ell)=
 \begin{cases}
  r_{P_W,C_W}^\circ(\ell),&\ell\in W,\\
  X_\ell,&\ell\notin W.
 \end{cases}
\]
Consequently, every $\mathsf{UP}$ or $\mathsf{OUT}$ score is zero if either tested variable lies outside $W$; otherwise, it equals the corresponding motif score.
Taking empty maxima as zero,
\[
 D_F^{\rm up}(C,P)
 =D_{F_W}^{\rm up,\circ}(C_W,P_W),
 \qquad
 D_F^{\rm out}(C,P)
 =D_{F_W}^{\rm out,\circ}(C_W,P_W).
\]

The first residual identity, applied with $P$ and with $F$, also gives
\[
 e_C(P)=e_C(F)
 \quad\Longleftrightarrow\quad
 e_{C_W}^\circ(P_W)=e_{C_W}^\circ(F_W).
\]
Because every node outside $W$ is isolated, ancestry is preserved under
the same reduction:
\[
 \begin{aligned}
 F\text{ is ancestral in }G
 &\quad\Longleftrightarrow\quad
 F_W\text{ is ancestral in }G^\circ,\\
 C\text{ is ancestral in }G[V\setminus F]
 &\quad\Longleftrightarrow\quad
 C_W\text{ is ancestral in }G^\circ[W\setminus F_W].
 \end{aligned}
\]
Thus restricting $(F,C,P)$ to $W$ preserves both the conditions
defining each margin and the corresponding scores.
The restricted candidates $(F_W,C_W,P_W)$ satisfy
\[
 F_W\subseteq W\text{ ancestral},\qquad
 \varnothing\ne C_W\subseteq W\setminus F_W,\qquad
 P_W\subseteq F_W,
 \qquad |C_W|+|P_W|\le s+d.
\]
Conversely, every such motif tuple is realized by taking $(F,C,P)=(F_W,C_W,P_W)$.

For every reduced tuple in the family defining
$\Delta_{\rm par}$,
\[
 e_{C_W}^\circ(P_W)\ne e_{C_W}^\circ(F_W)
 \quad\Longrightarrow\quad
 D_{F_W}^{\rm up,\circ}(C_W,P_W)>0,
\]
because a zero $\mathsf{UP}$ score is equivalent to equality of the two
residuals.

Likewise, consider a restricted candidate satisfying $e_{C_W}^\circ(P_W)=e_{C_W}^\circ(F_W)$ with $C_W$ non-ancestral in $G^\circ[W\setminus F_W]$.
The residual equality gives $\mathsf{UP}$.
If its $\mathsf{OUT}$ score were zero, the pair would pass, so Proposition~\ref{prop:sparse-passing} would imply that $C_W$ is ancestral, a contradiction.
Thus $D_{F_W}^{\rm out,\circ}(C_W,P_W)>0$.

The two margins are therefore minima over finite collections of positive motif scores depending only on $(B^\circ,s,d,\mathcal D)$.
Hence they belong to $(0,\infty]$, with an empty minimum equal to $+\infty$.
\end{proof}

We now apply this construction to the separated model class in
Section~\ref{sec:structural-optimality}, with a common separation
level for the two margins.

\begin{proposition}[Lower bound with fixed separation]
\label{prop:fixed-separation-lower}
Fix $s_0\ge2$, $d_0\ge1$, and a normalized dependence functional
$\mathcal D$ satisfying the independence zero-set condition. There exist
constants
\[
 K_0,\tau_0,c_{\rm lb,0}>0,\qquad \lambda_0\ge1,
\]
depending only on $(s_0,d_0,f,\mathcal D)$ and not on $(p,n,\delta)$, such
that the following holds. For every $p\ge2\max\{s_0,d_0\}$, sample size
$n\ge1$, and $0<\delta<1$, every estimator $\widehat G^{\rm sc}$ satisfies
\[
 n\le\frac{1-\delta}{c_{\rm lb,0}}
 \left\{
  s_0\log\frac{ep}{s_0}
  +d_0\log\frac{ep}{d_0}
 \right\}
 \quad\Longrightarrow\quad
 \sup_{M\in
 \mathcal M^{\rm sep}_{p,s_0,d_0}(K_0,\lambda_0;\tau_0,\tau_0)}
 P_M^n\left(\widehat G^{\rm sc}\ne G^{\rm sc}(M)\right)
 \ge\delta.
\]
\end{proposition}

\begin{proof}
We first construct a cyclic family. Set $\gamma=1/8$. For every $S=\{j_1<\cdots<j_{s_0}\}\subseteq[p]$, let
$B_S^{\rm cyc}$ represent a directed cycle on $S$ with coefficient
$\gamma$ on each edge and all nodes in $S^c$ isolated:
\[
(B_S^{\rm cyc})_{j_{r+1},j_r}=\gamma\quad
\text{for }1\le r<s_0,\qquad
(B_S^{\rm cyc})_{j_1,j_{s_0}}=\gamma,
\]
with all other entries equal to zero.

The principal block $(B_S^{\rm cyc})_{S,S}$ is the permutation matrix of a directed $s_0$-cycle, scaled by $\gamma$.
All entries outside this block are zero. Hence
\[
 \|B_S^{\rm cyc}\|_{\rm op}=\gamma,
 \qquad
 \|B_S^{\rm cyc}\|_F^2=s_0\gamma^2=\frac{s_0}{64}.
\]
Every principal submatrix of $B_S^{\rm cyc}$ has operator norm at most $\gamma<1$, so every principal submatrix of $I-B_S^{\rm cyc}$ is invertible.
Hence every model in this family satisfies Assumption~\ref{ass:PI}.
By construction, $s_{\max}=s_0$ and $d_B=0$.


We next construct the parent family using the packing in Lemma~\ref{lem:parent-packing} with $d=d_0$.
Every model in that family is acyclic, satisfies $s_{\max}=1$, and has $d_B=d_0$.
Since both families have operator norm at most $1/2$, Lemma~\ref{lem:packing-regularity} gives the common class constants $K_0:=K$ and $\lambda_0:=\lambda$.
Thus both families lie in $\mathcal M_{p,s_0,d_0}(K_0,\lambda_0)$.

We next verify uniform separation. Each cyclic model is a fixed
$s_0$-node motif padded with isolated coordinates, and each parent model
is a fixed $(d_0+1)$-node motif padded with isolated coordinates.
Lemma~\ref{lem:padding} therefore shows that all corresponding margins are
positive or $+\infty$ and independent of $p$ and of the embedding.

For the cyclic family, there are no edges between distinct SCCs. Hence every
ancestral peeling state $F$ satisfies $X_F\indep X_{V\setminus F}$, and
$e_C(P)=X_C=e_C(F)$ for every candidate in the definition of the margins.
Thus $\Delta_{\rm par}^{\rm cyc}=+\infty$. Let
$\Delta_{\mathcal D}^{\rm cyc}$, $\Delta_{\rm par}^{\rm par}$, and
$\Delta_{\mathcal D}^{\rm par}$ denote the remaining motif margins, and set
\[
 \tau_0
 :=
 \min\left\{
  1,\,
  \Delta_{\mathcal D}^{\rm cyc},\,
  \Delta_{\rm par}^{\rm par},\,
  \Delta_{\mathcal D}^{\rm par}
 \right\}.
\]
Then $\tau_0>0$, and every model in either packing belongs to
$\mathcal M^{\rm sep}_{p,s_0,d_0}(K_0,\lambda_0;\tau_0,\tau_0)$.

It remains to apply Fano's inequality. Write
$h_{s_0}=s_0\log(ep/s_0)$ and $h_{d_0}=d_0\log(ep/d_0)$. The cyclic and
parent families contain $N_{s_0}=\binom{p}{s_0}$ and
$N_{d_0}=\binom{p-1}{d_0}$ distinct condensations, respectively. Since
$\binom{p}{s_0}\ge\binom{p-1}{s_0-1}$,
Lemmas~\ref{lem:star-packing}(ii) and~\ref{lem:parent-packing}(ii) give
\[
 \log N_{s_0}
 \ge\frac{\log2}{2(1+\log2)}h_{s_0},
 \qquad
 \log N_{d_0}
 \ge\frac{\log2}{2(1+\log2)}h_{d_0}.
\]
Lemma~\ref{lem:kl-control}, tensorization, and
Lemma~\ref{lem:parent-packing}(iii) give
\[
 \max_S
 \operatorname{KL}(P_{B_S^{\rm cyc}}^n\|P_0^n)
 \le\frac{c_{\rm KL}s_0n}{64},
 \qquad
 \max_Q
 \operatorname{KL}(P_{B_Q}^n\|P_0^n)
 \le\frac{c_{\rm KL}n}{4}.
\]

Choose the cyclic family if $h_{s_0}\ge h_{d_0}$ and the parent family
otherwise, and let $N_0$ be its cardinality. Its log-cardinality is at least
$\log2\,(h_{s_0}+h_{d_0})/[4(1+\log2)]$, and its average KL divergence
from $P_0^n$ is at most
$c_{\rm KL}\max\{s_0/64,1/4\}n$. Set
\[
 c_{\rm lb,0}
 :=
 \frac{
  4(1+\log2)c_{\rm KL}\max\{s_0/64,1/4\}
 }{\log2}
 +4(1+\log2).
\]
Generalized Fano's inequality and $n\ge1$ give
\begin{align*}
 &\sup_{M\in
 \mathcal M^{\rm sep}_{p,s_0,d_0}(K_0,\lambda_0;\tau_0,\tau_0)}
 P_M^n\left(\widehat G^{\rm sc}\ne G^{\rm sc}(M)\right)
 \\
 &\quad\ge 1-
 \frac{4(1+\log2)c_{\rm KL}\max\{s_0/64,1/4\}n}
 {\log2\,(h_{s_0}+h_{d_0})}
 -\frac{4(1+\log2)}{h_{s_0}+h_{d_0}}
 \ge 1-\frac{c_{\rm lb,0}n}{h_{s_0}+h_{d_0}}.
\end{align*}
Under the sample-size condition in the proposition, the right-hand side is
at least $\delta$.
\end{proof}

\section{Experimental details and additional results}
\label{app:empirical}

\subsection{Experimental protocol}
\label{app:fullsample-protocol}

\paragraph{Graph structures and coefficients.}
The main experiments use two graph constructions with one or four cyclic SCCs.
Within each cyclic SCC of size $s_{\max}$, a center is connected reciprocally to the other $s_{\max}-1$ nodes, with coefficient $0.5/\sqrt{s_{\max}-1}$ on each edge, keeping the spectral radius at $0.5$ across SCC sizes.
For $s_{\max}=2$, this forms a 2-cycle; for $s_{\max}=3,4$, it forms overlapping 2-cycles sharing the center.
Each external edge has coefficient $0.8/\sqrt{d_B}$.

In the one-SCC construction, nodes $1,\ldots,s_{\max}$ form a root SCC with center 1, and there are $d_B-1$ additional singleton roots.
Each of the remaining $p-s_{\max}-d_B+1$ singleton children has node 1 and all additional roots as parents.
Increasing $p$ adds children while preserving the cyclic core.
Figure~\ref{fig:main-evidence} uses this construction.

The four-SCC construction has one root and three child cyclic SCCs, each of size $s_{\max}$, together with $d_B-1$ singleton roots.
The root SCC's center and all singleton roots are parents of each child SCC's center and all $p-4s_{\max}-d_B+1$ singleton children.
Figure~\ref{fig:fullsample-baseline-m4-preview} uses this construction.

\paragraph{Data generation.}
Given a coefficient matrix $B$, we generate independent observations of $X=(I-B)^{-1}\eps$. 
Unless otherwise specified, the coordinates of $\eps$ are independent and follow the standardized mixture $(W+1.6)/\sqrt{1.45}$, where $W\sim0.9\,N(-2,0.01)+0.1\,N(2,0.01)$.
Samples at different sample sizes are generated independently.

\paragraph{Method settings.}
BlockExo uses centered full-sample OLS and the absolute empirical covariances for $\mathsf{UP}$.
For $\mathsf{OUT}$ it uses the complete unbiased estimator of squared distance covariance, normalized by the empirical residual scales of the centered fitted residuals $y,z$:
\[
 \widehat{\mathcal D}(y,z)
 =\frac{\widehat{\operatorname{dCov}}_{U}^{\,2}(y,z)}
 {\sqrt{(n^{-1}\sum_i y_i^2)(n^{-1}\sum_i z_i^2)}}.
\]
The signed estimate is used without absolute values or clipping.
We compute the statistic using the \texttt{dcor} package \citep{ramos2023dcor}.
Previously computed statistics are cached across recovery rounds.
Thresholds are $\tau_j(t)=\kappa_j\sqrt{\ell(t)/n}$ at $\delta=0.05$,
with empirical multipliers $\kappa_{\rm up}=1.68$ and
$\kappa_{\rm out}=0.22$, selected using separate simulated datasets
and held fixed except in the noise-distribution comparisons below.

Baseline thresholds were selected using separate simulated datasets
and held fixed during evaluation.
For Coarsening, we report both the default threshold $0.1$ and the
selected threshold $0.2$ \citep{madaleno2026coarsening}.
Each threshold sets the initial unmixing cutoff and the edge cutoff; only the former is relaxed if no admissible matching exists.
DisjointCycles uses $\alpha=0.05$ \citep{drton2025disjoint}; the overlapping construction lies outside its disjoint-cycle assumption.
StableSpIn uses 3,000 epochs and edge cutoff $\omega=0.15$ \citep{misiakos2026stablespin}.

For the noise-distribution comparisons in Appendix~\ref{app:noise-comparison},
$\kappa_{\rm up}$ remains fixed, while $\kappa_{\rm out}=0.04$ for the
symmetric Gaussian mixture and $0.01$ for uniform noise.
These values were selected using separate simulated datasets from each
noise law and fixed before evaluation.

Using 36 selection datasets per distribution shared across the baselines,
we evaluate Coarsening
thresholds $\{0.025,0.05,0.1,0.2,0.3\}$, DisjointCycles levels
$\{0.0001,0.001,0.01,0.05,0.1\}$, and StableSpIn cutoffs
$\{0.025,0.05,0.09,0.15,0.25\}$.
Selection prioritizes exact recovery, then partition recovery and the
number of correctly recovered cyclic SCCs, with remaining ties resolved
by preference for the default and proximity to it.
The selected values are $0.2$, $0.01$, and $0.09$, respectively, for both
noise laws, with DisjointCycles and StableSpIn retaining their defaults
because all three recovery criteria are zero throughout their grids.
Coarsening $0.1$ is also retained.

The exact GroupLiNGAM search \citep{kawahara2010grouplingam} is infeasible at $p=50$ under our 48\,GiB memory budget.
Its large-graph approximation, L-GroupLiNGAM, targets approximate block ordering rather than exact condensation recovery and is therefore omitted.

\paragraph{Evaluation.}
All methods are evaluated on the same datasets across 20 replicates.
For methods returning variable-level graphs, we extract the SCC partition and condensation edges.
A replicate counts as an exact recovery when both the SCC partition and all condensation edges are correct, with components compared by their node sets.
Fits that do not complete, including those returning no passing block, count as recovery failures.

\subsection{Separation and sample cost}
\label{app:margin-audit}

We examine the contributions of structural complexity and separation to sample cost in the one-SCC construction.
In the dimension comparison, the hybrid margin is constant within each $d_B$, while the external-parent comparison changes both structural complexity and separation.

To compare recovery curves on the scale suggested by the sample bound, we normalize sample size as
\[
 \widehat z=\frac{n\widehat\Delta_{\rm hyb}^{\,2}}{L},
 \qquad
 \widehat\Delta_{\rm hyb}
 =\min\{\Delta_{\rm par}/\kappa_{\rm up},
         \widehat\Delta_{\mathcal D}/\kappa_{\rm out}\}.
\]
Parent margins are evaluated from population covariance, and dependence margins are estimated from reference samples generated independently of the recovery datasets.
Each setting has a single reference margin for normalizing its recovery curve.

\paragraph{Dimension dependence.}
At fixed $s_{\max}=3$ and $d_B$, the parent margin determines the same hybrid margin across $p=10,30,50$ in Figure~\ref{fig:main-evidence}(b).
Thus $n/L$ and $\widehat z$ differ by a constant factor within each group, and the closer alignment after normalization supports the dimension dependence captured by $L$.
The hybrid margins are approximately $0.201$, $0.132$, and $0.091$ for $d_B=1,2,3$, respectively, and agree across dimensions at numerical precision.

\paragraph{External-parent count.}
In this construction, increasing $d_B$ raises the structural cost and reduces the hybrid margin.
Moving from $d_B=1$ to $d_B=2$ expands the adjustment search and reduces
each external-edge coefficient from $0.8$ to $0.8/\sqrt{2}$.
For $s_{\max}=3$ and $p=50$, $L$ increases from approximately $28.5$ to
$32.0$, while $\widehat\Delta_{\rm hyb}$ decreases from $0.201$ to $0.132$.
Together, these changes increase $L/\widehat\Delta_{\rm hyb}^{\,2}$ by a
factor of approximately $2.58$, consistent with BlockExo's recovery transition
at larger sample sizes for $d_B=2$ in Figure~\ref{fig:main-evidence}(b).

\paragraph{Reference margin calculation.}
The reference calculations evaluate~\eqref{eq:dep-margin} over the full candidate class $\mathcal A$:
all ancestral sets $F$ and candidates satisfying $|C|+|P|\le s_{\max}+d_B$, without separate restrictions on $|C|$ and $|P|$.
Dependence margins are estimated using the complete dCov statistic on sixteen independent reference samples of 8,192 observations, with population regression coefficients and residual scales.
Scalar scores are averaged across these replicates before forming candidate maxima and the minimum over invalid candidates.
These calculations cover every setting in Figure~\ref{fig:main-evidence}.

\subsection{Dimension comparison with two cyclic SCCs}
\label{app:two-scc-dimension}

Normalization by $L$ brings recovery transitions across dimensions closer together within each $d_B$, also for graphs with a root and a child cyclic SCC (Figure~\ref{fig:appendix-two-scc-dimension}).
This extends the dimension comparison in Figure~\ref{fig:main-evidence}(b).

We extend the one-SCC construction in Appendix~\ref{app:fullsample-protocol} at $s_{\max}=3$ by adding a three-node child SCC with the same internal coefficients.
The center of the root SCC and the $d_B-1$ additional singleton roots are parents of the child SCC's center and all singleton children, with coefficient $0.8/\sqrt{d_B}$ on each external edge.
Adding singleton children varies $p\in\{10,30,50\}$ while $s_{\max}=3$ and the two cyclic SCCs remain fixed.
We consider $d_B\in\{1,2\}$ and use search bounds $(s,d)=(s_{\max},d_B)$.

\begin{figure}[!htbp]
\centering
\includegraphics[width=\linewidth]{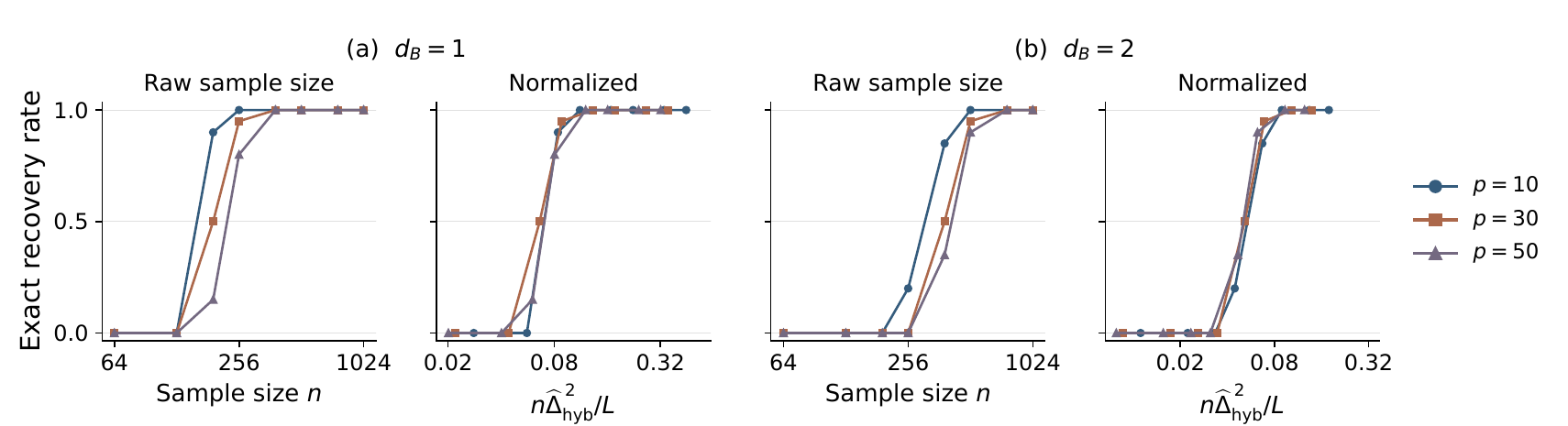}
\caption{Exact condensation recovery across dimensions with two cyclic SCCs.}
\label{fig:appendix-two-scc-dimension}
\end{figure}

The parent margin determines the hybrid margin, which is approximately $0.0946$ for $d_B=1$ and $0.0624$ for $d_B=2$ across $p=10,30,50$.
These values remain unchanged despite the decrease in the estimated dependence margin with dimension at $d_B=2$.
Thus $L$ accounts for the observed differences in sample cost across dimensions within each $d_B$.

Reference margins are computed over $\mathcal A$ following Appendix~\ref{app:margin-audit}, with dependence scores estimated from four independent reference samples of 16,384 observations.

\subsection{External-edge strength and dependence separation}
\label{app:margin-meaning}

Dependence separation cannot be explained by external-edge strength or covariance alone.
We examine these distinctions by varying external-edge coefficients and noise distributions in two separate experiments.

Both experiments use the candidate class $\mathcal A$ and scalar-score calculations described in Appendix~\ref{app:margin-audit}.
Parent margins are computed from population covariance, and dependence margins are estimated from 16 independent reference samples of 8,192 observations.

\paragraph{External-edge strength.}
We vary the external-edge scale in a six-variable construction with SCCs $\{1,2,3\}$, $\{4,5\}$, and $\{6\}$.
Before rescaling, the coefficients within the cyclic SCCs are $(B_{21},B_{32},B_{13},B_{31})=(0.60,-0.55,0.58,0.28)$ and $(B_{54},B_{45})=(0.55,-0.50)$.
The coefficients on edges between SCCs are $(B_{41},B_{43},B_{52},B_{65})=a(1,1/2,-3/4,1)$, with $a\in\{0.10,0.20,0.35,0.50,0.70,0.90\}$; all other entries are zero.
The two cyclic blocks are rescaled so that their entrywise absolute coefficient matrices have spectral radii $0.55$ and $0.50$, respectively.
We also consider three variants obtained by multiplying nonzero coefficients independently by $\operatorname{Unif}(0.94,1.06)$ factors and rescaling the cyclic blocks again, with the same factors across scales.
Errors are independent $\operatorname{Unif}(-\sqrt3,\sqrt3)$ variables.

Figure~\ref{fig:margin-landscape}(a) shows median margins across the four coefficient variants at each scale.
The parent margin increases with external-edge strength, whereas the estimated dependence margin does not.
Thus stronger external edges do not necessarily give greater dependence separation.

\paragraph{Noise distribution.}
We vary the noise distribution while keeping the graph, coefficients, and covariance fixed in the one-SCC construction with $p=10$, $s_{\max}=3$, and $d_B=2$.
The four noise laws are uniform, Beta$(2,5)$, symmetric mixture $\tfrac12N(-0.95,1-0.95^2)+\tfrac12N(0.95,1-0.95^2)$, and the skewed mixture in Appendix~\ref{app:fullsample-protocol}, all standardized to zero mean and unit variance.
The dependence boxplots in Figure~\ref{fig:margin-landscape}(b) summarize the Monte Carlo variability of the margin estimates computed separately from each reference sample.

The parent margin remains $0.223$, while the estimated dependence margin ranges from $0.00092$ for uniform noise to $0.0471$ for the skewed Gaussian mixture, a roughly $51$-fold difference.
Thus the noise distribution can substantially change dependence separation even when the graph and covariance are unchanged.

\begin{figure}[!htbp]
\centering
\includegraphics[width=0.9\linewidth]{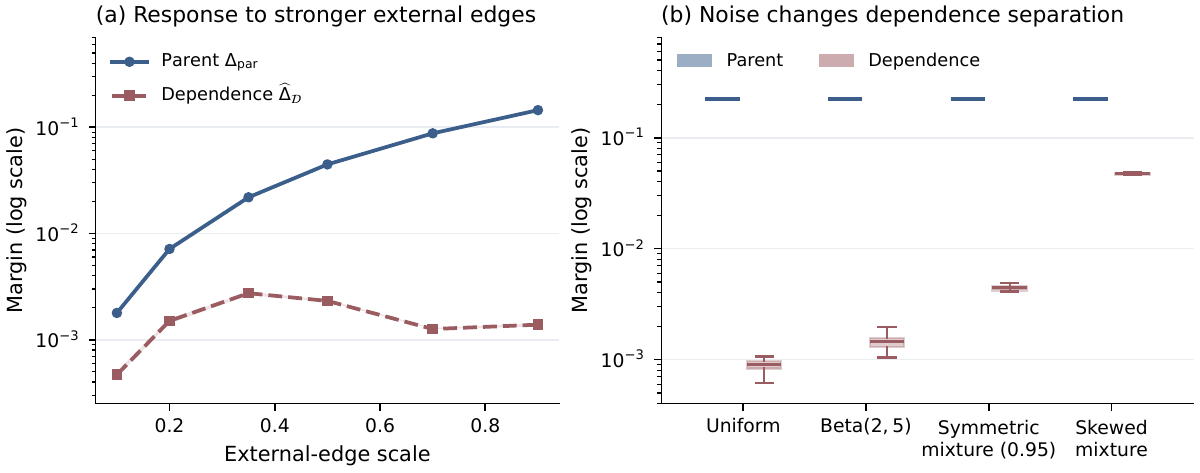}
\caption{Parent and dependence margins under varying external-edge
strength and noise distributions.}
\label{fig:margin-landscape}
\end{figure}

\subsection{Additional sample-efficiency comparisons}
\label{app:dag-comparison}

We extend the method comparison in
Figure~\ref{fig:fullsample-baseline-m4-preview} to different noise
distributions and DAGs, using $(s,d)=(3,2)$
for BlockExo throughout.

\paragraph{Noise distributions.}
\label{app:noise-comparison}
We use the four-SCC construction at $p=50$ with
$s_{\max}\in\{2,3\}$ and $d_B\in\{1,2\}$, replacing the noise law by
either the symmetric Gaussian mixture
$\tfrac12N(-0.95,1-0.95^2)+\tfrac12N(0.95,1-0.95^2)$
or $\operatorname{Unif}(-\sqrt3,\sqrt3)$.
The graph and coefficients remain fixed; both noise laws have zero mean
and unit variance, preserving the population covariance.

Under the symmetric Gaussian mixture, BlockExo begins recovering disjoint cycles
at smaller sample sizes than Coarsening
(Figure~\ref{fig:appendix-gm095-comparison}).
At $n=512$, it achieves $15/20$ and $13/20$ exact recoveries for
$d_B=1,2$, respectively, while both Coarsening thresholds achieve none.
For overlapping cycles, the recovery transitions are similar at $d_B=1$,
but Coarsening $0.2$ reaches high recovery accuracy earlier at $d_B=2$.
BlockExo attains $20/20$ exact recoveries in all four settings at $n=1024$.

\begin{figure}[!htbp]
\centering
\includegraphics[width=\linewidth]{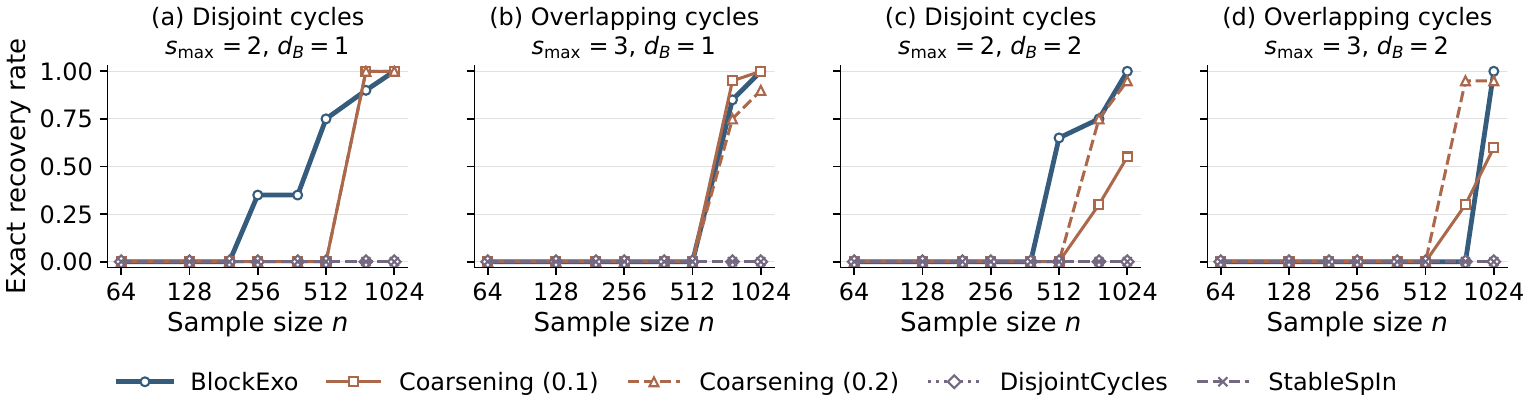}
\caption{Exact condensation recovery under symmetric Gaussian mixture noise.}
\label{fig:appendix-gm095-comparison}
\end{figure}

Under uniform noise, the relative sample cost depends on the cycle
structure and external-parent count
(Figure~\ref{fig:appendix-uniform-comparison}).
BlockExo begins recovering disjoint cycles at smaller sample sizes, although
Coarsening $0.2$ reaches $20/20$ earlier at $d_B=1$.
For overlapping cycles, Coarsening $0.2$ reaches high recovery accuracy
earlier at $d_B=1$, while the two methods perform similarly at $d_B=2$.
BlockExo attains $20/20$ exact recoveries in all four settings at $n=3072$.

Neither DisjointCycles nor StableSpIn achieves exact partition recovery
under either noise law at the tested sample sizes.
For DisjointCycles, both noise laws fall outside its generic-moment condition, and the overlapping construction also violates its cycle-disjointness assumption.

\begin{figure}[!htbp]
\centering
\includegraphics[width=\linewidth]{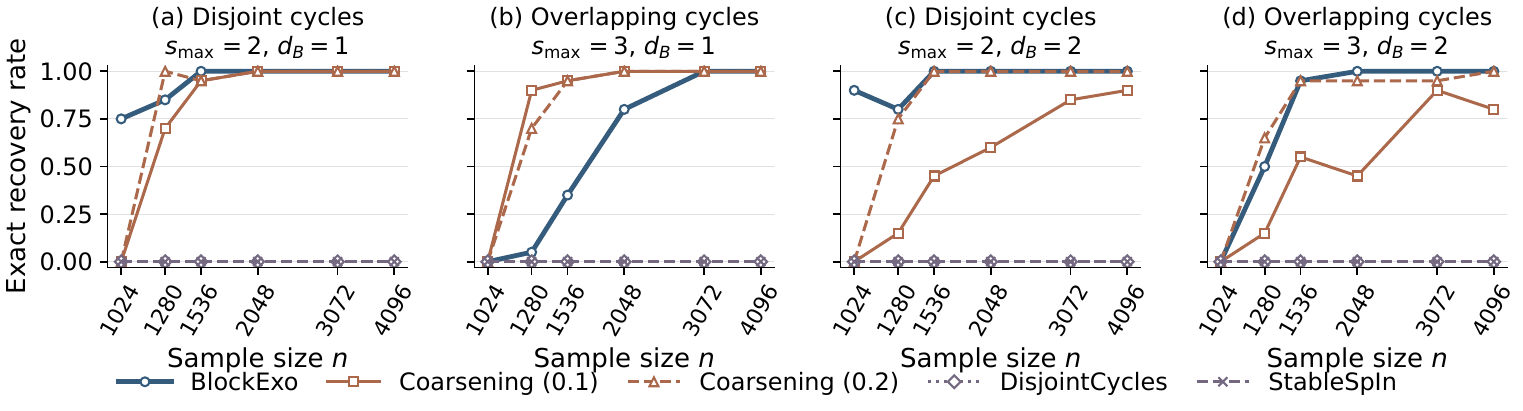}
\caption{Exact condensation recovery under uniform noise.}
\label{fig:appendix-uniform-comparison}
\end{figure}

The sample size needed for BlockExo to achieve high recovery accuracy
increases from the skewed Gaussian mixture to the symmetric Gaussian
mixture and then to uniform noise
(Figures~\ref{fig:fullsample-baseline-m4-preview},
\ref{fig:appendix-gm095-comparison}, and~\ref{fig:appendix-uniform-comparison}),
despite identical graphs, coefficients, and covariance.
This ordering is consistent with the one-SCC margin analysis in
Appendix~\ref{app:margin-meaning}, where the estimated dependence margin
decreases across these three noise laws in the same order.
Together, these results highlight distribution-dependent separation as
a source of recovery difficulty beyond SCC size and external-parent count.

%
%
%

\paragraph{DAG recovery.}
We consider DAGs with $p=50$ and $d_B\in\{1,2\}$, consisting of
$d_B$ roots, each connected to every remaining node.
Every SCC is a singleton, so exact condensation recovery is exact DAG recovery.
Even with search bounds $(s,d)=(3,2)$ allowing cyclic blocks,
BlockExo reaches $20/20$ exact recoveries at smaller sample sizes
than either Coarsening threshold for $d_B=1$, and at the same
sample size as Coarsening $0.2$ for $d_B=2$
(Figure~\ref{fig:appendix-dag-comparison}).

\begin{figure}[!htbp]
\centering
\includegraphics[width=\linewidth]{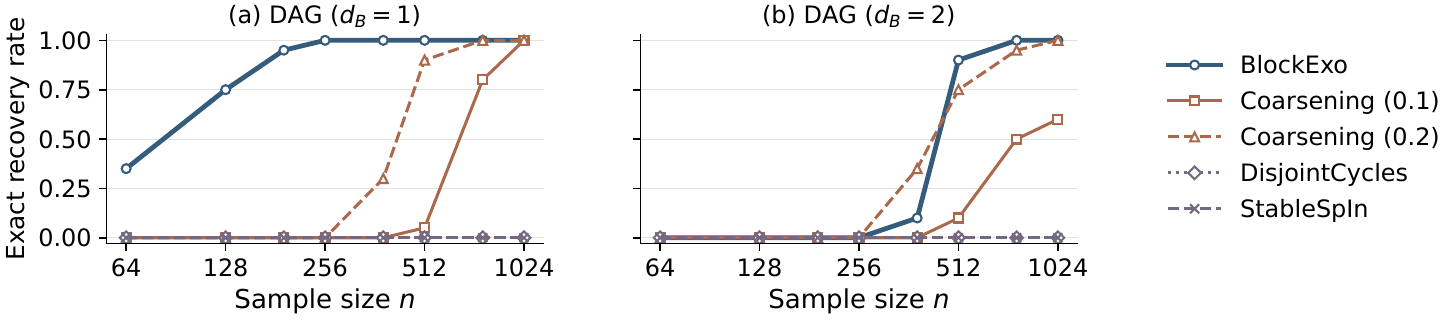}
\caption{Exact DAG recovery across methods.}
\label{fig:appendix-dag-comparison}
\end{figure}

\subsection{Sensitivity to search bounds}
\label{app:search-bound-sensitivity}

We compare the common search bounds $(s,d)=(3,2)$ with the exact bounds
$(s,d)=(s_{\max},d_B)$ on the same four-SCC datasets from
Figure~\ref{fig:fullsample-baseline-m4-preview}, keeping the search order,
threshold rule, and multipliers fixed.
Even without the exact search bounds, BlockExo attains $20/20$ exact
recoveries in every setting, although recovery requires larger sample
sizes at $d_B=1$.
For disjoint cycles at $n=192$, the common bounds reduce the exact-recovery rate
from $18/20$ to $1/20$; for overlapping cycles at $n=384$,
the rate decreases from $20/20$ to $3/20$.
At $(s_{\max},d_B)=(2,2)$, the two choices agree in exact-recovery
outcome for every replicate despite the different bounds.

\begin{figure}[!htbp]
\centering
\includegraphics[width=\linewidth]{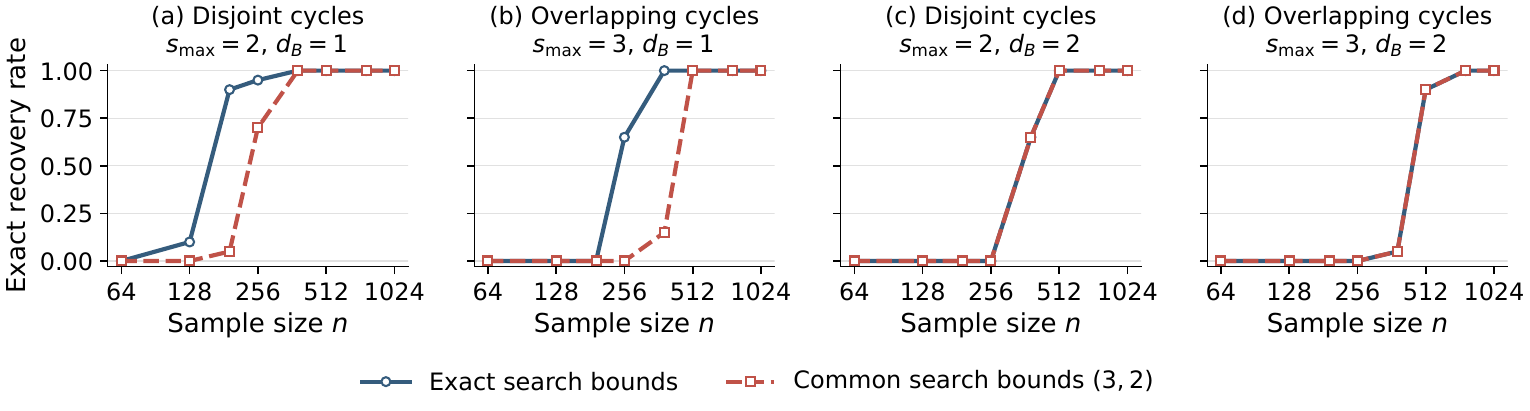}
\caption{Sensitivity of exact condensation recovery to search bounds.}
\label{fig:search-bound-sensitivity}
\end{figure}

\subsection{Runtime}
\label{app:runtime}

BlockExo's median runtime ranges from 8 to 34 seconds across the four
settings in Figure~\ref{fig:fullsample-baseline-m4-preview}.
BlockExo is slower than Coarsening; on disjoint cycles, its median
runtime is comparable to DisjointCycles and shorter than StableSpIn.
Table~\ref{tab:runtime-methods} reports median runtimes pooled across
sample sizes on the same inputs for which BlockExo completed,
regardless of recovery accuracy.

\begin{table}[!htbp]
\centering
\caption{Median runtimes in seconds for the
Figure~\ref{fig:fullsample-baseline-m4-preview} experiments.}
\label{tab:runtime-methods}
\small
\begin{tabular}{@{}lrrrrr@{}}
\toprule
$(s_{\max},d_B)$ & BlockExo & Coarsening $0.1$ & Coarsening $0.2$ & DisjointCycles & StableSpIn \\
\midrule
$(2,1)$ & 8.12 & 1.87 & 1.86 & 8.09 & 27.14 \\
$(2,2)$ & 11.14 & 1.85 & 1.87 & 8.26 & 27.09 \\
$(3,1)$ & 34.11 & 1.86 & 1.84 & 8.71 & 26.88 \\
$(3,2)$ & 27.88 & 1.87 & 1.84 & 8.54 & 27.12 \\
\bottomrule
\end{tabular}
\end{table}

Timing covers execution from initialization through recovery evaluation, including subprocess overhead.
CPU fits ran on an Intel Xeon Gold 6258R with one BLAS/OpenMP thread per fit. StableSpIn ran on an NVIDIA RTX A5000 GPU.



\end{document}